\documentclass[11pt]{article}
\usepackage[a4paper,margin=1in]{geometry}
\usepackage[T1]{fontenc}
\usepackage{lmodern}
\usepackage{amsmath,amssymb,amsthm}
\usepackage{mathtools}
\usepackage{graphicx}
\usepackage{array}
\usepackage{booktabs}
\usepackage{longtable}
\usepackage{multirow}
\usepackage{float}
\usepackage{placeins}
\usepackage{authblk}
\usepackage[dvipsnames]{xcolor}
\usepackage[most]{tcolorbox}
\usepackage{pifont}
\usepackage{microtype}
\usepackage{enumitem}
\usepackage[numbers,sort&compress]{natbib}
\usepackage{caption}
\usepackage{abstract}
\usepackage{tikz}
\usepackage{pgfplots}
\pgfplotsset{compat=1.18}
\usetikzlibrary{arrows.meta,positioning,fit,backgrounds,patterns,calc,
  shapes.geometric,decorations.pathreplacing,decorations.markings}
\usepackage[colorlinks=true,linkcolor=NavyBlue,citecolor=ForestGreen,urlcolor=NavyBlue]{hyperref}

\renewcommand{\arraystretch}{1.2}
\newtheorem{proposition}{Proposition}

\theoremstyle{definition}
\newtheorem{definition}{Definition}

\theoremstyle{remark}

\newtcolorbox{hypobox}{breakable,colback=NavyBlue!4,colframe=NavyBlue!55,boxrule=1pt,arc=3pt,
  left=8pt,right=8pt,top=6pt,bottom=6pt}
\newtcolorbox{claimbox}{breakable,colback=ForestGreen!4,colframe=ForestGreen!55,boxrule=1pt,arc=3pt,
  left=8pt,right=8pt,top=6pt,bottom=6pt}
\newtcolorbox{scopebox}{breakable,colback=black!3,colframe=black!42,boxrule=0.8pt,arc=3pt,
  left=8pt,right=8pt,top=6pt,bottom=6pt}
\newtcolorbox{warnbox}{breakable,colback=BrickRed!4,colframe=BrickRed!55,boxrule=1pt,arc=3pt,
  left=8pt,right=8pt,top=6pt,bottom=6pt}

\newcommand{\geff}{\gamma_{\mathrm{eff}}}

\usepackage[normalem]{ulem}
\newcommand{\pred}[1]{\textcolor{BrickRed}{\uline{#1}}}

\definecolor{cChain}{RGB}{0,114,178}
\definecolor{cCycle}{RGB}{213,94,0}
\definecolor{cGround}{RGB}{0,158,115}
\definecolor{cSelf}{RGB}{204,121,167}

\usepackage{xurl}

\definecolor{figA}{HTML}{0072B2}
\definecolor{figB}{HTML}{D55E00}
\definecolor{figC}{HTML}{009E73}
\definecolor{figD}{HTML}{CC79A7}
\definecolor{figE}{HTML}{56B4E9}
\definecolor{figF}{HTML}{E69F00}
\definecolor{figG}{HTML}{999999}
\definecolor{figH}{HTML}{44AA99}
\definecolor{figI}{HTML}{882255}
\definecolor{figJ}{HTML}{6699CC}
\definecolor{figInk}{HTML}{1A1A1A}
\definecolor{figSoft}{HTML}{5A5A5A}
\definecolor{figGrid}{HTML}{D8D8D4}
\definecolor{figPred}{HTML}{B03A2E}

\newcommand{\gnv}[1]{\csname gnv-#1\endcsname}
\expandafter\def\csname gnv-frontierCap\endcsname{0.850}
\expandafter\def\csname gnv-frontierBar\endcsname{0.50}
\expandafter\def\csname gnv-gateHigh\endcsname{0.993}
\expandafter\def\csname gnv-gateHighLo\endcsname{0.9934}
\expandafter\def\csname gnv-gateHighHi\endcsname{0.9935}
\expandafter\def\csname gnv-ceiling\endcsname{0.9938}
\expandafter\def\csname gnv-row2p\endcsname{0.017}
\expandafter\def\csname gnv-retrDisp\endcsname{0.500}
\expandafter\def\csname gnv-retrDedup\endcsname{0.367}
\expandafter\def\csname gnv-gammaLoPrior\endcsname{0.711}
\expandafter\def\csname gnv-gammaHiPrior\endcsname{0.411}
\expandafter\def\csname gnv-gammaDrop\endcsname{0.300}
\expandafter\def\csname gnv-ceilOn\endcsname{4}
\expandafter\def\csname gnv-ceilCap\endcsname{7}

\newcommand{\scTotal}{87}
\newcommand{\scAgainst}{37}
\newcommand{\scFor}{50}

\title{\bf Self-Cleaning and Captured Anyway:\\[2pt]
\large One Measured Primitive for Error in a Store an Agent Writes to Itself,\\
and What a Falling Score Actually Measures}

\author[1]{Wenhui Chen\thanks{\texttt{mc35092@um.edu.mo}}}
\author[2]{Jianlin Chen\thanks{\texttt{202330450231@mail.scut.edu.cn}}}
\author[1]{Ziyao Lin\thanks{\texttt{mc35081@um.edu.mo}}}
\author[1]{Chi Man Vong\thanks{Corresponding author. \texttt{cmvong@um.edu.mo}}}
\affil[1]{University of Macau}
\affil[2]{South China University of Technology}
\date{}

\begin{document}
\maketitle

\begin{abstract}
\noindent
An agent that writes its conclusions into a store it later retrieves from closes a loop, and the
contamination has been reported as a one-way process
\citep{zhu2026selfconfirmation,statecontam2026,consistencygate2026}. Taking that loop to the
infinite-tenure limit against an append-only store gives a different picture. \textbf{Because writing
never deletes, the reachable state space has a hard upper edge at $(n{-}1)/n$, so the outcome is a
choice between two edges rather than a decay}: the upper mode is not an attractor but a wall the
store builds itself. At $f_0{=}0.9$ the interval between the two modes holds $3.6\%$ of $220$ runs
where a uniform spread would put $20.6\%$, and is strictly empty on the first $15$. Four of seven
captured runs finish within two records of the computed ceiling $\gnv{ceiling}$. Against a criterion
frozen beforehand, the pooled mean describes $8.2\%$ of the runs it summarises and the median
$68.2\%$. A per-model average of this process is a statistic with no referent.

\textbf{Everything the model contributes is carried by one measured primitive with no fitted
parameter, the copy function $\gamma(\varphi)$, and it survives the move to real facts.}
Replacing the synthetic world by $36$ Wikidata facts spanning the full range of measured prior
knowledge leaves the threshold intact: $\mathrm{sign}(\hat\gamma - \gamma_{\rm crit})$, with
$\gamma_{\rm crit} = 1/k$ at $r{=}0, w{=}1$, predicts the direction of drift on $353$ of
$360$ real-fact runs, against $39$ of $40$ synthetic runs in the same batch; the denominators differ
ninefold, so the reading is that the threshold survives, not that the rates differ. The copy
probability under a unanimous false store falls from $\gnv{gammaLoPrior}$ to $\gnv{gammaHiPrior}$
across prior terciles, and the synthetic world sits $\gnv{gammaDrop}$ above the real curve's
low-prior end, so it bounds the risk rather than instantiating it. Nor does the empty interval
come from the answer space: a state-independent reader on the identical urn puts $113$ of $300$ runs
inside it, so what empties it is the copy probability rising with the number of false records
retrieved. Scale does not rescue the store: at frontier scale capture is $\gnv{frontierCap}$ pooled,
with \texttt{claude-sonnet-4.5} captured on $20$ of $20$ seeds, against a registered prediction of
ours of ${<}0.5$.

\textbf{We then attack the result on the axes a reader will, and report the disagreement rather than
the arm that suits us.} Whether a store holding \emph{many} distinct falsehoods still has an empty
interval is answered differently by our two arms: on the synthetic world the highest-scattering model
has $0$ of $44$ runs inside it, while on the $36$ Wikidata facts, where the planted falsehood is
another real capital, the runs whose store went multi-valued have $6.4\times$ the interval occupancy
of those that did not. That relocates the open question from how many falsehoods a store holds
to how \emph{distinguishable} they are. The empty interval survives at $f_0 \in \{0.1,0.3,0.5\}$, and
$f_0{=}0.9$ is not the adversarial extreme: capture peaks at $f_0{=}0.5$. Of four interventions with
criteria frozen first, timing dominates fraction at matched budget, a consistency gate drives every
model to $\gnv{gateHigh}$ by helping the false majority, and provenance marking moves nothing unless
the mark correlates with truth. Two registered arms test the apparatus itself: an embedding ranking
moves capture by $\gnv{retrDisp}$, so the uniform assumption is not innocuous and every threshold
here is scoped to it by name; and a sweep of the fraction of frozen retrieval slots reproduces both
endpoints exactly while refuting the variance account we proposed for them. One
methodological result reaches past this paper: the resampling unit is the \emph{seed}, at a design
effect of $3.75$ on a pooled level and $1.14$--$1.47$ on a paired within-seed contrast measured at
$44$ seeds. Every verdict here stands under it, including the four interventions', under a $44$-seed
control that moves their control's capture rate by up to $0.68$. It also reaches the two per-model
orderings, which move in opposite directions: claim 4, that the early copy rate orders capture
rate, falls from Spearman $+0.98$ at three seeds to $+0.31$ to $+0.80$ at forty-four, while claim
2's ordering, $+0.82$ ``with one adjacent inversion'' at three seeds, is exact at forty-four ($\rho =
+1.00$, $p = \gnv{row2p}$). The $1/n$ account of this paper's own timescale is vacuous over the first
$49$ of $150$ steps while the decision is measured at $t{=}2$ (Prop.~\ref{prop:step}). All $\scTotal$
graded rows are in Appendix~\ref{app:scorecard}, $\scAgainst$ graded withdrawn, failed,
self-correcting, undecidable or an acknowledged limit, against $\scFor$ that are not.
\end{abstract}

% =====================================================================
\section{Introduction}
\label{sec:intro}

An agent that writes notes to itself and later reads them has closed a loop: its output becomes its
input. Wrong entries are retrieved, reinforced and rewritten, and the store drifts with no external
source of error. This is the deployment-time analogue of training on generated data
\citep{shumailov2024curse,alemohammad2023mad}, with one difference that changes the mathematics: the
store is \emph{append-only}, so nothing is removed, only diluted. The failure is documented at the
level of deployed agents \citep{statecontam2026,zhu2026selfconfirmation,%
stale2026}, and the systems in which it arises are ordinary: agents that write reflections or
episodic summaries back into a retrieved store \citep{park2023generative,shinn2023reflexion,%
packer2023memgpt} on top of the standard retrieve-then-generate pipeline \citep{lewis2020rag}. What
is missing is a measurement of the \emph{dynamics}: what the loop converges to, and which of its
knobs the outcome is a function of.

Sending the loop to infinity with that constraint pinned organises the paper. Let $f$ be the
false share of the store and $n$ the number of records retrieved per step. The mean field is
$\dot f = \bigl(G_{\mathrm{eff}}(f) - f\bigr)/n$, and because writing never deletes, the
reachable state space has a hard upper edge at $(n{-}1)/n$, which for our apparatus is $0.9938$.
Two boundaries therefore come from the limit rather than from the literature: what the store can
hold, and what a bounded grounding budget $R$ can hold it to, $f_{\max}(R) = (n_{\mathrm{false},0} +
Tw)/(n_0 + Tw + R)$. The second is the useful one: the governing quantity is then a budget rather
than a fixed \emph{fraction} of verified data, and the same budget spent early or late gives
different outcomes. \emph{That is where our setting and the generational one come apart.} The
accumulate-versus-replace results \citep{gerstgrasser2024accum,shumailov2024curse} and the analytical
line \citep{dey2024pathway,fu2025theoretical} are stated in terms of the \emph{proportion} of real
data, since there real data is present every round and re-usable without limit, so no scheduling
question arises. The closest budgeted analogue is the data-constrained regime, where a fixed corpus
is re-used a bounded number of times \citep{muennighoff2023dataconstrained} and allocation under a
fixed budget is the object \citep{kaplan2020scaling,hoffmann2022chinchilla}; even there the budget
buys repetition rather than placement in time. Under a bounded per-tenure budget the placement
question does arise, and \S\ref{sec:sched} measures that it dominates the proportion.

\paragraph{Contributions.} (i) The outcome distribution is \emph{bimodal} rather than a
decaying mean, and one measured primitive with no fitted parameter, the copy function
$\gamma(\varphi)$, says which way a model goes (\S\ref{sec:bimodal}). It \emph{transports}, and the
condition under which it does is itself a result: measured at the resolution it must predict at,
$\gamma$ calibrated on three topics carries to forty-four; measured coarser it fails one-sidedly,
which fixes the resolution (\S\ref{sec:gammaout}). Against a criterion frozen first, the pooled mean
describes $8.2\%$ of the runs it summarises and the median $68.2\%$, so the defect is in that
statistic and not in summarising. (ii) A falling exact-match score on a contaminated store
measures loss of determinism, not of information: pass@8 minus pass@1 is $+0.383$ on stores
from which nothing was removed, so a score reads $C\cdot s$, capability times commitment
(\S\ref{sec:commit}). (iii) Four interventions with criteria frozen before their data, and
the design rule they yield: timing dominates fraction at matched budget, so \emph{when} a grounding
budget is spent matters more than how much of it there is; a consistency admission gate helps the
false majority, the record-level analogue of the adversarial-majority failure
\citet{consensustrap2026} report between agents; and provenance marking
\citep{provtiered2026,memmark2026} moves outcomes only when the mark correlates with truth
(\S\ref{sec:sched}). (iv) Three axes of our own proposing, taken to destruction against
pre-registrations we wrote first (\S\ref{sec:killed}). (v) A methodological result that
reaches past this paper: the resampling unit is the \emph{seed}, whose variance share reads $2.0\%$
at three seeds and $68.8\%$ at nine, a design effect of $3.75$, so each seed carries $1.33$ runs of
information however many models are run (\S\ref{sec:power}). (vi) That result applied to every row it
had not reached, which makes the verdicts here load-bearing: it separates cleanly into a factor of
$3.56$--$3.66$ on pooled \emph{levels} and $1.14$--$1.47$ on paired \emph{contrasts}, both measured
at $44$ seeds, so the four interventions' criteria stand, one is regraded \textsc{underpowered} for
missing its own bar on one of its three seeds, claim 2's ordering is exact and claim 4's does not
survive (\S\ref{sec:knife}). (vii) The scope of claim 1, measured rather than argued: the empty
interval is not a consequence of a two-valued answer space, and the regime it holds in has a boundary
in how far a store's falsehoods scatter (\S\ref{sec:answerspace}).
\citet{bean2025construct} report that only $16\%$ of $445$ benchmarks use uncertainty estimates;
this paper is in the $16\%$ and \S\ref{sec:power} is why that was not sufficient, which is the
same gap the evaluation-precision line attacks from the item side
\citep{polo2024tinybenchmarks,hofmann2025fluid,neuhof2026rankintervals}.

% =====================================================================
\subsection{Claims and evidence}
\label{sec:claims}

Every registered claim, its frozen criterion and its grade are in
Appendix~\ref{app:scorecard}: Table~\ref{tab:gradesummary} is the index, generated from the graded
rows' own \textsc{status} cells so that it cannot disagree with them, and the rows themselves are
Tables~\ref{tab:claims} and~\ref{tab:claimsapp}. The criteria and their freeze commits are
Appendix~\ref{app:prereg}. Where a criterion floored, the apparatus failure is reported rather
than the convenient verdict.

\noindent\textbf{What we do not claim.} Not priority: that an agent's own outputs degrade its
later behaviour is prior work, named in Table~\ref{tab:demarc} before any claim is made. Not a
location: no statistic here predicts \emph{where} the boundary between the two outcomes sits. Not
transfer to generational retraining. Not incidence in deployed stores; our worlds are synthetic
because traceability requires it. \textbf{No longer retrieval in general}:
\S\ref{sec:retriever} measures a $0.500$ displacement under an embedding ranking, so every number
here is a statement about uniform-without-replacement.
Figure~\ref{fig:overview} is the construction in one picture: the loop this paper pins, the one
function it measures inside that loop, and the two edges the function chooses between.
Figure~\ref{fig:arch} is then the whole of what follows on one page: the claims that survive
every attack in this paper, and the explanations the data refuted.

% =====================================================================
\begin{figure}[!tbp]
\centering
% GENERATED by t1_grounding/figures/make_t1_tikz.py -- do not edit by hand.
% Every coordinate is recomputed from t1_grounding/results/ at generation time.
\begin{tikzpicture}[
  bx/.style={draw=figSoft, line width=0.6pt, rounded corners=2pt, inner sep=3.5pt,
    font=\scriptsize\color{figInk}, align=center, fill=figGrid!30},
  hd/.style={font=\small\bfseries\color{figInk}, inner sep=1pt},
  lb/.style={font=\tiny\color{figSoft}, inner sep=1pt, align=left},
  ar/.style={-{Latex[length=1.6mm]}, line width=0.8pt}]
\useasboundingbox (0,0.45) rectangle (0.99\linewidth,-6.30);
\node[hd, anchor=south west] at (0.005\linewidth,0.16) {1. the loop, pinned};
\draw[figA, line width=1pt] (0.005\linewidth,0.10) -- (0.273\linewidth,0.10);
\node[hd, anchor=south west] at (0.35\linewidth,0.16) {2. the one thing measured};
\draw[figC, line width=1pt] (0.35\linewidth,0.10) -- (0.618\linewidth,0.10);
\node[hd, anchor=south west] at (0.69\linewidth,0.16) {3. what is at the boundary};
\draw[figB, line width=1pt] (0.69\linewidth,0.10) -- (0.958\linewidth,0.10);
\node[bx, text width=0.175\linewidth, anchor=north west] (ag) at (0.050\linewidth,-0.34) {the agent answers from $k=8$ records\\drawn uniformly, without replacement};
\node[bx, text width=0.175\linewidth, anchor=north west] (st) at (0.050\linewidth,-2.95) {an \textbf{append-only} store,\\$160$ records when tenure ends};
\draw[ar, figA] (ag.south) -- (st.north);
\node[lb, anchor=west, text width=0.105\linewidth, text=figA] at (0.160\linewidth,-2.30) {writes its own answer back};
\draw[ar, figA] (st.west) -- (0.032\linewidth,-3.55) -- (0.032\linewidth,-1.05) -- (ag.west);
\node[font=\tiny\color{figA}, rotate=90, anchor=south] at (0.023\linewidth,-2.65) {retrieves $k$ of them};
\node[lb, anchor=north west, text width=0.265\linewidth] at (0.005\linewidth,-4.62) {Writing never deletes, so no record can leave.};
\node[bx, text width=0.245\linewidth, anchor=north west] (gm) at (0.355\linewidth,-0.34) {$\gamma(\varphi)$: how often the agent copies a false value when a share $\varphi$ of what it retrieved carries it};
\node[lb, anchor=north west, text width=0.265\linewidth] at (0.355\linewidth,-2.30) {Measured in situ, with \textbf{no fitted parameter}. It is everything the model contributes.};
\draw[figSoft, line width=0.6pt] (0.37\linewidth,-3.75) -- (0.6\linewidth,-3.75);
\draw[figC, line width=1.2pt] (0.3987\linewidth,-3.92) -- (0.3987\linewidth,-3.58);
\node[font=\tiny\color{figSoft}, anchor=south west] at (0.37\linewidth,-3.70) {$\gamma = 0$};
\node[font=\tiny\color{figSoft}, anchor=south east] at (0.6\linewidth,-3.70) {$\gamma = 1$};
\node[font=\tiny\color{figC}, anchor=north west] at (0.3947\linewidth,-3.96) {$\gamma_{\rm crit} = 1/k = 0.125$, at $r{=}0$ and $w{=}1$};
\node[lb, anchor=north west] at (0.37\linewidth,-4.52) {self-cleaning};
\node[lb, anchor=north east] at (0.6\linewidth,-4.52) {capture};
\draw[figSoft, line width=0.5pt] (0.7\linewidth,-4.55) rectangle (0.738\linewidth,-0.45);
\fill[figB!18] (0.7\linewidth,-3.29) rectangle (0.738\linewidth,-2.45);
\draw[figB, line width=1.4pt] (0.7\linewidth,-0.48) -- (0.738\linewidth,-0.48);
\node[font=\tiny\color{figSoft}, anchor=east] at (0.694\linewidth,-0.45) {$1$};
\node[font=\tiny\color{figSoft}, anchor=east] at (0.694\linewidth,-4.55) {$0$};
\draw[ar, figB] (0.7190\linewidth,-2.55) -- (0.7190\linewidth,-0.64);
\draw[ar, figA] (0.7190\linewidth,-3.19) -- (0.7190\linewidth,-4.41);
\node[font=\tiny\color{figB}, anchor=north west, align=left, text width=0.205\linewidth] at (0.752\linewidth,-0.43) {the wall the store builds itself, $(n{-}1)/n = 0.9938$};
\node[font=\tiny\color{figB}, anchor=north west] at (0.752\linewidth,-1.83) {captured: $173$ of $220$};
\node[lb, anchor=north west, text width=0.205\linewidth] at (0.752\linewidth,-2.49) {between them $8$ of $220$, a share of $0.036$ where a uniform spread puts $0.206$; $0$ of $15$ at three seeds};
\node[font=\tiny\color{figA}, anchor=north west] at (0.752\linewidth,-3.63) {self-cleaning: $39$ of $220$};
\draw[figGrid, line width=0.8pt] (0.005\linewidth,-5.16) -- (0.955\linewidth,-5.16);
\node[font=\scriptsize\color{figInk}, anchor=north west, inner sep=1pt, text width=0.465\linewidth] at (0.005\linewidth,-5.30) {\textbf{It survives the move to real facts.} The sign of the measured $\gamma - \gamma_{\rm crit}$ predicts which edge a run reaches on $353$ of $360$ Wikidata runs, against $39$ of $40$ synthetic.};
\node[font=\scriptsize\color{figInk}, anchor=north west, inner sep=1pt, text width=0.465\linewidth] at (0.495\linewidth,-5.30) {\textbf{And it is why an average of this process has no referent.} The pooled mean describes $8.2\%$ of the runs it summarises; the median describes $68.2\%$.};
\end{tikzpicture}
\caption{\textbf{The apparatus, the primitive, and the boundary.} An agent answers from a store
it writes back into, and because the store is append-only no record can leave it, which is the
whole of the assumption. Everything the model contributes travels through one measured function,
the copy probability $\gamma(\varphi)$, and the sign of its distance from $\gamma_{\rm crit}$
says which edge a run reaches. What appears at the boundary is not a decay towards a mean but a
choice between the floor and a ceiling the store builds for itself. Every quantity drawn here is
read from \texttt{results/} at draw time.}
\label{fig:overview}
\end{figure}

% =====================================================================
\begin{figure}[!tbp]
\centering
% GENERATED by t1_grounding/figures/make_t1_tikz.py -- do not edit by hand.
% Every coordinate is recomputed from t1_grounding/results/ at generation time.

\begin{tikzpicture}[
  claim/.style={font=\footnotesize\color{figInk}, inner sep=1pt},
  gone/.style={font=\footnotesize\color{figPred}, inner sep=1pt},
  eviD/.style={font=\scriptsize\color{figSoft}, inner sep=1pt, align=left},
  secref/.style={font=\scriptsize\color{figSoft}, inner sep=1pt},
  hdr/.style={font=\small\bfseries\color{figInk}, inner sep=1pt}]

\node[hdr, anchor=south west] at (0,0.16) {What survives every attack in this paper};
\draw[figC, line width=1pt] (0,0.10) -- (0.99\linewidth,0.10);
\node[hdr, anchor=south west, text=figPred] at (0,-6.47)
  {Explanations the data refuted, and what refuted each};
\draw[figPred, line width=1pt] (0,-6.53) -- (0.99\linewidth,-6.53);
\node[claim, anchor=north west] at (0,0.00){\begin{minipage}{0.44\linewidth}\textbf{1} bimodal, the middle sparse and not empty\end{minipage}};
\node[eviD, anchor=north west] at (0.46\linewidth,-0.02){\begin{minipage}{0.33\linewidth}$0/15$ in $[0.35,0.50]$ at three seeds, $7/220$ at $220$; the graded share inside $(0.306,0.512)$ is $0.036$ against a uniform $0.206$; survives $f_0\in\{0.1,0.3,0.5\}$ and frontier scale (pooled $0.786$)\end{minipage}};
\node[secref, anchor=north east] at (0.97\linewidth,-0.02){\S\ref{sec:bimodal}};
\node[claim, anchor=north west] at (0,-1.76){\begin{minipage}{0.44\linewidth}\textbf{2} $\gamma(\varphi)$ is sufficient\end{minipage}};
\node[eviD, anchor=north west] at (0.46\linewidth,-1.78){\begin{minipage}{0.33\linewidth}zero fitted parameters; ordering exact at $44$ seeds, $\rho=+1.00$, $p=0.017$\end{minipage}};
\node[secref, anchor=north east] at (0.97\linewidth,-1.78){\S\ref{sec:gammaout}};
\node[claim, anchor=north west] at (0,-2.84){\begin{minipage}{0.44\linewidth}\textbf{3} timing dominates fraction\end{minipage}};
\node[eviD, anchor=north west] at (0.46\linewidth,-2.86){\begin{minipage}{0.33\linewidth}$0.20/0.33/0.40$ vs $0.47$, $5/5$ models; ordering holds at $44$ seeds, $220$ runs per arm\end{minipage}};
\node[secref, anchor=north east] at (0.97\linewidth,-2.86){\S\ref{sec:sched}};
\node[claim, anchor=north west] at (0,-3.92){\begin{minipage}{0.44\linewidth}\textbf{4} score $=C\cdot s$\end{minipage}};
\node[eviD, anchor=north west] at (0.46\linewidth,-3.94){\begin{minipage}{0.33\linewidth}pass@8 $-$ pass@1 $=+0.383$ on stores nothing was removed from\end{minipage}};
\node[secref, anchor=north east] at (0.97\linewidth,-3.94){\S\ref{sec:commit}};
\node[claim, anchor=north west] at (0,-5.00){\begin{minipage}{0.44\linewidth}\textbf{5} uniform retrieval is not innocuous\end{minipage}};
\node[eviD, anchor=north west] at (0.46\linewidth,-5.02){\begin{minipage}{0.33\linewidth}embedding ranking displaces capture by $0.500$; cause is ranking determinism\end{minipage}};
\node[secref, anchor=north east] at (0.97\linewidth,-5.02){\S\ref{sec:retriever}};
\node[gone, anchor=north west] at (0,-6.63){\begin{minipage}{0.40\linewidth}the $1/n$ timescale\end{minipage}};
\node[eviD, anchor=north west] at (0.41\linewidth,-6.65){\begin{minipage}{0.52\linewidth}Prop.~\ref{prop:step}: vacuous over $49$ of $150$ steps\end{minipage}};
\node[gone, anchor=north west] at (0,-7.51){\begin{minipage}{0.40\linewidth}the variance bound\end{minipage}};
\node[eviD, anchor=north west] at (0.41\linewidth,-7.53){\begin{minipage}{0.52\linewidth}$1.07\times$ against a required $2.83\times$\end{minipage}};
\node[gone, anchor=north west] at (0,-8.39){\begin{minipage}{0.40\linewidth}the mode balance\end{minipage}};
\node[eviD, anchor=north west] at (0.41\linewidth,-8.41){\begin{minipage}{0.52\linewidth}$8/7$ became $39/173$ at $15\times$ the sample\end{minipage}};
\node[gone, anchor=north west] at (0,-9.27){\begin{minipage}{0.40\linewidth}$\gamma$ from three topics\end{minipage}};
\node[eviD, anchor=north west] at (0.41\linewidth,-9.29){\begin{minipage}{0.52\linewidth}$12.1$ seed units out of sample, one-sidedly\end{minipage}};
\node[gone, anchor=north west] at (0,-10.15){\begin{minipage}{0.40\linewidth}\textbf{early copy rate orders capture}\end{minipage}};
\node[eviD, anchor=north west] at (0.41\linewidth,-10.17){\begin{minipage}{0.52\linewidth}$\rho$ $+0.80$ to $+0.31$ at $44$ seeds, was $+0.98$ at three\end{minipage}};
\node[gone, anchor=north west] at (0,-11.03){\begin{minipage}{0.40\linewidth}\textbf{the fill rule as the diagnosis}\end{minipage}};
\node[eviD, anchor=north west] at (0.41\linewidth,-11.05){\begin{minipage}{0.52\linewidth}it fires for one model, the one arm A predicts exactly\end{minipage}};
\node[font=\scriptsize\itshape\color{figSoft}, anchor=north west, text width=0.97\linewidth]
  at (0,-11.96)
  {Every refuted row is an \emph{explanation} offered for a survivor, never a
   measurement. Two of the six fell to this paper's own resampling-unit result
   (\S\ref{sec:power}) applied to rows it had not yet reached; the control those rows
   were measured against moves by $0.68$ in capture rate between three seeds and
   forty-four, and no verdict changes when it is replaced.};
\end{tikzpicture}
\caption{\textbf{The paper in one page: five results, and the six explanations the data
refuted.} Recomputed from \texttt{results/} at draw time, so it cannot drift from the
artifact it summarises. The five survivors are the whole of what we ask a reader to carry
away; the six refuted explanations are kept as a ledger because a corrected record is part
of the argument. \emph{Two of the six fell to this paper's own resampling-unit result
(\S\ref{sec:power}) applied to rows it had not yet reached}, Appendix~\ref{app:seedscale}
for the first, \ref{app:gammaest} for the second.}
\label{fig:arch}
\end{figure}

\begin{claimbox}
\textbf{What stands.} Four explanations offered along the way fail their tests (the $1/n$
timescale, the variance bound, one of the two $\gamma$ extrapolations, and the balance between
the two modes), and what survives all of them:
\begin{enumerate}[leftmargin=1.6em,itemsep=1pt,topsep=2pt,label=(\roman*)]
\item \textbf{Bimodality, with the interval between the modes strictly empty on $15$ runs and
  holding $3.6\%$ of $220$ against a uniform $20.6\%$}, at four levels of $f_0$, and surviving at
  frontier scale against our registered prediction that it would not.
\item \textbf{$\gamma(\varphi)$ is sufficient}: one measured one-dimensional function, zero
  fitted parameters, orders every model's capture rate and transports across a seed split.
\item \textbf{Timing dominates fraction at a matched budget} ($0.20/0.33/0.40$ against $0.47$,
  monotone in $5/5$ models), the one result that transfers directly to a deployed schedule.
\item \textbf{Score $=C\cdot s$}: contamination costs determinism rather than information, so a
  score drop is not a capability drop.
\item \textbf{The retrieval rule the other four are measured under is no longer an assumption.}
  The attached falsifier is fired: an embedding ranking moves capture by $0.500$, so the results
  are scoped to uniform retrieval \emph{by name}, and the resource that moves them is the
  \emph{variance} of the evidence stream rather than its breadth (\S\ref{sec:retriever}).
\end{enumerate}
Every refuted item is an \emph{explanation} offered for one of these, never the
measurement.
\end{claimbox}

\noindent\emph{Two companions exist and nothing below depends on either.} They pin
different resources and read their boundaries off different limits: \citet{t3resolving2026}
sends the item count to infinity against a fixed resolving power, \citet{t2library2026} sends
the instance count to infinity against a fixed per-paper testing budget. Every definition this
paper borrows is restated here and no number is taken from either, so each argument here stands
with both deleted, which \S\ref{sec:parent} makes explicit at the one place a companion's
claim is invoked, by stating it as a bound on ours rather than as support for it.

% =====================================================================
\section{Related work}
\label{sec:related}

We ask one question of each neighbouring result: what must be true for its conclusion to hold, and
does our apparatus put that assumption under strain? A result that survives is one we depend on. A
result that does not is not thereby wrong; its scope can now be stated more narrowly than its authors
could, because we ran a limit they had no reason to run. Table~\ref{tab:demarc} states row by row
what each neighbour establishes and what is added here.

\begin{table}[!htbp]
\centering\small
\caption{Nearest prior work. Every row is a phenomenon this paper does \emph{not} claim to have
discovered.}
\begin{tabular}{@{}p{0.18\linewidth}p{0.35\linewidth}p{0.41\linewidth}@{}}
\toprule
\textbf{Prior work} & \textbf{What it establishes} & \textbf{What is added here} \\
\midrule
Self-Confirmation Trap \citep{zhu2026selfconfirmation} &
  the closed loop named: own trajectories written to memory and reused &
  the outcome is bimodal rather than cumulative, and one primitive predicts which branch \\
State contamination \citep{statecontam2026} &
  unsafe content survives summarisation below detector thresholds &
  a different failure axis; no dynamics, no copy probability there \\
Admission control \citep{consistencygate2026} &
  write-time gating on self-consistency reduces contamination &
  the same gate has a regime where it erases the model's own cleaning \\
Majority aggregation traps
  \citep{consensustrap2026,minoritysentinel2026} &
  one-shot majority over samples locks in correlated error &
  the write side over tenure, where the gate decouples composition from reader quality \\
Accumulate vs.\ replace
  \citep{gerstgrasser2024accum,dey2024pathway} &
  accumulating real with synthetic data avoids collapse &
  a \emph{budgeted} grounding event, where schedule dominates at fixed fraction \\
Exact-match and pass@$k$ critiques &
  strict matching understates capability; pass$^k$ separates capability from reliability &
  an intervention that conserves $C$ while moving $s$ \\
Threat amplification in
  self-evolving systems \citep{selfevolve2026safety,alemohammad2023mad} &
  a component's failure is amplified by the loop that reuses its output &
  the amplification as one number with a fixed point, whose attractor is \eqref{eq:ceiling}
  rather than a divergence \\
State-conditioned memory
  \citep{memcompiler2026} &
  compiling a policy against state beats injecting retrieved text &
  the ceiling binds what the store \emph{contains}, so it binds whatever the compiler reads \\
\bottomrule
\end{tabular}
\label{tab:demarc}
\end{table}

\paragraph{Model collapse assumes the damage is summarised by a decaying statistic.} Training a
generative model on its predecessor's output degrades it \citep{shumailov2024curse,%
alemohammad2023mad,briesch2023selfconsuming}, costing tail coverage
\citep{dohmatob2024tails,seddik2024howbad} and output diversity \citep{guo2024diversity}; the
analytical line \citep{dey2024pathway,fu2025theoretical,bertrand2024stability} derives when it is
bounded, and collapse is avoidable if data accumulates rather than replaces
\citep{gerstgrasser2024accum} or if a correction function stabilises the loop
\citep{gillman2024selfcorrect}. Each result is a trajectory of a summary statistic across
generations, which needs the per-run outcome distribution to be unimodal enough for its mean to
denote something. \S\ref{sec:bimodal} measures the same loop at fixed tenure and finds otherwise: at
$f_0{=}0.9$ runs end at the two edges of the state space and the interval between them holds $3.6\%$
of $220$ runs against a uniform baseline of $20.6\%$. The step from that to a claim about summary
statistics is itself a claim, so we registered and measured it (P-T1-M1--M4; no API cost): the pooled
mean describes $8.2\%$ of what it summarises, the median $68.2\%$, against $74.5\%$ for the best
single centre at that width (Appendix~\ref{app:mean}). \emph{This distribution summarises perfectly
well; it is the mean that does not}, and the wider criticism is unsupported by our own data. Two
structural differences keep the settings apart, so the generational results are not wrong in their
own: retraining averages over a corpus while an agent's read samples $k$ of $n$ records, so the
variance driving our bimodality is averaged away there by construction; and real data is present
every round and reusable without limit, the $R \to T$ corner in which our schedule effect
vanishes.

\paragraph{Feedback-loop results outside language modelling assume the loop is closed through
the world.} A deployed predictor that shapes the data it is later trained on is the subject of
performative prediction \citep{perdomo2020performative}, of runaway feedback in predictive
policing \citep{ensign2018runaway}, of degenerate feedback in recommenders
\citep{jiang2019degenerate} and of model-driven amplification of dataset bias
\citep{taori2023feedback}. Those loops close through an environment whose response is exogenous: the
analysis turns on how the world reacts to the deployed model. Our loop closes inside one process,
with no environment and no second agent, so none of that machinery applies and what is left is the
reading behaviour of a single model against records it wrote itself. The lesson we keep is theirs:
the fixed point of the loop, not the quality of any single prediction, is the object to analyse.

\paragraph{Knowledge conflict fixes the contradiction and varies the model; we let the loop
generate it.} Whether a model follows a context that contradicts what it has memorised is
measured directly by the knowledge-conflict line: entity substitution in QA
\citep{longpre2021conflicts}, elicited parametric memory against constructed counter-memory
\citep{xie2024chameleon}, and the parametric baseline those studies contend with
\citep{petroni2019lama}. Which side wins is known to depend on how well the fact is represented
in pretraining \citep{mallen2023whennot,kandpal2023longtail}. Our $\gamma(\varphi)$ is that
measurement made \emph{in situ}: the conflicting evidence is produced by the agent's own earlier
writes rather than authored by an experimenter, its strength is the retrieved fraction $\varphi$
rather than a binary condition, and it is measured on the same runs whose trajectories it must
predict. \S\ref{sec:real} puts the parametric axis those papers vary, our $p_0$, on the same plot as
the loop, and finds $\gamma(1)$ falling from $0.711$ to $0.411$ across prior terciles while the drift
threshold does not move.

\paragraph{Agent-memory contamination assumes contamination is a quantity to be reduced.} A
second line studies the same failure in the store rather than the weights: contamination
surviving summarisation into agent state \citep{statecontam2026}, the self-confirmation trap
\citep{zhu2026selfconfirmation}, the best models near $55\%$ at noticing an invalidated memory
\citep{stale2026}, the memory-update gap scaling with conversation length \citep{supersede2026},
retrieval accuracy falling log-linearly under proactive interference \citep{pillm2025}, and three
surveys \citep{alwayson2026,agentmemsurvey2026,agentmemmech2026}, the last of which organises the
field by the write--manage--read cycle this apparatus closes. The shared move treats contamination as
a level to drive down, which assumes the final level is a monotone function of the level along the
way. \S\ref{sec:bimodal} measures that it is not: the sign of the trajectory is fixed in the first
few steps and the same $f_0$ produces both outcomes, so an intervention's value is set by when it
lands rather than by how much it removes.

\paragraph{The closest result to ours assumes the store is \emph{not} append-only, which is
exactly the assumption we pin.} \citet{romerl2026} characterise the steady-state occupancy of
erroneous coordinates in a self-evolving agent memory, under a coordinate-transition model over a
\emph{fixed} coordinate set whose contents are updated or replaced, and measure it on ALFWorld
and LifelongAgentBench. That is the same object as our $f$ and the same limit, taken under the
opposite structural assumption: replacement gives an ergodic chain with a stationary distribution, so
a steady-state occupancy exists and is the natural answer. Append-only forbids replacement, which
puts a hard edge at $(n{-}1)/n$ (Prop.~\ref{prop:ceiling}), removes the stationary distribution, and
leaves a \emph{choice between edges} rather than an occupancy to converge to. The two results do not
compete and neither transfers: our Prop.~\ref{prop:ceiling} is vacuous the moment a record can be
overwritten, and a stationary occupancy is undefined in our setting. This is the boundary
\S\ref{sec:limits}(i) declared as a hole in our own coverage, better stated as a division of labour
with a measured neighbour, with one asymmetry we owe: their regime is the deployed one, ours the one
an append-only design produces.

\paragraph{Memory poisoning assumes an adversary, and our store has none.} Knowledge-base
corruption at low poison rates \citep{chen2024agentpoison}, adversarial passages injected into a
retrieval corpus \citep{zhong2023corpuspoisoning,zou2025poisonedrag}, query-only injection
\citep{dong2025minja} and sentries against high injection rates \citep{torres2026ghostwriter} are
all organised around locating an injected item. Our apparatus removes the attacker: every false
record after step $0$ is written by the agent itself, in good faith, from what it read, so a detector
trained to find injected content has no positive class here and the store still reaches $0.99$ false.
Contamination needs no adversary, and a defence keyed to one does not bound it.

\paragraph{Provenance assumes the mark carries information about truth.} Tiered and watermarked
memories \citep{provtiered2026,memmark2026} and the tracing survey \citep{agenttraces2026} make
each record's origin recoverable so a reader can discount derived records. \S\ref{sec:sched}
registers and measures that prediction: marking moves outcomes only when the mark correlates with
truth, and in our loop it does not, the agent's own writes being as often right as its reads
permit.

\paragraph{Consistency and majority reads assume the majority is more often right.} Admission on
repeated sampling \citep{consistencygate2026}, which inherits the majority rule from
self-consistency decoding \citep{wang2023selfconsistency}, majority voting failing under
corrupted majorities
\citep{consensustrap2026} and the minority correct in roughly one disagreement in four
\citep{minoritysentinel2026} all read the majority as evidence. \S\ref{sec:sched} measures a majority
admission gate inside our loop and finds it erases model differences by making every model worse
alike: once the store is majority-false, the gate admits exactly the falsehood and rejects the
correction. \citet{consensustrap2026} anticipate this among agents in a debate; we measure it among
records in a store, where the adversarial majority is the reader's own past. Two deployed admission
controls are richer than ours and are the right comparison. \citet{amac2026} admit on five retention
factors (future utility, factual confidence, semantic novelty, temporal recency and a content-type
prior), so agreement with the incumbent store is one signal of five rather than the whole rule.
\citet{memguard2026} do what our own mechanism predicts would help: they \emph{type} memories by
functional role and keep the types separate at storage and at retrieval, changing the gate's
reference class instead of filtering within it, and the reference class is where we locate the harm.
We do not claim our result transfers to either, and a harmful regime in the weakest gate is
not ``conservative evidence about stronger versions'': that inference needs monotonicity in gate
strength, which we have not tested (\S\ref{sec:sched}).

\paragraph{The benchmarks measure retrieval; the object here is the loop.} BEAM
\citep{beam2026}, STATE-Bench \citep{statebench2026}, RECON \citep{recon2026} and the
precision-aware line \citep{precisionaware2026} evaluate what an agent can retrieve from a memory it
is given, and \citet{leakagefree2026} and \citet{huang2026llmgen} address the adjacent problem that
retrieval corpora are increasingly model-generated. All hold the store fixed and vary the reader; we
hold the reader fixed and let the store be written by it, making its composition endogenous rather
than a condition of the experiment. \S\ref{sec:commit} shows a system can score well on the first
design and fail the second.
Figure~\ref{fig:apparatus} is the loop this paper runs and the place each claim attaches to it.

\section{The synthetic apparatus}
\label{sec:apparatus}

\begin{figure}[!tbp]
\centering
\begin{tikzpicture}[
  font=\small,
  box/.style={draw, rounded corners=2pt, align=center, inner sep=4pt, line width=0.7pt},
  op/.style={box, fill=NavyBlue!6, draw=NavyBlue!55},
  st/.style={box, fill=black!4, draw=black!45},
  ms/.style={box, fill=ForestGreen!7, draw=ForestGreen!55},
  wl/.style={box, fill=BrickRed!6, draw=BrickRed!55},
  ar/.style={-{Latex[length=1.7mm]}, line width=0.65pt},
  lbl/.style={font=\scriptsize, inner sep=1.5pt}]

% One horizontal axis: input -> loop -> outcome. The two outcomes are placed
% symmetrically about that axis so that neither reads as the default.
\node[op]                      (gnd)   at (0,0)      {grounding\\[-1pt]\scriptsize $R$ of $T$ steps};
\node[op]                      (agent) at (4.6,0)    {agent};
\node[st]                      (store) at (7.9,0)    {store\\[-1pt]\scriptsize $n$ records, $f$ false};
\node[wl, anchor=west]         (ceil)  at (10.5,1.05){$\dfrac{n_{\mathrm{false}0}+W}{n_0+W+R}$\\[1pt]\scriptsize append-only ceiling};
\node[ms, anchor=west]         (floor) at (10.5,-1.15){$f \to 0$\\[-1pt]\scriptsize cleaned out};
\node[ms]                      (gam)   at (4.6,-1.95){$\gamma(\varphi)$\\[-1pt]\scriptsize copy function, in situ};

\node[lbl, below=1.2mm of gnd, align=center] (sched) {front / uniform / back};
% The enclosure is painted FIRST, so it cannot rule a line across the labels inside it --
% it did, through "append 1" -- and it fits (sched) as well, which is wider than the boxes it
% labels and was hanging out through the frame's left edge.
\node[draw=black!35, dashed, rounded corners=2pt, fit=(gnd)(agent)(store)(gam)(sched),
      inner sep=13pt] (fix) {};
\draw[ar] (gnd) -- node[lbl, above=0.6mm]{\emph{when}, not how much} (agent);
\draw[ar] (agent.20) to[bend left=20] node[lbl, above]{append $1$} (store.160);
\draw[ar] (store.200) to[bend left=20] node[lbl, below]{retrieve $k$} (agent.340);
\draw[ar] (agent) -- (gam);
\draw[ar, dashed, draw=black!55] (store.35)  -- (ceil.west);
\draw[ar, dashed, draw=black!55] (store.325) -- (floor.west);

\node[lbl, below=1.0mm of ceil, align=center]{$0.9938$ at $n_0{=}10$, $R{=}0$};

\node[lbl, font=\scriptsize\itshape, text=black!60, above right=0.3mm and 0pt of fix.north west]
      {fixed resource: $R$ grounding steps in a tenure of $T$};
\end{tikzpicture}
\caption{\textbf{The apparatus, and where the claims attach.} One move (retrieve $k$ of $n$,
answer, append one record), run for $T$ steps with $R$ of them grounded. Fixing $R$ and sending
$T$ to infinity makes two boundaries visible, and the alternative to the ceiling is zero rather
than a lower plateau. Claim~1 is that runs end at one boundary or the other; claim~2 that
$\gamma$ alone says which; claim~3 that the choice is made in the first few steps.}
\label{fig:apparatus}
\end{figure}
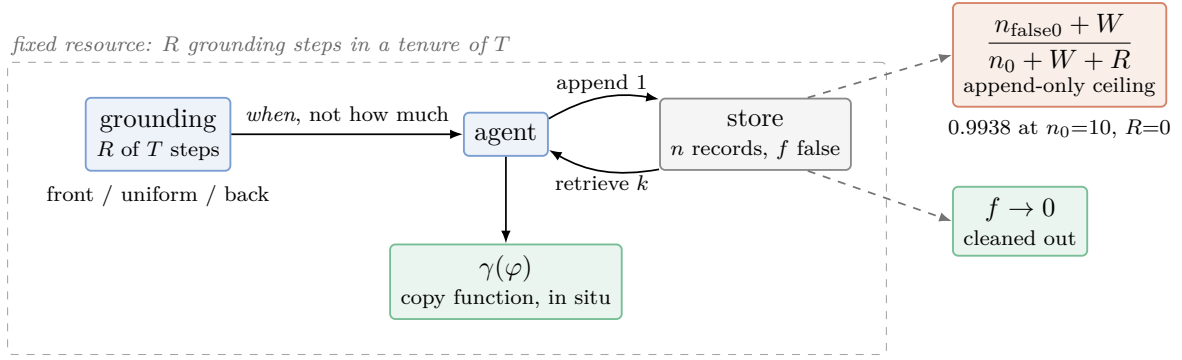

\begin{table}[!htbp]
\centering\small
\caption{The protocol. Every arm is this loop with one thing changed.}
\label{tab:protocol}
\begin{tabular}{@{}p{0.14\linewidth}p{0.82\linewidth}@{}}
\toprule
world & one tracked topic: a generated entity (\texttt{Baldrix Holdings}) with a generated
  numeric attribute (\texttt{68.91M}). The false value is a \emph{one-digit} corruption:
  plausible, never equal, unreachable by guessing \\
store at $t{=}0$ & $n_0{=}10$ records, $n_{\mathrm{false}} \in \{1,9\}$ of them asserting
  the corruption, giving $f_0 \in \{0.1, 0.9\}$ \\
each step & retrieve $k{=}8$ records uniformly from the topic; ask the plain question a
  deployed assistant would be asked; append the answer \\
prompt & never states the truth, never warns of error, so $\gamma$ measures ordinary
  behaviour rather than a response to a warning \\
tenure & $150$ steps, growing the store from $11$ to $161$ records: a $15\times$ range in
  $n$, which is what makes the $1/n$ term visible within one run \\
models & five, three seeds, temperature $0$ \\
\bottomrule
\end{tabular}
\end{table}

\paragraph{Why the world is generated.} Traceability requires it: a propagated error is attributable
only if the truth is known by construction, and with a real corpus a wrong answer is
indistinguishable from a wrong prior (Table~\ref{tab:protocol}). The same reasoning drives
leakage-free benchmark generation \citep{leakagefree2026} and the detection of model-generated text
inside a retrieval corpus \citep{huang2026llmgen}; we generate the world because we need not only
that the answer is parametrically unavailable but that every false record's \emph{provenance within
the run} is recoverable. The cost is \S\ref{sec:limits}(vii).

\paragraph{Why retrieval is uniform.} Uniform-without-replacement within the topic is the assumption
\eqref{eq:dyn} makes, so the first experiment tests the predicted dynamics rather than their
interaction with an embedding model's idiosyncrasies. It is also what most separates this apparatus
from the memory benchmarks \citep{beam2026,statebench2026,recon2026,precisionaware2026}, which hold
the store fixed and vary the reader. The choice carries a falsifier: swapping in an embedding
retriever changes which records compete, and if the behaviour moves, that displacement is the finding
rather than a nuisance. \textbf{\S\ref{sec:retriever} fires it, and the behaviour moves}, by $-0.500$
at $f_0=0.1$ and $+0.067$ at $f_0=0.9$.

\paragraph{How $\gamma$ is measured, and why it is not fitted.} At every step the number $j$ of
retrieved records asserting a false value is logged with whether the answer was the truth, and
pooling over steps gives $\gamma(j/k)$ directly, from the same runs whose trajectories are being
predicted. The curve is an observation of reading behaviour rather than a parameter tuned to fit a
trajectory, and \S\ref{sec:bimodal} puts it back into the loop in place of the model, the only sense
in which ``$\gamma$ is enough'' can be checked rather than asserted.

\begin{definition}[The copy function]
\label{def:gamma}
$\gamma(\varphi)$ is the probability that the agent asserts a false value when a fraction
$\varphi$ of the $k$ retrieved records assert it. With $f$ the false share of the store,
$\varphi \sim \mathrm{Bin}(k,f)/k$, so the per-step assertion probability is
$\geff(f) = \mathbb{E}_\varphi[\gamma(\varphi)]$ and, with one write per step,
\begin{equation}
\dot f \;=\; \frac{\geff(f) - f}{n}.
\label{eq:dyn}
\end{equation}
\end{definition}

\noindent \textbf{Equation~\eqref{eq:dyn} licenses a per-step statement and this paper's empirical
content is not carried by it.} \eqref{eq:dyn} looks like the model and is not. Its apparent content
is the denominator, a restoring force divided by a store that grows without bound, and we drew three
readings from that denominator, of which \emph{one survives}. The sign of $\geff(f)-f$ gives the
\emph{direction} of drift, which a mean-field reading would use to predict the outcome;
\S\ref{sec:bimodal} measures that four of five models drift downward at $f_0$ and are captured
anyway, so the direction is real and does not locate the boundary. We read the $1/n$ as deciding
\emph{how long} the outcome stays open; Prop.~\ref{prop:step} refutes that, the bound being vacuous
over exactly the interval in which the outcome is decided. We read the same $1/n$ as fixing the
\emph{spread} of outcomes; \S\ref{sec:limitbimodal}'s $n_0$ sweep refutes it at $1.07\times$ against
a required $2.83\times$. What is left is the per-step statement \eqref{eq:step} and the ceiling
\eqref{eq:ceiling}, both exact, both about what the store \emph{can} hold. The empirical
content of this paper is carried by $\gamma$: the existence of two outcomes, which model goes which
way, and the ordering are measurements of $\gamma(\varphi)$ put back into the loop
(\S\ref{sec:bimodal}, \S\ref{sec:gammaout}), none of it derived from \eqref{eq:dyn}. A reader who
deletes \eqref{eq:dyn} loses the notation and none of the results.

\subsection{Formal setting: three propositions}
\label{sec:formal}

Each of the following is elementary, and each is here rather than beside the measurement it governs
because each \emph{constrains what a criterion may say}. Two of this paper's frozen criteria failed
for reasons derivable from these statements before their data existed (\S\ref{sec:sched},
\S\ref{sec:power}), which is the argument for stating them up front. Proofs are in
Appendix~\ref{app:proofs}.

\begin{proposition}[Append-only ceiling]
\label{prop:ceiling}
Let the store hold $n(t)$ records at step $t$, of which $n_{\mathrm{false}}(t)$ are false, and
let no record ever be removed. Then the $n_0 - n_{\mathrm{false}}(0)$ truth records present at
$t{=}0$ are present at every later step, so the false share obeys
\begin{equation}
f(t) \;=\; \frac{n_{\mathrm{false}}(t)}{n(t)} \;\le\;
  \frac{n(t) - \bigl(n_0 - n_{\mathrm{false}}(0)\bigr)}{n(t)} \;<\; 1
\qquad\text{for all finite } t,
\label{eq:ceiling}
\end{equation}
with equality exactly when every write after $t{=}0$ is false. Under a grounding budget of $R$
steps, $R$ of the writes are true by construction, giving the per-run bound
$f_{\max}(R) = (n_{\mathrm{false}}(0) + W)/(n_0 + W + R)$ where $W$ is the number of ungrounded
writes the run actually made.
\end{proposition}

\noindent $f_{\max}(R)$ moves with the budget, so any criterion stated as an absolute
threshold on $f$ silently becomes a different criterion at each $R$. \S\ref{sec:sched} records a
criterion of ours that failed by it.

\begin{proposition}[Log-dilution]
\label{prop:dilution}
A grounded write at step $t$ multiplies $f$ by $n/(n{+}1)$ with $n = n_0 + t$, so $R$ grounded
writes placed at steps $t_1 \ldots t_R$ multiply $f$ by $\prod_i (n_0{+}t_i)/(n_0{+}t_i{+}1)$,
i.e.\ deliver a total log-dilution
\begin{equation}
L \;=\; \sum_{i=1}^{R} \frac{1}{n_0 + t_i} \;+\; O\!\left(\sum_i (n_0+t_i)^{-2}\right),
\qquad
\begin{aligned}
L_{\mathrm{front}} &\approx \ln\!\frac{n_0{+}R}{n_0}, \\
L_{\mathrm{unif}}  &\approx \frac{R}{T}\ln\!\frac{n_0{+}T}{n_0}, \\
L_{\mathrm{back}}  &\approx \ln\!\frac{n_0{+}T}{n_0{+}T{-}R},
\end{aligned}
\label{eq:dilution}
\end{equation}
for the three schedules that place the $R$ writes first, evenly, and last. The three are
strictly ordered $L_{\mathrm{front}} > L_{\mathrm{unif}} > L_{\mathrm{back}}$ for every
$0 < R < T$, with no free parameter.
\end{proposition}

\noindent The ordering is parameter-free and is the schedule axis's claim; the magnitudes are not,
since a bounded observable must compress an unbounded one.

\begin{proposition}[Resolution at $s$ seeds]
\label{prop:se}
If $f_{\mathrm{end}}$ is two-valued with choice probability $p$, then a per-model mean over $s$
seeds has
\begin{equation}
\mathrm{se}(\hat p) \;=\; \sqrt{p(1-p)/s}
\;\;\Longrightarrow\;\;
\mathrm{se} \approx 0.29 \text{ at } s{=}3,
\qquad 0.19 \text{ at } s{=}7 ,
\label{eq:se}
\end{equation}
and for a quantity pooled over $m$ models sharing the same $s$ seeds, with intraclass
correlation $\rho$ over seeds, the effective sample size is $s m / \bigl(1 + (m{-}1)\rho\bigr)$.
\end{proposition}

\noindent \textbf{Two scopings the second half needs, both measured rather than assumed.} It is a
statement about a pooled \emph{level}; a within-seed \emph{contrast} cancels the seed effect and
carries a design effect near one, measured at $1.14$--$1.47$ against the level's $3.56$--$3.66$ on the
same $44$-seed runs (\S\ref{sec:knife}). It writes one $\rho$ for $m$ models: measured pair by pair
at $44$ seeds, $\rho$ runs $+0.18$ to $+0.99$, so the implied design effect runs $1.73$--$4.95$ and
the single figure this paper quotes is a mean over heterogeneous pairs. Every correction's direction
is unchanged and the mean $\rho$ is $+0.53$ to $+0.65$; what is not safe is reading $3.75$ as a
per-pair constant, and \S\ref{sec:power} carries the range.

\begin{proposition}[Per-step bound]
\label{prop:step}
Each step appends exactly one record, so
\begin{equation}
\bigl|f(t{+}1) - f(t)\bigr| \;\le\; \frac{1}{n_0+t+1},
\qquad
\sum_{s=t}^{T-1}\frac{1}{n_0+s+1} \;=\; \log\frac{n_0+T}{n_0+t} + O(n_0^{-1}).
\label{eq:step}
\end{equation}
The cumulative bound is therefore \emph{logarithmic} in the remaining tenure, and it constrains
nothing until it falls below $1$. At this paper's apparatus ($n_0 = 10$, $T = 150$) that happens
at $t = 49$: for the first $49$ steps \eqref{eq:step} permits $f$ to reach either endpoint from
anywhere.
\end{proposition}

\begin{warnbox}
\textbf{This corrects a sentence of ours, which is why the proposition is here.} \S\ref{sec:early}
wrote that the outcome is decided early \emph{because} \eqref{eq:dyn}'s restoring force carries
$1/n$, and called the timescale ``a consequence rather than a caveat''. Prop.~\ref{prop:step} makes
that argument quantitative and does not support it: the bound is vacuous for the first $49$ of $150$
steps, while the measured decision point is $t{=}2$. The early decision is a measurement
($12/15$ at $t{=}2$, no better at $t{=}30$) and not a consequence of the $1/n$ factor, which
gives only the per-step bound. Whatever concentrates the decision into the first few steps is derived
nowhere in this paper, and we say so rather than gesturing at the denominator.
\end{warnbox}

\noindent The second half of Prop.~\ref{prop:se} is the one that caught us. \eqref{eq:se} is
\emph{per model}; the design-effect factor $1+(m{-}1)\rho$ applies to everything pooled across models,
which is most of what this paper reports. \S\ref{sec:power} measures $\rho$ and gets $3.75$ at nine
seeds, so $45$ runs carry the information of $12$; \S\ref{sec:knife} measures it again at forty-four
and separates the case it does and does not apply to.

\begin{scopebox}
\textbf{One convention, adopted after a scoring failure elsewhere in this project.} Correctness
is re-derived from stored raw answer strings at analysis time, never recorded as a boolean during
the run, so fixing a scorer costs an analysis re-run rather than an experiment re-run. It caught
a field-name error in this paper's own figure code, which reported a mean gap of $0.000$ where
the correct value is $+0.383$.
\end{scopebox}

\section{Results}
\label{sec:results}

\subsection{Bimodal outcomes, predicted by $\gamma$}
\label{sec:bimodal}

\noindent Figure~\ref{fig:bimodal} is this subsection in two panels: every run at $f_0{=}0.9$
with the interval between its two modes empty, and each model's capture rate against what its own
measured $\gamma$ predicts; Table~\ref{tab:deadband} is the same outcomes seed by seed.

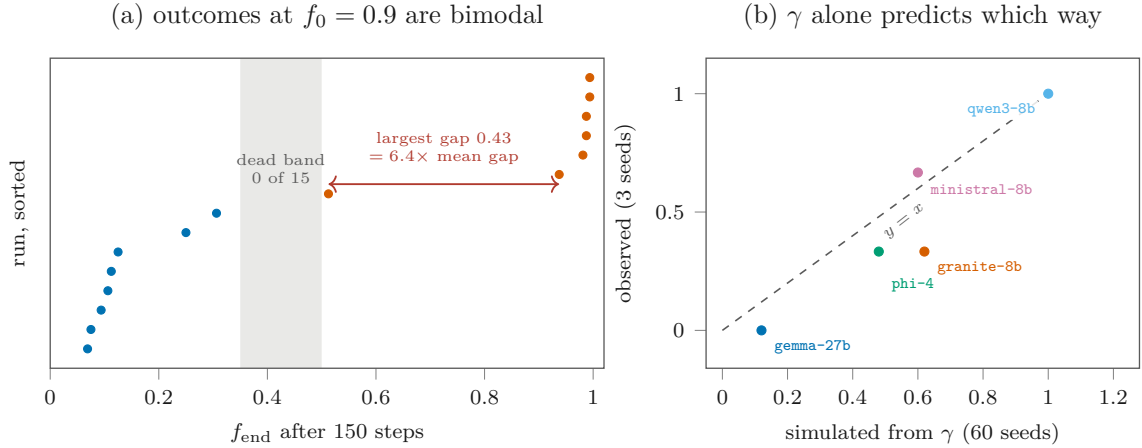
\begin{figure}[!tbp]
\centering
% GENERATED by t1_grounding/figures/make_t1_tikz.py -- do not edit by hand.
% Every coordinate is recomputed from t1_grounding/results/ at generation time.

\begin{tikzpicture}
\begin{axis}[name=pa, width=0.46\linewidth, height=4.1cm, axis on top, tick align=outside, tick pos=left, axis line style={figSoft, line width=0.5pt}, tick label style={font=\scriptsize, color=figInk}, label style={font=\scriptsize, color=figInk}, title style={font=\small, color=figInk, yshift=1pt}, grid=none, scale only axis,
  xmin=0, xmax=1.02, ymin=-1, ymax=15,
  xlabel={$f_{\mathrm{end}}$ after 150 steps}, ylabel={run, sorted},
  ytick=\empty, title={(a) outcomes at $f_0=0.9$ are bimodal}]
\addplot[draw=none, fill=figGrid, fill opacity=0.55, forget plot]
  coordinates {(0.35,-1) (0.50,-1) (0.50,15) (0.35,15)} \closedcycle;
\node[font=\tiny, figSoft, align=center] at (axis cs:0.425,9.3)
  {dead band\\0 of 15};
\addplot[only marks, mark=*, mark size=1.5pt, figA] coordinates {(0.0688,0.0000) (0.0750,1.0000) (0.0938,2.0000) (0.1062,3.0000) (0.1125,4.0000) (0.1250,5.0000) (0.2500,6.0000) (0.3063,7.0000)};
\addplot[only marks, mark=*, mark size=1.5pt, figB] coordinates {(0.5125,8.0000) (0.9375,9.0000) (0.9812,10.0000) (0.9875,11.0000) (0.9875,12.0000) (0.9938,13.0000) (0.9938,14.0000)};
\draw[figPred, <->, line width=0.7pt] (axis cs:0.5125,8.5)
  -- (axis cs:0.9375,8.5);
\node[font=\tiny, figPred, align=center, anchor=south]
  at (axis cs:0.7250,9.0)
  {largest gap 0.43\\$=6.4\times$ mean gap};
\end{axis}
\begin{axis}[name=pb, at={(pa.east)}, xshift=1.35cm, anchor=west,
  width=0.36\linewidth, height=4.1cm, axis on top, tick align=outside, tick pos=left, axis line style={figSoft, line width=0.5pt}, tick label style={font=\scriptsize, color=figInk}, label style={font=\scriptsize, color=figInk}, title style={font=\small, color=figInk, yshift=1pt}, grid=none, scale only axis,
  xmin=-0.05, xmax=1.28, ymin=-0.16, ymax=1.15,
  xlabel={simulated from $\gamma$ (60 seeds)}, ylabel={observed (3 seeds)},
  title={(b) $\gamma$ alone predicts which way}]
\addplot[figSoft, dashed, line width=0.6pt, forget plot] coordinates {(0,0) (1,1)};
\node[font=\tiny, figSoft, rotate=38] at (axis cs:0.56,0.46) {$y=x$};
\addplot[only marks, mark=*, mark size=1.7pt, figA] coordinates {(0.120,0.000)};
\node[font=\tiny, align=left, fill=white, fill opacity=0.85, text opacity=1, inner sep=1pt, text=figA, anchor=west] at (axis cs:0.150,-0.070) {\texttt{gemma-27b}};
\addplot[only marks, mark=*, mark size=1.7pt, figC] coordinates {(0.480,0.333)};
\node[font=\tiny, align=left, fill=white, fill opacity=0.85, text opacity=1, inner sep=1pt, text=figC, anchor=west] at (axis cs:0.510,0.193) {\texttt{phi-4}};
\addplot[only marks, mark=*, mark size=1.7pt, figB] coordinates {(0.620,0.333)};
\node[font=\tiny, align=left, fill=white, fill opacity=0.85, text opacity=1, inner sep=1pt, text=figB, anchor=west] at (axis cs:0.650,0.263) {\texttt{granite-8b}};
\addplot[only marks, mark=*, mark size=1.7pt, figD] coordinates {(0.600,0.667)};
\node[font=\tiny, align=left, fill=white, fill opacity=0.85, text opacity=1, inner sep=1pt, text=figD, anchor=west] at (axis cs:0.630,0.597) {\texttt{ministral-8b}};
\addplot[only marks, mark=*, mark size=1.7pt, figE] coordinates {(1.000,1.000)};
\node[font=\tiny, align=left, fill=white, fill opacity=0.85, text opacity=1, inner sep=1pt, text=figE, anchor=east] at (axis cs:0.970,0.930) {\texttt{qwen3-8b}};
\end{axis}
\end{tikzpicture}
\caption{\textbf{The outcome is a coin flip, and the copy function says which way it is
weighted.} (a) Every run at $f_0{=}0.9$, sorted: eight end below $0.31$, seven above
$0.51$, and the interval between them is empty. (b) Replacing the model by its own measured
$\gamma$ and simulating the same loop lands within one seed of each model's capture rate with
nothing fitted, Spearman $+0.82$, one adjacent inversion at a simulated gap of $0.02$.}
\label{fig:bimodal}
\end{figure}

\begin{table}[!htbp]
\centering\small
\caption{Outcomes at $f_0{=}0.9$, per seed. The quantity a decaying-mean account would
summarise has an empty neighbourhood around its own mean.}
\label{tab:deadband}
% GENERATED by t1_grounding/figures/make_t1_tikz.py from results/converge.json.
% The band's edges are the runs either side of the largest gap, not constants.
\begin{tabular}{@{}lp{0.52\linewidth}c@{}}
\toprule
outcome & $f_{\mathrm{end}}$ per run & count \\
\midrule
escaped & 0.069 \; 0.075 \; 0.094 \; 0.106 \; 0.113 \; 0.125 \; 0.250 \; 0.306 & 8/15 \\
\emph{dead band} $[0.35,0.50]$ & --- & \textbf{0/15} \\
captured & 0.512 \; 0.938 \; 0.981 \; 0.988 \; 0.988 \; 0.994 \; 0.994 & 7/15 \\
\midrule
\multicolumn{2}{@{}l}{largest gap in the sorted outcomes / mean gap}
  & \textbf{6.4}$\times$ \\
\bottomrule
\end{tabular}

\end{table}

\paragraph{The model's whole contribution is $\gamma$, to within one seed.} Replacing the agent with a
rule that asserts a false value with probability $\gamma(j)$ (the measured curve, $j$ the number of
false records retrieved) and simulating the same loop puts each model's capture rate within one seed
of the observed count (Table~\ref{tab:gammasim}):

\begin{table}[!htbp]
\centering\small
\caption{Simulated from the measured $\gamma$ alone, $60$ seeds, against the observed
$3$-seed counts. No parameter is fitted; the curve is measured in situ during the same
runs.}
\label{tab:gammasim}
\begin{tabular}{@{}lcc@{}}
\toprule
model & capture rate simulated from $\gamma$ & observed \\
\midrule
\texttt{gemma-3-27b}    & 0.12 & 0/3 \\
\texttt{phi-4}          & 0.48 & 1/3 \\
\texttt{granite-4.1-8b} & 0.62 & 1/3 \\
\texttt{ministral-8b}   & 0.60 & 2/3 \\
\texttt{qwen3-8b}       & 1.00 & 3/3 \\
\bottomrule
\end{tabular}
\end{table}

\begin{claimbox}
Every observed count lies inside the binomial range of its simulated probability, and every simulated
rate lands within one seed ($1/3$) of the count observed; the largest gap is \texttt{granite} at
$0.29$, or $0.86$ seed units. What the language model brings to this loop is one function of one
variable, and the bimodality and the timescale are arithmetic on top of it. \emph{The ordering is
reproduced with one exception, which we state rather than sort away}: $\gamma$ puts \texttt{granite}
($0.62$) marginally above \texttt{ministral} ($0.60$) while the runs reverse them. The simulated gap
is $0.02$, which three seeds cannot resolve either way, so the defensible summary at three seeds is
Spearman $+0.82$ with one adjacent inversion, not exact reproduction.
\end{claimbox}

\subsubsection{Two objections to the empty interval}
\label{sec:answerspace}

The apparatus plants one false value and corrupts one digit of it, so the store's realised answer
space is small and a reader is entitled to ask whether the empty interval is a fact about a
self-writing loop or arithmetic on a two-valued outcome. The question has two halves with different
answers, both settled without an API call.

\paragraph{A two-valued answer space does not empty the interval.} By construction
$f_{\mathrm{end}} = (n_{\mathrm{false},0} + \#\{\text{false writes}\})/(n_0+T)$, so a reader
whose copy probability does not depend on the store lands where that constant puts it, in the
interior. Running the identical urn with the model replaced by a coin of fixed bias puts $113$ of
$300$ runs inside the dead band at $p{=}0.3$ and $82$ of $300$ at $p{=}0.5$
(Table~\ref{tab:constp}). \emph{What empties the interval is the copy probability rising with the
number of false records retrieved}, a property of the reader and the object $\gamma(j)$ measures. The
two-valued space is not sufficient, so it is not the explanation.

\paragraph{The premise is not uniformly true either, and our two arms disagree about what
follows.} E-bimodal2's runs log, step by step, whether each write was the truth, the planted
falsehood or a value the store had not seen. Writing $q$ for the share of non-truth writes that
took a fresh value, the median run has $q{=}0$, while \texttt{qwen3-8b} has a median of $43$
novel-value \emph{writes} in $150$ steps, $q = 0.33$, and $0$ of its $44$ runs in the dead band;
across all $220$ runs the $26$ with twenty or more novel writes have $0$ in the band and the $161$
with none have $6$ (Figure~\ref{fig:answerspace}(a)). That log counts \emph{writes} and not
distinct values, which matters, because a novel value once written is retrievable and can be
reinforced.

E-real resolves the ambiguity and points the other way. Its planted falsehood is \emph{another real
capital} rather than a one-digit corruption (\texttt{realfacts.py}), its models produce novel errors
at a pooled $q$ of $0.26$, and it logs the number of distinct values in each retrieved window. On the
$360$ real-fact runs at $f_0{=}0.9$ the runs with five or more novel writes carry a mean of $2.51$
distinct retrieved values against $1.66$ for the rest, so their stores really did go multi-valued,
and they hold $5$ of $41$ runs in the dead band against $6$ of $311$, a factor of $6.4$. Within
\texttt{granite-4.1-8b} alone it is $4/39$ against $1/141$. The registered criterion P-T1-N6 still
passes, at $0.014$ overall against the synthetic control's $0.125$; what does not survive is the
reading that a richer answer space leaves the interval empty.

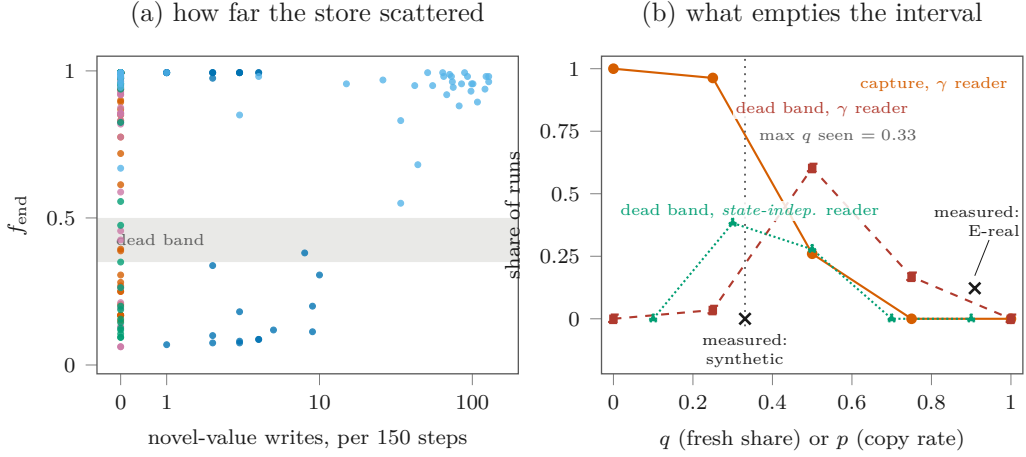
\begin{figure}[!tbp]
\centering
% GENERATED by t1_grounding/figures/make_t1_tikz.py -- do not edit by hand.
% Every coordinate is recomputed from t1_grounding/results/ at generation time.

\begin{tikzpicture}
\begin{axis}[name=pa, width=0.35\linewidth, height=4.2cm, axis on top, tick align=outside, tick pos=left, axis line style={figSoft, line width=0.5pt}, tick label style={font=\scriptsize, color=figInk}, label style={font=\scriptsize, color=figInk}, title style={font=\small, color=figInk, yshift=1pt}, grid=none, scale only axis,
  xmode=log, xmin=0.35, xmax=200, ymin=-0.03, ymax=1.05,
  xtick={0.5,1,10,100}, xticklabels={$0$,$1$,$10$,$100$},
  xlabel={novel-value writes, per $150$ steps},
  ylabel={$f_{\mathrm{end}}$},
  title={(a) how far the store scattered}]
\addplot[draw=none, fill=figGrid, fill opacity=0.55, forget plot] coordinates
  {(0.35,0.35) (200,0.35) (200,0.50) (0.35,0.50)} \closedcycle;
\node[font=\tiny, figSoft, anchor=west] at (axis cs:0.4,0.425) {dead band};
\addplot[only marks, mark=*, mark size=1.1pt, figA, opacity=0.8] coordinates {(3.000,0.994) (3.000,0.994) (4.000,0.994) (0.500,0.994) (1.000,0.994) (3.000,0.075) (2.000,0.994) (0.500,0.994) (0.500,0.994) (0.500,0.994) (9.000,0.113) (0.500,0.994) (0.500,0.994) (5.000,0.119) (1.000,0.994) (10.000,0.306) (4.000,0.087) (0.500,0.994) (2.000,0.075) (2.000,0.994) (2.000,0.100) (1.000,0.994) (8.000,0.381) (3.000,0.181) (2.000,0.975) (0.500,0.825) (0.500,0.994) (2.000,0.994) (0.500,0.994) (1.000,0.069) (0.500,0.994) (0.500,0.994) (4.000,0.087) (9.000,0.200) (3.000,0.994) (0.500,0.994) (0.500,0.994) (3.000,0.994) (3.000,0.081) (2.000,0.338) (1.000,0.994) (0.500,0.994) (4.000,0.994) (1.000,0.994)};
\addplot[only marks, mark=*, mark size=1.1pt, figB, opacity=0.8] coordinates {(0.500,0.994) (0.500,0.994) (0.500,0.975) (0.500,0.994) (0.500,0.938) (0.500,0.250) (0.500,0.831) (0.500,0.994) (0.500,0.994) (0.500,0.994) (0.500,0.281) (0.500,0.994) (0.500,0.994) (0.500,0.994) (0.500,0.994) (0.500,0.894) (0.500,0.169) (0.500,0.994) (0.500,0.169) (0.500,0.994) (0.500,0.250) (0.500,0.994) (0.500,0.394) (0.500,0.169) (0.500,0.306) (0.500,0.869) (0.500,0.169) (0.500,0.775) (0.500,0.938) (0.500,0.900) (0.500,0.994) (0.500,0.719) (0.500,0.150) (0.500,0.206) (0.500,0.994) (0.500,0.613) (0.500,0.994) (0.500,0.850) (0.500,0.388) (0.500,0.269) (0.500,0.994) (0.500,0.994) (0.500,0.994) (0.500,0.994)};
\addplot[only marks, mark=*, mark size=1.1pt, figD, opacity=0.8] coordinates {(0.500,0.994) (0.500,0.994) (0.500,0.994) (0.500,0.994) (0.500,0.994) (0.500,0.425) (0.500,0.856) (0.500,0.919) (0.500,0.994) (0.500,0.994) (0.500,0.588) (0.500,0.775) (0.500,0.856) (0.500,0.212) (0.500,0.919) (0.500,0.988) (0.500,0.875) (0.500,0.994) (0.500,0.131) (0.500,0.994) (0.500,0.194) (0.500,0.994) (0.500,0.062) (0.500,0.831) (0.500,0.819) (0.500,0.456) (0.500,0.994) (0.500,0.963) (0.500,0.994) (0.500,0.931) (0.500,0.969) (0.500,0.994) (0.500,0.062) (0.500,0.994) (0.500,0.994) (0.500,0.925) (0.500,0.994) (0.500,0.944) (0.500,0.263) (0.500,0.994) (0.500,0.994) (0.500,0.850) (0.500,0.856) (0.500,0.994)};
\addplot[only marks, mark=*, mark size=1.1pt, figC, opacity=0.8] coordinates {(0.500,0.994) (0.500,0.994) (0.500,0.994) (0.500,0.994) (0.500,0.994) (0.500,0.119) (0.500,0.994) (0.500,0.950) (0.500,0.994) (0.500,0.956) (0.500,0.125) (0.500,0.956) (0.500,0.938) (0.500,0.994) (0.500,0.994) (0.500,0.188) (0.500,0.094) (0.500,0.994) (0.500,0.156) (0.500,0.994) (0.500,0.094) (0.500,0.994) (0.500,0.825) (0.500,0.350) (0.500,0.475) (0.500,0.556) (0.500,0.200) (0.500,0.994) (0.500,0.994) (0.500,0.994) (0.500,0.994) (0.500,0.994) (0.500,0.263) (0.500,0.106) (0.500,0.975) (0.500,0.994) (0.500,0.994) (0.500,0.994) (0.500,0.975) (0.500,0.144) (0.500,0.994) (0.500,0.994) (0.500,0.994) (0.500,0.994)};
\addplot[only marks, mark=*, mark size=1.1pt, figE, opacity=0.8] coordinates {(0.500,0.975) (123.000,0.981) (4.000,0.981) (128.000,0.963) (74.000,0.963) (75.000,0.944) (1.000,0.994) (121.000,0.938) (0.500,0.988) (0.500,0.944) (98.000,0.931) (3.000,0.850) (51.000,0.994) (73.000,0.981) (89.000,0.994) (34.000,0.831) (65.000,0.981) (0.500,0.963) (101.000,0.956) (0.500,0.994) (85.000,0.956) (55.000,0.950) (34.000,0.550) (0.500,0.669) (26.000,0.969) (99.000,0.956) (0.500,0.956) (68.000,0.919) (0.500,0.994) (71.000,0.988) (0.500,0.994) (0.500,0.981) (44.000,0.681) (108.000,0.894) (0.500,0.994) (0.500,0.981) (128.000,0.981) (0.500,0.994) (0.500,0.988) (93.000,0.981) (42.000,0.950) (82.000,0.881) (64.000,0.994) (15.000,0.956)};
% The label that stood here said "qwen3-8b: median 43, 0/44 in the band", which the caption
% says in its own words two lines below: "qwen3-8b writes a median of 43 while still being
% captured on 44 of 44 seeds with nothing in the dead band". Its ground sat on the series.
\end{axis}
\begin{axis}[name=pb, at={(pa.east)}, xshift=1.1cm, anchor=west,
  width=0.35\linewidth, height=4.2cm, axis on top, tick align=outside, tick pos=left, axis line style={figSoft, line width=0.5pt}, tick label style={font=\scriptsize, color=figInk}, label style={font=\scriptsize, color=figInk}, title style={font=\small, color=figInk, yshift=1pt}, grid=none, scale only axis,
  xmin=-0.03, xmax=1.03, ymin=-0.22, ymax=1.05, ytick={0,0.25,0.5,0.75,1},
  xlabel={$q$ (fresh share) or $p$ (copy rate)},
  ylabel={share of runs}, title={(b) what empties the interval}]
\addplot[figB, mark=*, mark size=1.6pt, line width=0.8pt] coordinates
  {(0.000,1.000) (0.250,0.963) (0.500,0.260) (0.750,0.000) (1.000,0.000)};
\addplot[figPred, mark=square*, mark size=1.6pt, line width=0.8pt, dashed] coordinates
  {(0.000,0.000) (0.250,0.035) (0.500,0.603) (0.750,0.168) (1.000,0.000)};
\addplot[figC, mark=triangle*, mark size=1.9pt, line width=0.8pt, densely dotted] coordinates
  {(0.100,0.000) (0.300,0.383) (0.500,0.278) (0.700,0.000) (0.900,0.000)};
\draw[figSoft, dotted, line width=0.7pt] (axis cs:0.331,-0.03)
  -- (axis cs:0.331,1.05);
\node[font=\tiny, figSoft, anchor=south west] at (axis cs:0.341,0.66)
  {max $q$ seen $=0.33$};
\node[font=\tiny, align=left, fill=white, fill opacity=0.85, text opacity=1, inner sep=1pt, text=figB, anchor=north east] at (axis cs:1.0,0.98) {capture, $\gamma$ reader};
\node[font=\tiny, align=left, fill=white, fill opacity=0.85, text opacity=1, inner sep=1pt, text=figPred, anchor=north west] at (axis cs:0.30,0.88)
  {dead band, $\gamma$ reader};
\node[font=\tiny, align=left, fill=white, fill opacity=0.85, text opacity=1, inner sep=1pt, text=figC, anchor=north west] at (axis cs:0.01,0.47)
  {dead band, \emph{state-indep.} reader};
% the two MEASURED points, laid on the scripted curve they agree and disagree with
\addplot[only marks, mark=x, mark size=3pt, line width=1pt, figInk] coordinates
  {(0.331,0.000) (0.909,0.122)};
\node[font=\tiny, align=left, fill=white, fill opacity=0.85, text opacity=1, inner sep=1pt, text=figInk, anchor=north, align=center] at (axis cs:0.331,-0.05)
  {measured:\\synthetic};
\node[font=\tiny, align=left, fill=white, fill opacity=0.85, text opacity=1, inner sep=1pt, text=figInk, anchor=south east, align=right]
  at (axis cs:1.03,0.32) {measured:\\E-real};
\draw[figInk, line width=0.35pt] (axis cs:0.945,0.315)
  -- (axis cs:0.909,0.172);
\end{axis}
\end{tikzpicture}
\caption{\textbf{What the empty interval needs, and what it does not.} (a) Every E-bimodal2 run
by the number of novel-value \emph{writes} it made; the log counts writes, not distinct values.
The three models that never scatter sit on the left edge and \texttt{qwen3-8b} writes a median of
$43$ while still being captured on $44$ of $44$ seeds with nothing in the dead band. (b) Scripted
controls, with the two measured points laid on them. The state-independent reader (dotted) fills
the band, so bimodality is not arithmetic on a two-valued space. The $\gamma$ reader deciding on
the largest single-value block (dashed) keeps the band empty to $q\approx0.25$ and refills it at
$q{=}0.5$. The two crosses are measured: E-bimodal2's highest-$q$ model, which contradicts that
curve, and E-real's multi-valued subset, which follows it. \emph{The arms disagree and the
disagreement is the result}; they differ in whether a novel value is another variant of the same
numeral or a different city, which is why \S\ref{sec:answerspace} relocates the open question from
count to distinguishability.}
\label{fig:answerspace}
\end{figure}

\begin{claimbox}
\textbf{What this buys, what it concedes, and what the arm we then bought settled.} It refutes the
strong form of the objection: the interval is not empty by arithmetic, and a state-independent reader
on the same two-valued urn fills it. It concedes the weak form as a question, which the two arms
above answer differently: on the synthetic world the highest-$q$ model has an empty band, on real
facts the multi-valued runs have $6.4\times$ the band occupancy of the two-valued ones. \emph{We
registered an arm and ran it}: three cells at $v{=}3$ planted falsehoods differing only in how far
apart they are, $300$ runs at \$2.45 (\S\ref{sec:multivalue}). It settles less than we hoped and more
than we expected: neither the \emph{number} of planted falsehoods nor their
\emph{distinguishability} moves the interval, and the reason is a measurement that refutes the
scripted reader this paragraph was built on. The scripted sweep above assumed the copy probability
responds to the largest single-value block; held out over $44{,}989$ steps, $\gamma$ as a function of
$\varphi$ beats $\gamma$ as a function of that block by $1491$ in log-likelihood. \textbf{No
scattering mechanism can act through $\gamma$, because $\gamma$ does not see scattering.} The
largest-block reading is refuted; the two-observable disagreement stays \textsc{split} with two
candidate explanations removed, and Appendix~\ref{app:answerspace} keeps both sweeps and both
splits.
\end{claimbox}

\subsubsection{The arm bought to settle it}
\label{sec:multivalue}

The question left open above is whether the empty interval needs the store's falsehoods to be few,
or indistinguishable, or neither. It is not answerable from data on disk, every arm this paper had
planting exactly one falsehood. It was therefore registered
(\texttt{prereg/E\_MULTIVALUE\_PREREG.md}), amended twice \emph{before} any spend, and run: three
cells, twenty seeds, five models, $300$ runs, \$2.45 against a hard stop of \$7.30. Each cell plants
three falsehoods differing only in how far apart those three are: \texttt{near} uses the published
one-digit corruption three times over, \texttt{far} uses different magnitudes plus a leading-digit
change, and \texttt{sem} uses three different real capitals. \texttt{near} against \texttt{far} is
the load-bearing contrast and holds the world fixed; \texttt{sem} changes the world too and is a
bridge to \S\ref{sec:real}, which is that cell at one falsehood.

\begin{warnbox}
\textbf{Both graded criteria came out against the hypothesis that motivated the arm, and one of
them came out against the prereg's own first version.} P-T1-V1$''$ asked for \texttt{far}'s
share of runs in the dead band to exceed \texttt{near}'s by $0.10$. It is
$-0.030$: \textsc{failed}, and if anything the more distinguishable falsehoods are captured more
often ($0.970$ against $0.870$). Paired over the twenty seeds the two cells share, \texttt{far}
is fuller on $0$ and \texttt{near} on $3$, at a sign-test $p$ of $0.25$ where twenty seeds could
have reached $2\times10^{-6}$. P-T1-V2$''$ asked whether the \emph{count} matters, and predicted
it does not: \texttt{near} sits $-0.002$ from the published one-falsehood arm, inside a band of
$\pm 0.10$, \textsc{supported}. Amendment 2 wrote that prediction down against the
original prereg's opposite one, so the count question is settled and it is settled against the
version of the arm we first froze.
\end{warnbox}

\noindent Two measurements explain the pair of nulls, and neither was predicted.

\begin{claimbox}
\textbf{First, the store re-concentrates, so an initial scatter is transient.} The three planted
falsehoods give a largest single-value block of $3.00$ at the first step and $3.75$ distinct
values in the retrieved window. By the end of the tenure the block is $5.08$ (\texttt{near}) and
$5.54$ (\texttt{far}) while the distinct count has fallen to $2.65$ and $2.57$: whichever
falsehood is copied first grows, and the rest are diluted. \textbf{Second, and this is the
measurement that matters, $\gamma$ does not see scattering at all.} Both $\varphi = j/k$ and the
largest single-value block $j^*$ are one-dimensional lookup tables over $0 \ldots k$, so they can
be compared on equal terms; fitted per model, held out per seed, over $44{,}989$ steps, $\varphi$
wins by $1491$ in log-likelihood. Def.~\ref{def:gamma}'s independent variable is the right one,
which is a result in the primitive's favour, and it is what makes the two nulls inevitable
rather than surprising: a reader that responds to the count of non-truth records cannot respond
to how those records disagree with one another. \emph{The scripted largest-block reader of
\S\ref{sec:answerspace} is therefore a refuted model of this apparatus's reader, and no bound drawn
from it survives.}
\end{claimbox}

\noindent What the arm does not settle is \S\ref{sec:answerspace}'s disagreement; it narrows rather
than closes it. \texttt{sem} lands at $0.010$, beside \S\ref{sec:real}'s one-falsehood $0.019$ and
away from \texttt{far}'s $0.000$, so the world carries an effect distinguishability does not.
E-real's correlation is not reverse causation either: the band occupancy of runs producing a novel
error in their \emph{first three steps} is already $4.0\times$ that of runs producing none, so it is
present before the outcome is. Six of the eleven real-fact runs in the band produce no novel error at
all, however, and with $\gamma$ measured blind to scattering no mechanism is left on the table. Row
1k stays \textsc{split} with the count and the distinguishability of planted falsehoods excluded, two
candidates fewer than before the money was spent.

\subsubsection{The cost of deciding the disagreement}
\label{sec:decidable}

The paragraph above reports two arms that disagree and declines to pick one. That is honest and not
a complete answer, because two different questions are folded into the disagreement and only
one is about sample size. Both are answered here in arithmetic rather than adjectives, so a reader
pressing the point can see the number and check it (Table~\ref{tab:decidable}).

\begin{table}[!htbp]\centering\small
\caption{\textbf{What it would cost to decide the within-arm half of the disagreement, and what it
already costs to state it honestly.} Every row is recomputed from
\texttt{results/ANSWERSPACE.json} at build time. The two intervals are the same contrast under two
treatments of the same clustering, and \emph{they do not agree}: the rule of thumb spans zero and
the resampling does not.}
\label{tab:decidable}
% GENERATED by t1_grounding/figures/make_t1_tikz.py -- do not edit by hand.
% Recomputed from results/ANSWERSPACE.json at build time.
\begin{tabular}{@{}p{0.60\linewidth}r@{}}
\toprule
quantity & value \\
\midrule
difference in band occupancy, multi-valued minus two-valued & $0.1027$ \\
unclustered standard error, and its $z$ & $0.0517$, $z = 1.99$ \\
clusters: Wikidata facts $\times$ runs each & $36\times10$ \\
fact-level $\mathrm{ICC}$, and the design effect it implies & $0.077$, $1.69$ \\
$95\%$ interval, $\sqrt{\mathrm{deff}}$ rule of thumb & $[-0.0292,\,0.2345]$ \textbf{spans} $0$ \\
$95\%$ interval, facts resampled with replacement & $[0.0055,\,0.2412]$ \textbf{excludes} $0$ \\
bootstrap mass at or below $0$ & $0.0201$ \\
runs per group for $80\%$ power, independent & $97$ \\
\ldots and with the measured design effect & $164$ \\
multi-valued share of runs, so total runs needed & $0.116$, $1405$ \\
\ldots which is this many facts, against the corpus we have & $\mathbf{141}$ against $36$ \\
\bottomrule
\end{tabular}

\end{table}

\paragraph{Question one, within E-real, is decided, and only just, which is the part worth
printing.} Runs whose store went multi-valued hold the dead band at $0.1220$ against $0.0193$ for
the two-valued ones, a difference of $0.1027$. The runs are \emph{not} independent: the $360$ runs
at $f_0{=}0.9$ sit on $36$ Wikidata facts, ten runs each, and the fact-level intraclass correlation
is $0.077$, so the design effect is $1.69$ and any interval computed as though the runs were
independent is anti-conservative. \textbf{Two ways of accounting for that disagree, and we report the
one that is a measurement while naming the one that is a rule.} Scaling the standard error by
$\sqrt{\mathrm{deff}}$, the textbook correction, widens the interval to $[-0.0292,\,0.2345]$, which
spans zero. Resampling the $36$ \emph{facts} with replacement, $20{,}000$ draws, gives
$[0.0055,\,0.2412]$, which does not, with $2.0\%$ of the draws at or below zero. A design effect is a
rule of thumb applied to a variance; a cluster bootstrap is the sampling distribution of the
statistic actually computed, so the bootstrap decides. We print both so a reader who applies
the rule of thumb finds the discrepancy here rather than in a review, and because $2.0\%$ is not a
comfortable margin: this contrast is decided at the edge of its own resolution, and one more fact
behaving like the two that carry most of the band mass would undo it.

\paragraph{Question two, between the arms, is not a power question at all and no corpus size
settles it.} E-bimodal2's highest-scattering model has an \emph{empty} band on $44$ of $44$ seeds;
E-real's multi-valued subset has the \emph{fullest} one. The arms disagree not about a rate but about
what a novel value \emph{is}: another variant of one numeral in the synthetic world, a different city
on the real axis. \textbf{An item count cannot decide a construct difference}, which is why the arm
this paper bought manipulated \emph{distinguishability} rather than buying more items, and why it
returned the removal of two candidate mechanisms rather than a verdict (\S\ref{sec:multivalue}). The
criterion companion's rule applies to us in the direction that costs us: a criterion whose axis
cannot express the difference under test is not repaired by a larger $n$ \citep{t5rerun2026}.

\paragraph{The price of the replication that would decide question one.} At the observed rates,
$80\%$ power at two-sided $0.05$ needs $97$ runs per group if the runs were independent and $164$
with the measured design effect. Only $0.116$ of runs go multi-valued, and multi-valuedness is an
\emph{outcome} rather than an assignment (we cannot allocate to it), so reaching $164$
multi-valued runs takes about $1{,}405$ runs in total, at ten runs per fact $\mathbf{141}$ real facts
against the $\mathbf{36}$ this paper has. That is a fourfold corpus and not a re-analysis,
which is the point of this subsection: the external-validity objection to this paper is correct, it
has a price, and the price is written down where a reader can check the arithmetic instead of asking
for it.

\subsubsection{What one run looks like}
\label{sec:casestudy}

A distribution over endpoints hides the sequence of writes that produces it, so before the retest we
read two runs step by step. They are chosen by rule rather than picked: the same model's lowest and
highest $f_{\mathrm{end}}$ among the forty-four fresh seeds.

% GENERATED by t1_grounding/make_casestudy.py -- do not edit by hand.
\begin{table}[!ht]
\centering\small
\caption{\textbf{Two runs of \texttt{gemma-3-27b-it}, same configuration, same starting store, opposite outcomes, and the divergence is the \emph{first write}.} Chosen by rule rather than picked: the lowest and highest $f_{\mathrm{end}}$ among this model's $44$ fresh seeds. $j$ is how many of the $k{=}8$ retrieved records asserted the false value; \emph{write} is what the agent then asserted. At step $0$ both runs retrieve $j{=}7$ and $j{=}7$ of $8$, effectively the same evidence, and answer differently. Nothing external distinguishes them: the seed selects the topic and the false value, and from there the store's own composition does the rest. 28 of this model's runs tie exactly at the ceiling and 1 at the floor value shown; ties are broken by lowest seed so the table is reproducible, and no run was chosen for what it shows.}
\label{tab:casestudy}
\begin{tabular}{@{}r|cccl|cccl@{}}
\toprule
& \multicolumn{4}{c|}{escaped, $f_{\mathrm{end}}=0.069$ (seed $39$)}& \multicolumn{4}{c}{captured, $f_{\mathrm{end}}=0.994$ (seed $52$)} \\
step & $n$ & $f$ & $j$ & write & $n$ & $f$ & $j$ & write \\
\midrule
$0$ & $11$ & $0.818$ & $7$ & truth & $11$ & $0.909$ & $7$ & \textbf{false} \\
$1$ & $12$ & $0.750$ & $6$ & truth & $12$ & $0.917$ & $7$ & \textbf{false} \\
$2$ & $13$ & $0.769$ & $7$ & \textbf{false} & $13$ & $0.923$ & $8$ & \textbf{false} \\
$3$ & $14$ & $0.714$ & $6$ & truth & $14$ & $0.929$ & $7$ & \textbf{false} \\
$20$ & $31$ & $0.323$ & $4$ & truth & $31$ & $0.968$ & $8$ & \textbf{false} \\
$21$ & $32$ & $0.312$ & $3$ & truth & $32$ & $0.969$ & $8$ & \textbf{false} \\
\multicolumn{9}{@{}c@{}}{\ldots} \\
$148$ & $159$ & $0.069$ & $1$ & truth & $159$ & $0.994$ & $8$ & \textbf{false} \\
$149$ & $160$ & $0.069$ & $2$ & truth & $160$ & $0.994$ & $8$ & \textbf{false} \\
\midrule
\multicolumn{1}{@{}l}{\emph{all $150$}} & \multicolumn{4}{l|}{truth $148$, false $1$, novel $1$} & \multicolumn{4}{l}{truth $0$, false $146$, novel $4$} \\
\bottomrule
\end{tabular}
\end{table}

\noindent Two things in the bottom row are worth naming. First, neither run is mixed: one wrote the truth 148 times of $150$ and the other wrote the seeded falsehood 146 times, which is what a bimodal endpoint distribution looks like from the inside, not a drift, a commitment. Second, novel errors (answers that are neither the truth nor the planted falsehood) number 5 across these $300$ steps. That matters for \S\ref{sec:gammaout}: \texttt{baseline\_urn.py} justifies its scripted $\gamma$ reader by noting that the reader asserts only the seeded false value while real runs also produce novel errors, so the reader ``should if anything capture more easily''. On these two runs there is almost nothing for that argument to be about, which is one reason the argument did not predict the direction of the simulation's actual bias.

\FloatBarrier

\noindent The same knife is turned on this section in Appendix~\ref{app:moved}: the
criterion that condemns a pooled mean is applied to our own pooled numbers, and it condemns
some of them.

\subsubsection{Claim 2 out of sample, at a \texorpdfstring{$14\times$}{14x} harder bar}
\label{sec:gammaout}

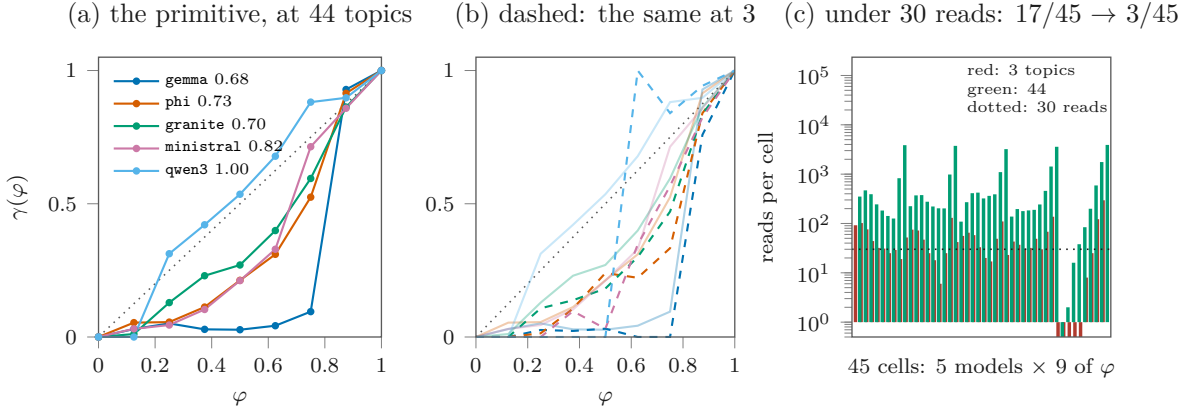
\begin{figure}[!tbp]
\centering
% GENERATED by t1_grounding/figures/make_t1_tikz.py -- do not edit by hand.
% Every coordinate is recomputed from t1_grounding/results/ at generation time.

\begin{tikzpicture}
\begin{axis}[name=pa, width=0.235\linewidth, height=3.7cm, axis on top, tick align=outside, tick pos=left, axis line style={figSoft, line width=0.5pt}, tick label style={font=\scriptsize, color=figInk}, label style={font=\scriptsize, color=figInk}, title style={font=\small, color=figInk, yshift=1pt}, grid=none, scale only axis,
  xmin=0, xmax=1, ymin=0, ymax=1.05,
  xlabel={$\varphi$}, ylabel={$\gamma(\varphi)$},
  title={(a) the primitive, at 44 topics},
  legend style={font=\tiny, draw=none, fill=none, at={(0.02,0.98)}, anchor=north west,
    row sep=-1pt, legend cell align=left}]
\addplot[figSoft, dotted, line width=0.6pt, forget plot] coordinates {(0,0) (1,1)};
\addplot[figA, mark=*, mark size=1.0pt, line width=0.8pt] coordinates {(0.0000,0.0000) (0.1250,0.0299) (0.2500,0.0510) (0.3750,0.0287) (0.5000,0.0273) (0.6250,0.0423) (0.7500,0.0952) (0.8750,0.9287) (1.0000,1.0000)};
\addplot[figB, mark=*, mark size=1.0pt, line width=0.8pt] coordinates {(0.0000,0.0000) (0.1250,0.0539) (0.2500,0.0559) (0.3750,0.1123) (0.5000,0.2127) (0.6250,0.3103) (0.7500,0.5248) (0.8750,0.9145) (1.0000,1.0000)};
\addplot[figC, mark=*, mark size=1.0pt, line width=0.8pt] coordinates {(0.0000,0.0000) (0.1250,0.0111) (0.2500,0.1290) (0.3750,0.2299) (0.5000,0.2702) (0.6250,0.3994) (0.7500,0.5949) (0.8750,0.8626) (1.0000,1.0000)};
\addplot[figD, mark=*, mark size=1.0pt, line width=0.8pt] coordinates {(0.0000,0.0000) (0.1250,0.0306) (0.2500,0.0449) (0.3750,0.1033) (0.5000,0.2116) (0.6250,0.3292) (0.7500,0.7140) (0.8750,0.8584) (1.0000,1.0000)};
\addplot[figE, mark=*, mark size=1.0pt, line width=0.8pt] coordinates {(0.0000,0.0000) (0.1250,0.0000) (0.2500,0.3125) (0.3750,0.4211) (0.5000,0.5357) (0.6250,0.6784) (0.7500,0.8814) (0.8750,0.8976) (1.0000,1.0000)};
\legend{{\texttt{gemma}} 0.68, {\texttt{phi}} 0.73, {\texttt{granite}} 0.70, {\texttt{ministral}} 0.82, {\texttt{qwen3}} 1.00}
\end{axis}
\begin{axis}[name=pb, at={(pa.east)}, xshift=1.25cm, anchor=west,
  width=0.215\linewidth, height=3.7cm, axis on top, tick align=outside, tick pos=left, axis line style={figSoft, line width=0.5pt}, tick label style={font=\scriptsize, color=figInk}, label style={font=\scriptsize, color=figInk}, title style={font=\small, color=figInk, yshift=1pt}, grid=none, scale only axis,
  xmin=0, xmax=1, ymin=0, ymax=1.05, xlabel={$\varphi$},
  title={(b) dashed: the same at 3}]
\addplot[figSoft, dotted, line width=0.6pt, forget plot] coordinates {(0,0) (1,1)};
\addplot[figA, dashed, line width=0.8pt] coordinates {(0.0000,0.0000) (0.1250,0.0000) (0.2500,0.0263) (0.3750,0.0227) (0.5000,0.0312) (0.6250,0.0000) (0.7500,0.0000) (0.8750,0.7586) (1.0000,1.0000)};
\addplot[figA, opacity=0.32, line width=0.8pt] coordinates {(0.0000,0.0000) (0.1250,0.0299) (0.2500,0.0510) (0.3750,0.0287) (0.5000,0.0273) (0.6250,0.0423) (0.7500,0.0952) (0.8750,0.9287) (1.0000,1.0000)};
\addplot[figB, dashed, line width=0.8pt] coordinates {(0.0000,0.0000) (0.1250,0.0000) (0.2500,0.0139) (0.3750,0.1064) (0.5000,0.2400) (0.6250,0.2222) (0.7500,0.3333) (0.8750,0.8400) (1.0000,1.0000)};
\addplot[figB, opacity=0.32, line width=0.8pt] coordinates {(0.0000,0.0000) (0.1250,0.0539) (0.2500,0.0559) (0.3750,0.1123) (0.5000,0.2127) (0.6250,0.3103) (0.7500,0.5248) (0.8750,0.9145) (1.0000,1.0000)};
\addplot[figC, dashed, line width=0.8pt] coordinates {(0.0000,0.0000) (0.1250,0.0000) (0.2500,0.1077) (0.3750,0.1379) (0.5000,0.1818) (0.6250,0.3000) (0.7500,0.4706) (0.8750,0.8367) (1.0000,1.0000)};
\addplot[figC, opacity=0.32, line width=0.8pt] coordinates {(0.0000,0.0000) (0.1250,0.0111) (0.2500,0.1290) (0.3750,0.2299) (0.5000,0.2702) (0.6250,0.3994) (0.7500,0.5949) (0.8750,0.8626) (1.0000,1.0000)};
\addplot[figD, dashed, line width=0.8pt] coordinates {(0.0000,0.0000) (0.1250,0.0000) (0.2500,0.0000) (0.3750,0.0968) (0.5000,0.0312) (0.6250,0.3469) (0.7500,0.5667) (0.8750,0.8235) (1.0000,1.0000)};
\addplot[figD, opacity=0.32, line width=0.8pt] coordinates {(0.0000,0.0000) (0.1250,0.0306) (0.2500,0.0449) (0.3750,0.1033) (0.5000,0.2116) (0.6250,0.3292) (0.7500,0.7140) (0.8750,0.8584) (1.0000,1.0000)};
\addplot[figE, dashed, line width=0.8pt] coordinates {(0.0000,0.0000) (0.1250,0.0000) (0.2500,0.0000) (0.3750,0.0000) (0.5000,0.0000) (0.6250,1.0000) (0.7500,0.8400) (0.8750,0.9426) (1.0000,1.0000)};
\addplot[figE, opacity=0.32, line width=0.8pt] coordinates {(0.0000,0.0000) (0.1250,0.0000) (0.2500,0.3125) (0.3750,0.4211) (0.5000,0.5357) (0.6250,0.6784) (0.7500,0.8814) (0.8750,0.8976) (1.0000,1.0000)};
\end{axis}
\begin{axis}[name=pc, at={(pb.east)}, xshift=1.55cm, anchor=west,
  width=0.215\linewidth, height=3.7cm, axis on top, tick align=outside, tick pos=left, axis line style={figSoft, line width=0.5pt}, tick label style={font=\scriptsize, color=figInk}, label style={font=\scriptsize, color=figInk}, title style={font=\small, color=figInk, yshift=1pt}, grid=none, scale only axis, ymode=log,
  xmin=-1, xmax=45, ymin=0.5, ymax=234780,
  xtick=\empty, xlabel={45 cells: 5 models $\times$ 9 of $\varphi$},
  ylabel={reads per cell},
  title={(c) under 30 reads: 17/45 $\to$ 3/45}]
\addplot[ybar, bar width=1.3pt, draw=none, fill=figPred, bar shift=-0.8pt]
  coordinates {(0.0,92.0) (1.0,102.0) (2.0,76.0) (3.0,44.0) (4.0,32.0) (5.0,31.0) (6.0,25.0) (7.0,29.0) (8.0,19.0) (9.0,52.0) (10.0,75.0) (11.0,72.0) (12.0,47.0) (13.0,25.0) (14.0,18.0) (15.0,6.0) (16.0,25.0) (17.0,130.0) (18.0,42.0) (19.0,56.0) (20.0,65.0) (21.0,58.0) (22.0,33.0) (23.0,20.0) (24.0,17.0) (25.0,49.0) (26.0,110.0) (27.0,23.0) (28.0,43.0) (29.0,37.0) (30.0,31.0) (31.0,32.0) (32.0,49.0) (33.0,30.0) (34.0,68.0) (35.0,137.0) (36.0,0.5) (37.0,0.5) (38.0,0.5) (39.0,0.5) (40.0,0.5) (41.0,8.0) (42.0,25.0) (43.0,122.0) (44.0,295.0)};
\addplot[ybar, bar width=1.3pt, draw=none, fill=figC, bar shift=0.8pt]
  coordinates {(0.0,351.0) (1.0,469.0) (2.0,392.0) (3.0,244.0) (4.0,183.0) (5.0,142.0) (6.0,126.0) (7.0,827.0) (8.0,3866.0) (9.0,224.0) (10.0,371.0) (11.0,376.0) (12.0,276.0) (13.0,221.0) (14.0,203.0) (15.0,202.0) (16.0,983.0) (17.0,3744.0) (18.0,109.0) (19.0,270.0) (20.0,411.0) (21.0,422.0) (22.0,322.0) (23.0,363.0) (24.0,390.0) (25.0,1099.0) (26.0,3214.0) (27.0,138.0) (28.0,196.0) (29.0,178.0) (30.0,184.0) (31.0,189.0) (32.0,243.0) (33.0,458.0) (34.0,1427.0) (35.0,3587.0) (36.0,0.5) (37.0,2.0) (38.0,16.0) (39.0,38.0) (40.0,84.0) (41.0,199.0) (42.0,590.0) (43.0,1758.0) (44.0,3913.0)};
\addplot[figInk, dotted, line width=0.6pt, forget plot]
  coordinates {(-1,30) (45,30)};
\node[font=\tiny, align=left, fill=white, fill opacity=0.85, text opacity=1, inner sep=1pt, text=figInk, anchor=north east] at (axis cs:45,195650)
  {red: 3 topics\\green: 44\\dotted: 30 reads};
\end{axis}
\end{tikzpicture}
\caption{\textbf{The primitive, and the three claims it carries.} Recomputed from
\texttt{results/} at draw time. (a) $\gamma(\varphi)$ on the forty-four topics it must predict,
with the diagonal it has to cross for the store to drift upward; four of five sit below it over
most of the interior and are captured anyway. (b) The same curves measured on three topics
(dashed): not noisy versions of (a) but \emph{lower} in the interior, on $29$ of $35$ cells, and
a lower $\gamma$ under-predicts capture. (c) The denominator behind each cell: at three topics
$17$ of $45$ cells rest on fewer than $30$ reads and $5$ on none. The primitive did not change
between (a) and (b); the number of observations behind it did.}
\label{fig:gamma}
\end{figure}

The check above passed because every simulated rate landed ``within one seed'' of the observed
count. That bar is denominated in the resampling unit, so it loosens as data gets scarcer: one seed
is $0.333$ at three seeds and $0.023$ at forty-four. E-bimodal2 supplies forty-four real seeds per
model on exactly the arm the simulation was built to predict, making the identical sentence a
fourteen-fold harder test with nothing rewritten (Figure~\ref{fig:gamma}). It was not preregistered,
the opportunity not existing when the prereg was frozen, so nothing here is graded. One confound the
data resolves: the published $\gamma$ was measured in situ on seeds $0$--$2$, so a miss on seeds
$9$--$52$ could mean either that $\gamma(\varphi)$ is incomplete or that it is complete but
topic-dependent; E-bimodal2's runs log their own $\gamma$, so we ran both arms.

\begin{table}[!htbp]
\centering\small
\caption{Simulated capture rate against $44$ real seeds per model, $4000$ simulated seeds per
arm. Arm A is the published curve on seeds $0$--$2$; arm B re-measures $\gamma$ on seeds
$9$--$52$ and predicts the same runs. \texttt{qwen3-8b} captured on every seed, so it carries no
$z$ and is excluded from the means rather than floored. The whole tabular is generated from
\texttt{results/GAMMA\_OUTSAMPLE.json}, an artifact that did not exist until this table did: the
$4000$-seed simulation behind it is seeded and reproducible, and nothing reproduced it.}
\label{tab:gammaout}
% GENERATED by t1_grounding/figures/make_t1_tikz.py from results/GAMMA_OUTSAMPLE.json.
\begin{tabular}{@{}lccccc@{}}
\toprule
model & observed & arm A & $z_A$ & arm B & $z_B$ \\
\midrule
\texttt{gemma-3-27b} & $0.682$ & $0.085$ & $-8.50$ & $0.614$ & $-0.97$ \\
\texttt{phi-4} & $0.727$ & $0.398$ & $-4.90$ & $0.807$ & $+1.18$ \\
\texttt{granite-4.1-8b} & $0.705$ & $0.522$ & $-2.65$ & $0.795$ & $+1.32$ \\
\texttt{ministral-8b} & $0.818$ & $0.553$ & $-4.55$ & $0.835$ & $+0.29$ \\
\texttt{qwen3-8b} & $1.000$ & $1.000$ & --- & $1.000$ & --- \\
\midrule
mean $|{\rm sim}-{\rm obs}|$ & & $12.1$ seed units & & $2.2$ seed units & \\
under-predicts & & $5/5$ & & $1/5$ & \\
\bottomrule
\end{tabular}

\end{table}

\noindent Arm A fails, arm B lands, and the difference is what $\gamma$ was measured on. The
published curve misses by twelve seed units and misses one-sidedly, all five models capturing more
often than it predicts, the worst by $8.5\sigma$; re-measured on the topics it must predict, the same
construction with the same simulator lands within $1.32\sigma$ everywhere and its errors change sign.
$\gamma(\varphi)$ is therefore sufficient and the claim survives, but $\gamma$ measured on
three topics is not $\gamma$, and claim 2 lacked that qualifier. Arm B fits $\gamma$ on the seeds it
then predicts, so a split-half re-partition rules out a nine-cell curve reproducing its own fit
(Appendix~\ref{app:splithalf}): fitting on unseen topics costs $15$--$16\%$ of accuracy rather than a
factor of anything, and arm A's one-sidedness is gone, at out-of-sample signed mean $z$ of $+0.32$
and $+0.35$ against arm A's $-5.15$. \textbf{$\gamma$ transports}, and claim 2b rests on a test that
could have killed it.

\subsubsection{What the limit says about the two outcomes}
\label{sec:limitbimodal}

The mean-field equation \eqref{eq:dyn} cannot produce two outcomes: it is deterministic and, when
$\geff(f) < f$, monotone, so $f$ decreases to zero from any start. Four of five models satisfy
$\geff(f) < f$ at $f_0$ and capture anyway, so the bimodality lives in what \eqref{eq:dyn}
averages over, and the limit that motivates the apparatus says how much room that leaves.

\paragraph{The derivation, and the bound it fails.} Writing the per-step change exactly and
treating successive writes as conditionally independent gives a \emph{convergent} variance sum,
$\sum_{n_0}^{N}\geff(1-\geff)/n^2 = O(1/n_0)$, against a drift that grows like
$(\geff{-}f)\ln(N/n_0)$: the noise a run can accumulate over unbounded tenure is finite and set
by how small the store was at the start, so the spread of the limiting $f$ should fall as
$1/\sqrt{n_0}$ and the two outcomes should merge as the initial store grows. \textbf{Four sweeps
at $n_0 \in \{10,20,40,80\}$ refute it}: an eightfold increase should shrink the within-model
spread by $\sqrt{8}=2.83$ and it shrinks by $1.07$. What breaks is conditional independence.
Successive reads of the same store by the same model are not independent draws from $\geff$, and
whatever correlates them does not weaken as the store grows; the between-model part stays flat,
as a property of $\gamma$ rather than of $n_0$ should, which makes the failure interpretable
rather than a decomposition artefact (Appendix~\ref{app:variance}).

\noindent The \emph{drift} half stands and is what \S\ref{sec:limit} builds on. The \emph{variance}
half does not, so we claim no bound on the outcome spread, and bimodality is measured at $f_0 \in
\{0.1,0.9\}$ and unexplained at the level of a mechanism. What is now known to be missing is a model
of the correlation between successive reads, a sharper open problem than we had before the check.

\paragraph{Ordinal, and silent about location.} The per-model early copy rate orders the
per-model capture rate at Spearman $+0.92$ (window $3$) to $+0.98$ (window $8$--$10$), stable
across window widths. The mean-field sign criterion, capture iff $\geff(f_0) > f_0$, does not
locate the boundary: four of five models drift downward at $f_0$ and capture one to two runs in
three anyway. Existence and ordering are measured; location is not derived.

\begin{scopebox}
\textbf{One of this paper's own statistics is an apparatus check, not a finding.} The trajectory
obeys $dm/dn = G(m/n)$, so it should depend on $n/n_0$ alone, and the between-$n_0$ spread measures
$0.32$ of the within-seed spread against a null of $0.58$. We first reported that as a result about
contamination. It is not: a scripted reader that always answers the truth satisfies the collapse
\emph{exactly}, and scripted majority and last-occurrence readers land at $0.23$--$0.72$, straddling
the measured $0.32$. The collapse tests that read behaviour has no tenure dependence, worth checking
and not a discovery.
\end{scopebox}

% =====================================================================
\subsection{What a falling score measures}
\label{sec:commit}

\subsubsection{Two boundaries, not two attractors}
\label{sec:boundaries}

Before the control, one structural fact decides what the control can find. At $n_0{=}10$,
$n_{\mathrm{false}}(0){=}9$, $R{=}0$ and $150$ steps, Prop.~\ref{prop:ceiling}'s ceiling is $159/160 =
0.9938$. \textbf{That proposition is elementary and the single most useful result in this paper}, for
a reason unrelated to its difficulty: $f_{\max}(R)$ moves with the budget, and \S\ref{sec:sched}
records a frozen criterion of ours that failed for exactly that reason, derivable from
\eqref{eq:ceiling} before the sweep was launched (Table~\ref{tab:ceiling},
Figure~\ref{fig:ceiling}).

\begin{figure}[!tbp]
\centering
\input{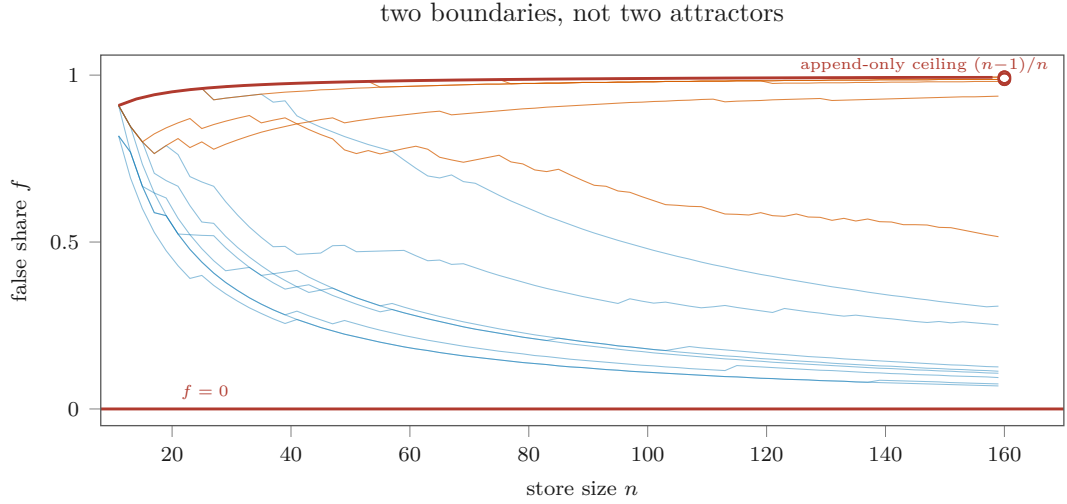}
\caption{\textbf{Two boundaries, not two attractors.} Every trajectory at $f_0{=}0.9$ against the
two edges of the state space: $f{=}0$ and the append-only ceiling $(n{-}1)/n$ of
\eqref{eq:ceiling}. Captured runs (orange) end on the ceiling rather than at an interior value;
escaping runs (blue) run to the other edge. Circled points finish within two records of the
ceiling: $\gnv{ceilOn}$ of the $\gnv{ceilCap}$ captured runs do.}
\label{fig:ceiling}
\end{figure}

\begin{table}[!htbp]
\centering\small
\caption{The captured runs against the ceiling \eqref{eq:ceiling} permits: two sit on it exactly
and four within two records. ``Truth records'' is $n(1{-}f)$ at the final step, out of $160$.}
\label{tab:ceiling}
% GENERATED by t1_grounding/figures/make_t1_tikz.py from results/converge.json.
\begin{tabular}{@{}lcccrc@{}}
\toprule
model & seed & $f_{\mathrm{end}}$ & $(n{-}1)/n$ & truth records & gap to ceiling \\
\midrule
\texttt{phi-4} & 2 & 0.9938 & 0.9938 & 1 & \textbf{0.0000} \\
\texttt{qwen3-8b} & 0 & 0.9938 & 0.9938 & 1 & \textbf{0.0000} \\
\texttt{granite-4.1-8b} & 0 & 0.9875 & 0.9938 & 2 & 0.0062 \\
\texttt{ministral-8b} & 1 & 0.9875 & 0.9938 & 2 & 0.0062 \\
\texttt{qwen3-8b} & 1 & 0.9812 & 0.9938 & 3 & 0.0125 \\
\texttt{qwen3-8b} & 2 & 0.9375 & 0.9938 & 10 & 0.0563 \\
\texttt{ministral-8b} & 0 & 0.5125 & 0.9938 & 78 & 0.4813 \\
\bottomrule
\end{tabular}

\end{table}

\begin{claimbox}
\textbf{Neither mode is an attractor of $\geff$.} One is escape toward $f{=}0$; the other is
saturation against the structural ceiling of \eqref{eq:ceiling}, where the run has asserted the false
value at essentially every step. Both are \emph{boundaries of the state space} rather than interior
fixed points, which is why the search for an interior $f^\star$ failed three times and why the
surviving account is ordinal. The one run at $0.5125$ is still descending at step $150$ and is the
only interior value in the set.
\end{claimbox}

\noindent \eqref{eq:ceiling} also settles what the next section can show. In an append-only
store, information cannot be lost, only diluted: the truth is retrievable on any draw with
probability $1-f^k > 0$, at least $0.57$ at $f{=}0.9$ and $k{=}8$, so a falling exact-match score
here is \emph{necessarily} not a loss of what the store holds (Figure~\ref{fig:commit}). The control
measures how large the resulting mismeasurement is: $0.383$.

\begin{scopebox}
\textbf{This is the scope condition on the whole decomposition.} Under eviction or pruning
\eqref{eq:ceiling} fails, truth records can leave, and $C$ is no longer conserved: the score would
report a genuine loss mixed with a commitment shift, and separating them needs a control we do not
have. \emph{Score} $= C \cdot s$ is therefore claimed for append-only stores and untested for stores
that forget, which are the stores the tenure argument recommends (\S\ref{sec:limit}, bound (ii)). The
two halves of this paper hold in adjacent regimes rather than the same one.
\end{scopebox}

Every quantity above is a single-draw score: retrieve $k$, ask once, was the answer the truth.
This apparatus conserves $C$ by construction, which is why the control is interpretable: by
Prop.~\ref{prop:ceiling} no truth record can leave an append-only store, so what the store can answer
is fixed and only the reader's commitment moves. The control measures a mismeasurement of known sign
and needs no external calibration (Table~\ref{tab:mixed}).

\begin{table}[!htbp]
\centering\small
\caption{Stores built at controlled contamination, $m{=}8$ independent retrievals each.
``Mixed'' means the eight draws disagreed.}
\label{tab:mixed}
% GENERATED by t1_grounding/figures/make_t1_tikz.py from results/recovery.json.
\begin{tabular}{@{}lrrr@{}}
\toprule
model & pass@1 at $f{=}0.9$ & pass@8 at $f{=}0.9$ & share of items \emph{mixed} at $f{=}0.9$ \\
\midrule
\texttt{google/gemma-3-27b-it} & $0.188$ & $0.625$ & $0.625$ \\
\texttt{ibm-granite/granite-4.1-8b} & $0.188$ & $0.750$ & $0.750$ \\
\texttt{microsoft/phi-4} & $0.234$ & $0.750$ & $0.750$ \\
\texttt{mistralai/ministral-8b-2512} & $0.125$ & $0.625$ & $0.625$ \\
\midrule
\multicolumn{3}{@{}l}{mean pass@8 $-$ pass@1 over $f \ge 0.5$}
  & $\mathbf{+0.383}$ \\
\bottomrule
\end{tabular}

\end{table}

\begin{figure}[!tbp]
\centering
% GENERATED by t1_grounding/figures/make_t1_tikz.py -- do not edit by hand.
% Every coordinate is recomputed from t1_grounding/results/ at generation time.

\begin{tikzpicture}
\begin{axis}[width=0.80\linewidth, height=4.8cm, axis on top, tick align=outside, tick pos=left, axis line style={figSoft, line width=0.5pt}, tick label style={font=\scriptsize, color=figInk}, label style={font=\scriptsize, color=figInk}, title style={font=\small, color=figInk, yshift=1pt}, grid=none, scale only axis,
  xmin=-0.02, xmax=1.02, ymin=0, ymax=1.05,
  xlabel={store contamination $f$}, ylabel={accuracy},
  title={contamination costs determinism, not information}]
\addplot[figB, mark=*, mark size=1.1pt, line width=0.8pt] coordinates {(0.0000,1.0000) (0.2000,0.7500) (0.4000,0.5156) (0.6000,0.3281) (0.8000,0.3125) (0.9000,0.1875)};
\addplot[figB, dashed, mark=square*, mark size=1.1pt, line width=0.8pt, opacity=0.85] coordinates {(0.0000,1.0000) (0.2000,1.0000) (0.4000,1.0000) (0.6000,0.8750) (0.8000,0.7500) (0.9000,0.7500)};
\addplot[figC, mark=*, mark size=1.1pt, line width=0.8pt] coordinates {(0.0000,1.0000) (0.2000,0.7656) (0.4000,0.5000) (0.6000,0.3438) (0.8000,0.3281) (0.9000,0.2344)};
\addplot[figC, dashed, mark=square*, mark size=1.1pt, line width=0.8pt, opacity=0.85] coordinates {(0.0000,1.0000) (0.2000,1.0000) (0.4000,0.8750) (0.6000,0.5000) (0.8000,0.5000) (0.9000,0.7500)};
\addplot[figD, mark=*, mark size=1.1pt, line width=0.8pt] coordinates {(0.0000,1.0000) (0.2000,0.7812) (0.4000,0.4844) (0.6000,0.4375) (0.8000,0.3125) (0.9000,0.1250)};
\addplot[figD, dashed, mark=square*, mark size=1.1pt, line width=0.8pt, opacity=0.85] coordinates {(0.0000,1.0000) (0.2000,1.0000) (0.4000,1.0000) (0.6000,0.8750) (0.8000,1.0000) (0.9000,0.6250)};
\addplot[figA, mark=*, mark size=1.1pt, line width=0.8pt] coordinates {(0.0000,1.0000) (0.2000,0.7656) (0.4000,0.3750) (0.6000,0.4844) (0.8000,0.3750) (0.9000,0.1875)};
\addplot[figA, dashed, mark=square*, mark size=1.1pt, line width=0.8pt, opacity=0.85] coordinates {(0.0000,1.0000) (0.2000,1.0000) (0.4000,0.3750) (0.6000,0.5000) (0.8000,0.5000) (0.9000,0.6250)};
% The two label notes are gone. They read "dashed, pass@8: the truth is still reachable" and
% "solid, pass@1: what an exact-match score reads", which is the caption's first two sentences
% -- "Solid: single-draw accuracy, which is what an exact-match score reads. Dashed: the truth
% appearing in at least one of eight independent retrievals." -- three centimetres below. Their
% grounds hid 134pt of the curves they were naming; figure_audit reports it as COVERS.
\node[font=\tiny, align=left, fill=white, fill opacity=0.85, text opacity=1, inner sep=1pt, text=figPred, anchor=north east] at (axis cs:0.98,0.62)
  {mean gap at $f\geq0.5$: $+0.383$};
\end{axis}
\end{tikzpicture}
\caption{\textbf{What a falling score measures.} Solid: single-draw accuracy, which is what
an exact-match score reads. Dashed: the truth appearing in at least one of eight
independent retrievals. The gap is what the store still holds and the score does not
report.}
\label{fig:commit}
\end{figure}
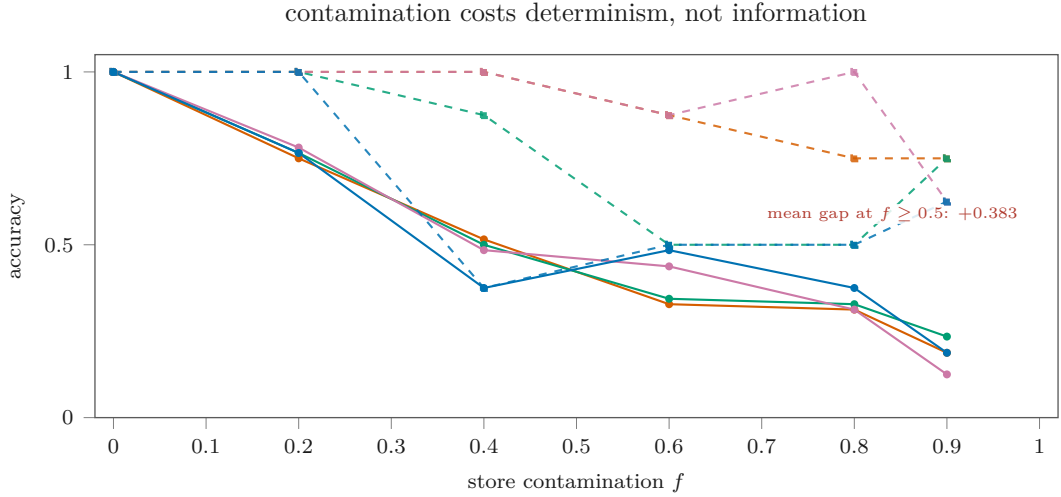

\begin{claimbox}
\textbf{Contamination is loss of determinism, not loss of information.} The store still
reaches the truth and the system has not adopted the false value; it has stopped having a
policy. Writing $C$ for the fraction solvable at all and $s$ for the share of draws
committing to one version, an exact-match score reads $C \cdot s$, and the existing
reading of such scores is its $s{\to}1$ case.
\end{claimbox}

\noindent Two demarcations are owed. That exact match understates capability is well known and
usually attributed to \emph{formatting} strictness, whereas our two versions are semantically
different answers; and that a single pass rate conflates capability with reliability is the
pass@$k$/pass$^k$ distinction, where unreliability is a property of the agent to be measured. A third
neighbour varies the scorer rather than the system and finds the measurement moves with it
\citep{judgechanges2026,reliabilityvalidity2026,targetnoninv2026}; here the scorer is fixed and
exact, and what moves is $s$. What is new is a setting in which $C$ is conserved by
construction while $s$ moves, which makes the gap a mismeasurement rather than a trade-off. The cost
to the rest of the paper is that claims phrased as capability loss must be re-read as determinism
loss; the dynamics of \S\ref{sec:bimodal} are unaffected, $f$ being defined on the store's contents
and $\gamma$ on the model's reading.

\begin{scopebox}
\textbf{A companion finds the same decomposition on the training-time side, and this paper does
not lean on it.} Under label conflict it reports $\mathrm{acc}_A+\mathrm{acc}_B$ conserved across
twelve arms while the exact-match score swings by $0.2275$ \citep{dpd2026}. That is a second
instance rather than a premise: the conservation used above is Prop.~\ref{prop:ceiling}'s, measured
here, and the argument stands with the companion deleted. We say so because the companion is not
posted, and by \S\ref{sec:limits}'s own standard a claim a reader cannot reach may not be
load-bearing.
\end{scopebox}

% =====================================================================
\subsection{Four interventions}
\label{sec:sched}

Each intervention has its criteria in \texttt{prereg/}, committed before its data existed, and is
reported at whatever the criteria say. Two of the four falsified a prediction of ours. The table
comes first so a reader sees what was at risk before what happened.

\begin{table}[!htbp]
\centering\small
\caption{What each intervention put at risk, and what came back. ``Attacks'' names the
established account it was built to contradict; where there is none, it tests a claim of ours.}
\label{tab:interventions}
\begin{tabular}{@{}p{0.13\linewidth}p{0.38\linewidth}p{0.26\linewidth}p{0.15\linewidth}@{}}
\toprule
\textbf{intervention} & \textbf{frozen criterion} & \textbf{attacks} & \textbf{outcome} \\
\midrule
grounding schedule &
  capture monotone in schedule, $\ge4/5$ models; front-loading beats no grounding by $\ge1/3$ &
  the fraction-of-verified-data rule \citep{gerstgrasser2024accum} &
  ordering \textbf{supported} $5/5$; magnitude \textbf{failed} \\
majority admission gate &
  helps or is neutral at low $f_0$; harmful or no better at high $f_0$; both $\ge4/5$ &
  write-time admission control as an unconditional safeguard
  \citep{consistencygate2026} &
  \textbf{supported} $5/5$ and $5/5$ \\
cost-matched state slot &
  helps at low $f_0$, $\ge4/5$; a direction at high $f_0$, $\ge4/5$ same way &
  our own two readings of what a self-fed slot does &
  low \textbf{supported}; high a \textbf{tie} \\
source marker &
  harmless at low $f_0$; lowers capture at high $f_0$; lowers $\gamma$; each $\ge4/5$ &
  our own registered D13$'$ &
  harmless \textbf{supported}; the other two \textbf{failed} \\
\bottomrule
\end{tabular}
\end{table}

\begin{table}[!htbp]
\centering\small
\caption{Grounding at a matched budget: $15$ truth writes in $150$ steps, three
schedules, everything else identical, against the zero-grounding arms already on disk.
Pooled capture rate over five models and three seeds.}
\label{tab:matchedbudget}
\begin{tabular}{@{}lcccc@{}}
\toprule
& front & uniform & back & ungrounded control \\
\midrule
pooled capture rate & \textbf{0.20} & 0.33 & 0.40 & 0.47 \\
\bottomrule
\end{tabular}
\end{table}

\begin{figure}[!htbp]
\centering
% GENERATED by t1_grounding/figures/make_t1_tikz.py -- do not edit by hand.
% Every coordinate is recomputed from t1_grounding/results/ at generation time.

\begin{tikzpicture}
\begin{axis}[width=0.80\linewidth, height=4.8cm, axis on top, tick align=outside, tick pos=left, axis line style={figSoft, line width=0.5pt}, tick label style={font=\scriptsize, color=figInk}, label style={font=\scriptsize, color=figInk}, title style={font=\small, color=figInk, yshift=1pt}, grid=none, scale only axis,
  xmin=-0.6, xmax=2.6, ymin=0, ymax=0.630,
  xtick={0,1,2}, xticklabels={front, uniform, back},
  xlabel={when the 15 grounding writes land (budget matched)},
  ylabel={capture rate}, title={timing, not fraction}]
\addplot[ybar, bar width=16pt, draw=none, fill=figA] coordinates {(0,0.2000)};
\addplot[ybar, bar width=16pt, draw=none, fill=figD] coordinates {(1,0.3333)};
\addplot[ybar, bar width=16pt, draw=none, fill=figB] coordinates {(2,0.4000)};
\addplot[figSoft, dashed, line width=0.8pt, forget plot]
  coordinates {(-0.6,0.4667) (2.6,0.4667)};
\node[font=\tiny, figSoft, anchor=south east] at (axis cs:2.55,0.479)
  {no grounding at all (0.47)};
\node[font=\scriptsize, figInk, anchor=south] at (axis cs:0,0.210) {0.20};
\node[font=\scriptsize, figInk, anchor=south] at (axis cs:1,0.343) {0.33};
\node[font=\scriptsize, figInk, anchor=south] at (axis cs:2,0.410) {0.40};
\draw[figPred, ->, line width=0.7pt] (axis cs:0,0.170)
  -- (axis cs:2,0.370);
\node[font=\tiny, figPred, anchor=north] at (axis cs:1,0.150)
  {monotone in 5 of 5 models};
\end{axis}
\end{tikzpicture}
\caption{\textbf{Timing, not fraction.} All three arms deliver the same fifteen grounding
writes; only their placement differs. Pooled rates should be read as $\pm$ one run
(\S\ref{sec:replicate}); the ordering holds in all five models.}
\label{fig:schedule}
\end{figure}
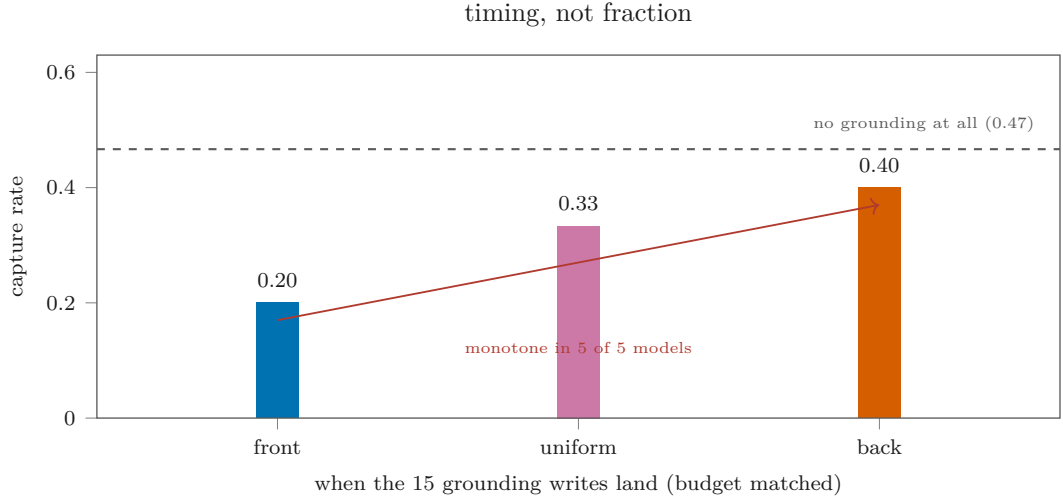

\noindent The final grounded fraction is identical in the three arms by construction, so a rule
stated in terms of the fraction of verified data predicts no difference. Capture is monotone in
schedule in \textbf{all five models} (P-S1) and the back-loaded arm sits within $0.07$ of never
grounding at all (P-S2): grounding delivered after the window has closed buys almost nothing (Figure~\ref{fig:schedule}, Table~\ref{tab:interventions}) (Table~\ref{tab:matchedbudget}). Our
magnitude bar failed, front-loading reducing pooled capture by $0.27$ against a frozen $0.33$
(P-S3), through a calibration error of ours, one model being captured in every arm and another in
none, so the pooled figure is compressed by a ceiling and a floor we did not anticipate.

\subsubsection{Closed forms and a budget transient}
\label{sec:limit}

The schedule result is what \eqref{eq:dyn} predicts once unbounded tenure with an endogenous stream
is taken seriously. A grounding write takes the store from $n$ to $n{+}1$ without changing the count
of false records, so $f \mapsto f\,n/(n{+}1)$, the step Prop.~\ref{prop:dilution} accumulates. The
resulting ordering is what \S\ref{sec:sched} measures and confirms at $220$ runs per arm. Three
further things follow from \eqref{eq:dilution}, and \emph{none is measured here}.

\textbf{(i) An ordering, with no free parameter.} At $n_0{=}10$, $R{=}15$, $T{=}150$,
\eqref{eq:dilution} gives $L = 0.947 / 0.332 / 0.099$, i.e.\ $1 : 0.35 : 0.10$; the measured
reduction in capture against the ungrounded control is $0.27 / 0.14 / 0.07$, i.e.\
$1 : 0.52 : 0.26$. Same order, compressed, which is what a bounded observable must do to an
unbounded one. The ordering is the claim; the compression is why we do not claim the magnitudes.

\textbf{(ii) An upper bound on what any read-side intervention can buy.} As $T \to \infty$ with $R$
fixed, $L_{\mathrm{back}} \to 0$ while $L_{\mathrm{front}} \to \ln(1{+}R/n_0)$, a constant: a finite
budget spent late is worth nothing in the limit and spent early retains a finite value forever. More
generally $\dot f \propto 1/n$ bounds \emph{every} intervention acting through reading behaviour, its
effect per step vanishing with tenure, so only interventions that change $n$ itself (forgetting) or
add edges from outside the loop (grounding) survive. The apparatus's limit, not its parameters,
selects which knobs matter.

\textbf{(iii) The budget dependence is a transient, derived but not measured.}
$\mathrm{spread}(R) = L_{\mathrm{front}} - L_{\mathrm{back}}$ vanishes at $R{=}0$ and again at
$R{=}T$, where every step is grounded and the three schedules are one arm, so it has an interior
maximum, which \eqref{eq:dilution} puts near $R/T \approx 0.5$.

\textbf{(iv) The transient is in the fraction; in the tenure the effect is unbounded.} At fixed
$\alpha = R/T$ with $T \to \infty$ the same expression diverges like $\ln T$, so (iii) and (iv)
are one surface read along two axes rather than a contradiction, and the two limits do not
commute: the spread is identically zero at $\alpha{=}1$ and unbounded at every $\alpha < 1$, so
free re-grounding is a single point at which the schedule axis collapses rather than a regime
this one approaches. Appendix~\ref{app:limits} gives the derivation, the numbers, and the frozen
criterion of ours that it exposes as set on the wrong side of the peak.

\paragraph{The budget sweep, and both of its frozen criteria failing upward.} Two further
budgets, $R \in \{5,45\}$, everything else identical, against the $R{=}15$ arms already on disk;
five models and three seeds each.

\begin{table}[!htbp]
\centering\small
\caption{Capture rate by budget and schedule; the ungrounded control is $0.47$. The frozen
falsifier was that the spread must shrink as $R$ grows.}
\label{tab:budget}
\begin{tabular}{@{}rcccc@{}}
\toprule
$R$ & front & uniform & back & spread \\
\midrule
5  & 0.20 & 0.27 & 0.40 & 0.20 \\
15 & 0.20 & 0.33 & 0.40 & 0.20 \\
45 & \textbf{0.07} & 0.20 & \textbf{0.47} & \textbf{0.40} \\
\bottomrule
\end{tabular}
\end{table}

\noindent \textbf{The ordering holds at every budget}, front $\le$ uniform $\le$ back at $R = 5, 15,
45$, an independent confirmation of \S\ref{sec:sched}'s result on an axis it was not measured on; at
$R{=}45$ front-loading leaves capture at $0.07$ against back-loading's $0.47$, the ungrounded control
to two decimals, so forty-five grounding writes placed late buy exactly nothing
(Table~\ref{tab:budget}). \textbf{Both frozen criteria failed}: P-B3's floor clause fired, all three
$R{=}5$ arms lying within one seed of the control, and P-B1 required $\mathrm{spread}(45) <
\mathrm{spread}(15)$ against measured values of $0.40$ and $0.20$. The failure runs in the direction
\eqref{eq:dilution} predicts, which by the pre-registration locates the error in our criterion and
does not license the transient as a result.

\noindent \textbf{The sweep past the peak did not resolve, and we bought the resolution we had
called unbuyable.} Three further budgets graded \textsc{untested}, because past $R{=}120$ the
ceiling has fallen to $f_{\max} = 0.516$, only $0.016$ above the capture threshold, so the frozen
observable cannot express the outcome its criterion asks about; a scale-free re-reading on six
fresh seeds then \emph{split} between two equally defensible observables. We ran $1980$ runs at
$44$ seeds: the discriminator is undecided at $1.14\sigma$, the shape criterion fails at $-0.074$
against $-0.15$, and the effect fell from $0.115$ to $0.021$ while the standard error shrank
$21\%$ more slowly than $1/\sqrt{n}$, so $3\sigma$ now needs about $1895$ seeds. The peak's
location is claimed in neither direction anywhere in this paper; the \emph{ordering} survives.
Appendix~\ref{app:moved} gives the four attempts in full.

\noindent An accidental replicate prices what no other figure here prices, a quantity most
published results leave unpriced and that a replication attempt has to reconstruct after the fact
\citep{besiroglu2024replication}. A command intended to count configurations executed them, running
this sweep twice. Appendix~\ref{app:moved} reads the accident and gives the arm-by-arm figures; the
verdicts survive it, and its consequence is that every pooled rate here should be read as $\pm$ one
run, every other number carrying the same replication variance unpriced.

\subsection{The retrieval rule's own falsifier}
\label{sec:retriever}

\S\ref{sec:apparatus} calls uniform-without-replacement ``a choice with a falsifier attached'',
and a falsifier that is never fired is not a falsifier. Criteria in
\texttt{prereg/E\_RETRIEVER\_PREREG.md}, frozen and committed before any retriever code existed,
with an amendment of the same day recorded below.

Seven arms on one shared batch: the published control, three ranking retrievers (\textsc{MiniLM}
cosine, with the topic oracle kept, dropped, and dropped-with-deduplication), and a uniform ladder at
$k \in \{4,2,1\}$ that P-T1-R3 reads its curve off (Table~\ref{tab:retriever}). Every arm carries a
static distractor mass of $90$ off-topic records, since \texttt{loop.run} builds one topic and the
oracle filter is otherwise vacuous; that mass also gives the arms room to differ in the way
irrelevant context is known to matter \citep{shi2023distracted}, and holding $k$ fixed keeps position
effects \citep{liu2024lostmiddle} constant across arms. \emph{Only selection changes}: every arm
renders the agent the same $k$ lines.

\begin{table}[!htbp]
\centering\small
\caption{E-retriever. $k_{\rm eff}$ is the mean number of distinct on-topic values retrieved,
defined in the preregistration before it was measured, and $\mathrm{var}(n_f)$ the across-step
variance of the retrieved false count. Two models judged; see the floor note.
\textbf{The whole tabular is generated from} \texttt{results/retriever/GRADE.json}, \emph{and one
cell of it was wrong}: the oracle-kept row printed $\mathrm{var}(n_f) = 0.347$ where the grader
computes $0.477$. It grades nothing, the prereg's amendment added it as the mechanism variable,
so no verdict moves, and it was wrong for as long as it was typed.}
\label{tab:retriever}
% GENERATED by t1_grounding/figures/make_t1_tikz.py from retriever/GRADE.json.
\begin{tabular}{@{}lrrrrrr@{}}
\toprule
arm & $k$ & $k_{\rm eff}$ & capture $f_0{=}0.1$ & capture $f_0{=}0.9$ & on-topic &
$\mathrm{var}(n_f)$ \\
\midrule
uniform (control)      & $8$ & $1.95$ & $0.533$ & $0.933$ & $1.000$ & $5.753$ \\
embed, oracle kept     & $8$ & $1.40$ & $0.000$ & $1.000$ & $1.000$ & $0.477$ \\
embed, no oracle       & $8$ & $1.45$ & $0.033$ & $1.000$ & $0.938$ & $1.062$ \\
embed, dedup           & $8$ & $3.22$ & $0.900$ & $0.833$ & $0.402$ & $0.520$ \\
uniform                & $4$ & $1.45$ & $0.233$ & $0.933$ & $1.000$ & $1.070$ \\
uniform                & $2$ & $1.27$ & $0.400$ & $0.933$ & $1.000$ & $0.509$ \\
uniform                & $1$ & $1.00$ & $0.000$ & $1.000$ & $1.000$ & $0.060$ \\
\bottomrule
\end{tabular}

\end{table}

\noindent \textbf{P-T1-R0, the instrument gate, is \textsc{supported}}: the oracle-free retriever
puts $93.8\%$ of what it returns on topic, so it is a retriever and not a broken index.
\textbf{P-T1-R1 is \textsc{supported}}, the displacement being $-0.500$ against a zero band of
$0.10$, so the uniform assumption is not innocuous and \S\ref{sec:limits} carries it by
name. \textbf{P-T1-R5 is \textsc{supported}}: keeping the topic oracle reproduces $1.00$ of the
displacement, so the cause is ranking \emph{determinism} and better topic routing cannot repair
it.

\begin{claimbox}
\textbf{P-T1-R2 is \textsc{refuted}, and the refutation is in our favour, which is why it was written
down first.} We predicted ranking would make collapse \emph{more} likely and it makes it \emph{less}:
at $f_0=0.1$ capture falls from $0.533$ to $0.033$. It does not simply help, it \textbf{polarises},
rising from $0.900$ to $1.000$ at $f_0=0.9$. Uniform hedges in both directions; a deterministic
ranking drives the store to whichever end it started nearer. \emph{The published apparatus is
therefore the pessimistic choice at low contamination and the optimistic one at high}, and this
paper's numbers bound the harm only where the store starts mostly clean.
\end{claimbox}

\noindent \textbf{The mechanism is not the one the preregistration was built on.}
\textbf{P-T1-R3 is \textsc{not supported}}: both ranking arms sit off the uniform curve by
$0.308$, so the intended reading, that ranking narrows effective breadth and $k_{\rm eff}$ is the
right variable, does not survive. What survives is the alternative the preregistration's own
hand-computation named as the thing it was least sure of: \emph{ranking freezes the retrieved set
across steps where uniform resamples it, which is variance and not breadth}. A $\$0$ diagnostic
before the graded arm found the identical eight records returned on $59$ of $59$ consecutive
steps under ranking against $0$ of $59$ under uniform, and $\mathrm{var}(n_f)$ is $5.753$ under
the control against $0.477$ under ranking at nearly the same $k_{\rm eff}$.

\begin{claimbox}
\textbf{The reduction is exact, and $k_{\rm eff}$ is the wrong summary of it.} Ranking at nominal
$k=8$ reproduces uniform at $k=1$ \emph{exactly}, capture $0.000$ and $1.000$ against $0.000$ and
$1.000$ at both contamination levels, so there is a reduction and nominal $k$ is not the operative
variable. The statistic we preregistered to express it, $k_{\rm eff}$, nonetheless reads $1.40$ for
ranking against $1.00$ for uniform-at-one and \emph{misses the match it was defined to catch}. What lines the
two up is the variance of the evidence stream, $0.240$ against $0.060$, with the control two orders
of magnitude above both. We registered the wrong summary of the right idea, and P-T1-R3's
failure is that fact rather than an absence of structure.
\end{claimbox}

\paragraph{Is the determinism a property of ranking, or of \textsc{MiniLM}?} The arm above uses one
embedder, so a reader may suspect that what we call ranking determinism is the degeneracy of a
sentence embedder in a world whose candidate answers differ by one digit. The two quantities the
mechanism rests on are properties of the \emph{ranker} and the record text rather than of the agent
(can the score separate a truth record from its corruption, and how many distinct values survive into
the top-$k$), so both are computable offline, on the exact text form the apparatus stores, for any
ranker. Three were run: \textsc{MiniLM} cosine, Okapi BM25 over whitespace tokens, and cosine over
character trigrams, the last chosen because a changed digit alters up to three of its features and it
therefore \emph{can} see what an embedder cannot.

\begin{claimbox}
\textbf{The mechanism is not the embedder's, and the answer is structural.} No ranker separates
the truth from its one-digit corruption: the largest separation over all three is $0.0052$, and
for BM25 and character trigrams it is exactly $0$. The reason is not a weakness of any of them.
\emph{The query is the question and the corruption is in the answer}, so no query--document score
has any signal about the corrupted field. Every ranker's retrieved value set is invariant over
tenure on $12$ of $12$ store compositions. \textsc{MiniLM} is the \emph{least} degenerate of
the three: its $k_{\rm eff}$ runs $1$--$4$ while BM25's and character trigrams' are both exactly
$1$, a lexical score tying every on-topic record so that insertion order decides, making the top-$k$
``the $k$ oldest records on topic''. Swapping the embedder for a lexical index does not restore a
draw; it substitutes a first-occurrence rule for a similarity rule and collapses breadth further.
\S\ref{sec:retriever} measured the most favourable of the three.
\end{claimbox}

\noindent What this does not do is replace a loop arm. A ranker that separated the values would
still have to be run in the loop to show the capture rate moves back, and none of the three separates
them; the check is a statement about the retrieval step, not about the outcome
(\texttt{analyze\_retriever\_family.py}, Appendix~\ref{app:retfamily}). It also names the apparatus
feature that makes the question decidable at all: because the corruption is a digit inside the
answer, this world is one where retrieval \emph{cannot} help, and a world whose falsehoods differ
semantically from the truth is one where it might. That is a scope statement about the apparatus and
belongs with \S\ref{sec:limits}(vii).

\begin{warnbox}
\textbf{Three costs of this arm, two of them ours and one paid in dollars.} P-T1-R3's ladder has
almost no dynamic range and is not monotone. P-T1-R4 is \textsc{supported} but does not
mean what its name says: the deduplicated arm displaces capture by $0.367$ against ranking's
$0.500$, but with only two distinct on-topic values, one-record-per-value fills two slots and
pads the other six from off-topic, so that arm tests dedup-\emph{plus}-padding. The run
halted on its own spend guard at $\$5.06$ against a hard stop of $\$5.00$, having cost $\$5.13$
against a preregistered estimate of $\$2.06$; thirty-three dropped runs put \texttt{qwen3-8b}
below the floor clause, so two of three models are judged, on $10$ seeds each.
Appendix~\ref{app:retriever} gives all three in full, with the post-hoc reading of the empty
interval under ranking.
\end{warnbox}

% =====================================================================
\subsection{Sweeping the resource the retriever arm identified}
\label{sec:temp}

\S\ref{sec:retriever} leaves the paper claiming that the operative resource under a ranking retriever
is the variance of the evidence stream rather than its breadth. A second retriever family would not
test that: the query contains no value, so no query-similarity index can rank candidates by truth. We
swept the resource instead. Arm $\mathrm{mix}m$ fills $m$ of the $k$ slots with the deterministic
top-$m$ and samples the rest uniformly, so $m{=}0$ is the published control and $m{=}8$ is
\S\ref{sec:retriever}'s ranking. Criteria are frozen in \texttt{prereg/E\_TEMP\_PREREG.md} and the
grader was committed before any \texttt{mix} number was read.

\begin{table}[!htbp]
\centering\small
\caption{E-temp, five arms, two models, ten seeds, $f_0 = 0.1$. $\mathrm{var}(n_f)$ is the
across-step variance of the retrieved false count. The whole tabular is generated from
\texttt{results/temp/GRADE.json}, which already held both rows.}
\label{tab:temp}
% GENERATED by t1_grounding/figures/make_t1_tikz.py from results/temp/GRADE.json.
% The WHOLE tabular; a fragment ending in \\ breaks \bottomrule across the file.
\begin{tabular}{@{}lccccc@{}}
\toprule
 & $m{=}0$ & $m{=}2$ & $m{=}4$ & $m{=}6$ & $m{=}8$ \\
\midrule
capture             & $0.450$ & $0.350$ & $0.150$ & $0.200$ & $0.000$ \\
$\mathrm{var}(n_f)$ & $5.911$ & $4.069$ & $1.344$ & $0.496$ & $0.347$ \\
\bottomrule
\end{tabular}

\end{table}

\noindent P-T1-X0 reproduces both endpoints of \S\ref{sec:retriever} to $+0.000$, so the batches
are the same instrument; P-T1-X1 holds, $\mathrm{var}(n_f)$ being monotone in $m$; P-T1-X3 holds,
all three interior arms lying strictly between the endpoints and outside both zero bands, so the
transition is graded rather than a step at $m{>}0$; and P-T1-X4 holds in its strong form, with
the back-fill complete, three models clearing the floor and every P-T1-R verdict unchanged (Table~\ref{tab:temp}).

\begin{warnbox}{P-T1-X2 is \textsc{not supported}, and the criterion could not have decided it}
Capture is not monotone in $m$ ($0.200$ at $m{=}6$ against $0.150$ at $m{=}4$) and $m{=}6$ sits
$0.177$ off the leave-one-out curve against $\mathrm{var}(n_f)$. By the frozen criteria the
variance account is refuted, so this paper has now proposed and refuted two mechanisms for one
displacement.

The defect is ours. The violation is $3/20$ against $4/20$, a single run, and the conjunct was
registered as a strict inequality with no tolerance. At $p \approx 0.175$ the paired standard error
over the shared seeds is $0.095$, and \S\ref{sec:power}'s design effect of $3.75$ makes the effective
sample nearer five runs than twenty: we froze a zero-width band on a statistic this paper had already
measured to be too noisy to support one. The verdict stands as registered, and it licenses that the
smooth account is unsupported, not that a threshold account is supported.
\end{warnbox}

\subsection{The synthetic world on a real axis}
\label{sec:real}

\S\ref{sec:limits}(vii) is the one limitation no internal rigour repairs: every number above was
measured on entities that do not exist, built by \texttt{world.py} so that no value can have been
memorised. Read as a limit, that is a claim about a coordinate the paper never measured, $p_0$, the
probability the model answers correctly from an empty store; the synthetic world is its $p_0 \to 0$
corner.

We measured the axis. Thirty-six $(\text{entity},\text{capital})$ pairs were drawn from Wikidata over
one property, so the question form is fixed and only the entity varies, stratified by the number of
Wikipedia language editions carrying the entity, a proxy of the same kind as the popularity measures
under which parametric knowledge degrades \citep{mallen2023whennot,kandpal2023longtail}; sitelinks
select the sample and are not the axis. $p_0$ is measured by an empty-store probe with ten
repetitions, split by parity so terciles are cut on one half and reported on the other. Two
amendments were made before any call: entities named after their own capital were excluded, that leak
running $1/12$ in the top stratum against $9/12$ in the bottom and inflating $p_0$ exactly where the
design needs it near zero; and an answer rate $a_0$ was defined against the frozen list of real
capitals, since a model that declines identically ten times would otherwise score maximal answer
concentration. Criteria are in \texttt{prereg/E\_REAL\_PREREG.md}, with the grader committed before
the first call.

\begin{table}[!htbp]
\centering\small
\caption{E-real. Probe cells are single-turn; loop runs are the full $150$-step ouroboros.
Grading is at $f_0 = 0.1$, the registered level, where both arms are gradable on every run.}
\label{tab:real}
\begin{tabular}{@{}llcc@{}}
\toprule
criterion & quantity & value & verdict \\
\midrule
P-T1-N0(a) & matched-cell $\gamma(\varphi)$ against the published panel & $0.083$
  & \textsc{supported} \\
P-T1-N0(b) & $p_0$ on synthetic topics & $0.000$ & \textsc{supported} \\
P-T1-N1 & span of measured $p_0$ & $[0.000, 1.000]$ & \textsc{supported} \\
P-T1-N2 & $\gamma(1)$, bottom against top $p_0$ tercile & $0.711 \to 0.411$ & \textsc{supported} \\
P-T1-N3 & synthetic $\gamma(1)$ against the lowest-$p_0$ real cell & $1.000$ vs $0.700$
  & \textsc{not supported} \\
P-T1-N4 & threshold agreement, real against synthetic & $353/360$ vs $39/40$ & \textsc{supported} \\
P-T1-N6 & mass in $(0.306, 0.512)$, real against synthetic & $0.014$ vs $0.125$
  & \textsc{supported} \\
\bottomrule
\end{tabular}
\end{table}

\begin{claimbox}{The threshold transfers off the synthetic world}
$\mathrm{sign}(\hat\gamma - \gamma_{\rm crit})$ predicts $\mathrm{sign}(f_{\mathrm{end}} - f_0)$
on $353$ of $360$ real-fact runs (Wilson $95\%$ $[0.960, 0.991]$), against $39$ of $40$ synthetic
runs in the same batch ($[0.871, 0.996]$). \emph{The two are not being compared}: the denominators
differ ninefold and a difference interval on these counts is $0.10$ wide, so the claim is that the
threshold holds on real facts, the synthetic arm being a same-batch sanity check rather than a
matched control. The threshold is the fixed point of $G_{\mathrm{eff}}(f) = f$, which linearises to
$\gamma_{\rm crit} = 1/k$ \textbf{at $r{=}0$ and $w{=}1$}, the configuration of every run reported
here, and a scoping the linearisation requires: with $w>1$ writes per step or $r>0$ retirement the
fixed point moves and $1/k$ is not it. Within that scope it contains no term for what the model
knows, and needs none. The empty interval is not an artefact of zero-prior topics either: it is
emptier on real facts, $0.014$ against $0.125$.
\end{claimbox}

\noindent The copy probability does move with the prior, as P-T1-N2 requires: $\gamma(1)$ falls
from $0.711$ in the bottom $p_0$ tercile to $0.411$ in the top, and a bootstrap over the $108$
cells puts that drop at $[0.125, 0.464]$, entirely above the registered band (Table~\ref{tab:real}).

\begin{warnbox}{P-T1-N3 fails, and two registered readings fall with it}
The synthetic world is not the $p_0 \to 0$ endpoint of the real curve. Its $\gamma(1)$ is $1.000$
against $0.700$ at the lowest-$p_0$ real cell, so it sits $0.300$ \emph{above} the real curve rather
than at its end. What the measurement supports is weaker than what this arm set out to show, and more
defensible: the idealisation bounds the copy probability from above by a measured $0.300$ while
preserving the threshold. A post-hoc check on the answer rate gives a weaker effect than $p_0$
($0.171$ against $0.300$) and low-$a_0$ real cells still sit at $0.724$, so neither coordinate
accounts for the residual.

The prereg's second amendment registered a secondary criterion on the empty-store answer
concentration $c_0$, holding that a disagreement between it and $p_0$ would show confidence
rather than correctness competing with the store. That reading fails: the $c_0$ terciles
differ by $-0.013$, with the upper tercile scoring lower on the reporting half, so the split-half
design inverted, which is what it was registered to detect. The disagreement is between an axis
of contrast $1.000$ and one of contrast $0.013$, not between two axes.
\end{warnbox}

\subsection{Three withdrawn axes}
\label{sec:killed}

\noindent \textbf{Both were unresolvable in advance, and the resampling unit is why.} The seed
selects the topic and the topic dominates the variance: a share reading $2.0\%$ at three seeds reads
$68.8\%$ at nine, a design effect of $3.75$. Every frozen criterion here is a point estimate against
a threshold, so no verdict changes, but every summary interval was built on the wrong unit.
Appendix~\ref{app:moved} gives the decomposition and the cost of resolution.

\subsubsection{The three axes and their refutations}

\noindent Three axes were each a post-hoc reading of data in hand, turned into a
pre-registration with fresh seeds and then refuted by it: that damage scales with unaided
cleaning ability (the ordering survives at $0.264$ vs $0.249$, the magnitude fails at $0.015$
against a frozen bar of $0.30$, and the ``cleanest model'' had a $3$-seed baseline of $0.150$
that becomes $0.729$ on seven seeds, so the independent variable was noise); that a marker's
effect is a single function of label--truth correlation (strict five-point monotonicity holds in
$0/5$); and a sign reversal unifying the one-shot and store-level timescales, which rests on the
first and dies with it. They are listed because a framework that only ever produces confirmations is
not being tested, and because \S\ref{sec:power} shows all three were beyond the resolution of the
design that proposed them. Appendix~\ref{app:killed} gives each in full.

% =====================================================================
\subsection{The falsifier on the unreached rows}
\label{sec:knife}

\S\ref{sec:power} establishes that the resampling unit is the seed, that a three-seed sample cannot
see it, and that a pooled level therefore carries a design effect of $3.75$. That result killed the
three axes above and downgraded one graded ordering. Half of this paper's graded verdicts rest on
three-seed arms, so the knife has to be carried to every row, in three parts, one of them a claim of
ours that does not survive.

\paragraph{First: $3.75$ is a property of pooled levels, and the intervention criteria are not
levels.} Each of the four interventions in \S\ref{sec:sched} compares a treated run against a
control run \emph{on the same seed}, so its criterion is a paired contrast, and a contrast cancels
whatever the seed does to both arms. That defence cannot be checked at three seeds (the content of
\S\ref{sec:power} is that three seeds read the seed variance as near zero when it is not), so it is
checked where a $44$-seed sample exists, on the schedule axis, and only then transported. Measured on
the same runs at $44$ seeds, a pooled level carries $3.56$--$3.66$ and a paired front-minus-back
contrast $1.14$--$1.47$; the clustered standard error is $1.80$--$1.85$ times the naive one for
levels and $0.93$--$1.15$ times it for contrasts, so for a contrast the naive error is if anything
conservative. \emph{The correction those criteria need is a factor of about $1.2$, not of $3.75$.}
The three-seed subsample of the identical data reports a level design effect of $1.33$--$1.68$,
reproducing the invisibility on a second arm.

\paragraph{Second: what three seeds do cost is a different thing, and one criterion pays it.}
Re-evaluating each of the eight frozen criteria on each seed alone (Table~\ref{tab:perseed})
reproduces every published value exactly and shows seven of eight holding on all three seeds.
P-A1, the low-$f_0$ half of the state-slot arm, holds at $5/5$, $3/5$ and $5/5$: it misses its own
bar of four on one of its three seeds, so row 9 is graded \textsc{underpowered} rather than
\textsc{measured}. Two further readings belong here rather than in a correction. The gate's $5/5$
(row 8) is one collapse observed five times (its treated arm has a cross-model spread of $0.0006$
against the control's $0.9250$), that section's own finding restated as a limit on its own evidence.
With three seeds as the unit, the smallest two-sided $p$ a sign test can attain is $0.25$ whatever
the data say, which is why none of these arms is reported with a $p$.

\paragraph{Third: the control was a lucky draw, and replacing it changes nothing.} Every high-$f_0$
intervention is measured against \texttt{converge}'s three seeds, and \texttt{bimodal2} is the
identical configuration at forty-four. The control's per-model capture rate moves from
$0.00/0.33/0.33/0.67/1.00$ to $0.68/0.71/0.73/0.82/1.00$, up to $+0.68$ for one model. Re-reading all
four high-$f_0$ criteria against the better control leaves every verdict where it was: the gate stays
$5/5$, the scrambled-marker ordering stays $5/5$, P-P2 improves from $2/5$ to $3/5$ and is still
\textsc{failed}, and P-A2 remains a tie. Figure~\ref{fig:seedscale} is the three parts in one
frame.

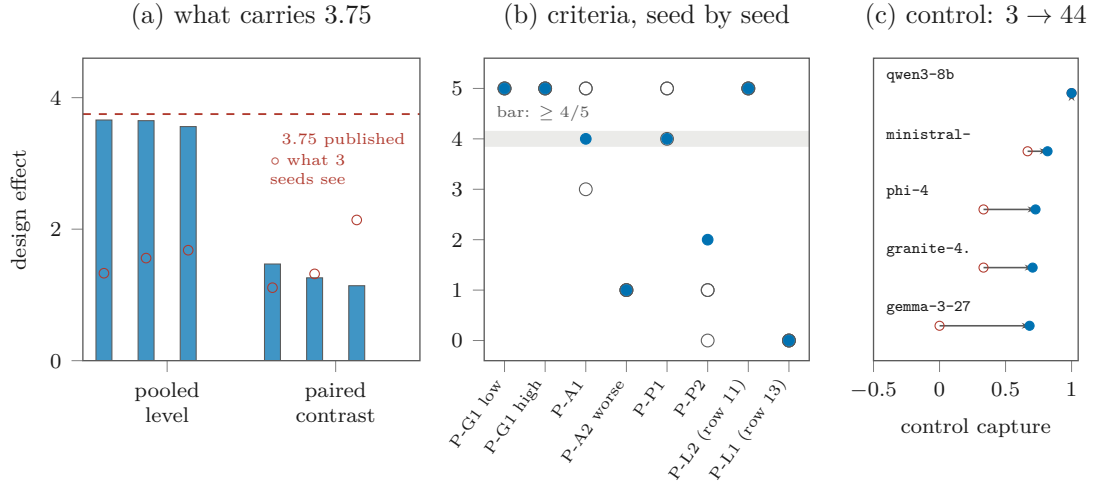
\begin{figure}[!tbp]
\centering
% GENERATED by t1_grounding/figures/make_t1_tikz.py -- do not edit by hand.
% Every coordinate is recomputed from t1_grounding/results/ at generation time.

\begin{tikzpicture}
\begin{axis}[name=pa, width=0.28\linewidth, height=4.0cm, axis on top, tick align=outside, tick pos=left, axis line style={figSoft, line width=0.5pt}, tick label style={font=\scriptsize, color=figInk}, label style={font=\scriptsize, color=figInk}, title style={font=\small, color=figInk, yshift=1pt}, grid=none, scale only axis,
  ybar, bar width=6pt, xmin=0, xmax=8, ymin=0, ymax=4.6,
  xtick={2,6}, xticklabels={{pooled\\level},{paired\\contrast}},
  xticklabel style={align=center, font=\scriptsize},
  ylabel={design effect},
  title={(a) what carries $3.75$}]
\addplot[draw=figSoft, fill=figA, fill opacity=0.75] coordinates {(0.5,3.66) (1.5,3.65) (2.5,3.56) (4.5,1.47) (5.5,1.26) (6.5,1.14)};
\addplot[only marks, mark=o, mark size=1.8pt, figPred] coordinates {(0.5,1.33) (1.5,1.56) (2.5,1.68) (4.5,1.11) (5.5,1.32) (6.5,2.14)};
\draw[figPred, dashed, line width=0.7pt] (axis cs:0,3.75) -- (axis cs:8,3.75);
\node[font=\tiny, figPred, anchor=north east] at (axis cs:7.9,3.62)
  {$3.75$ published};
\node[font=\tiny, figPred, anchor=north west, align=left] at (axis cs:4.2,3.30)
  {$\circ$ what $3$\\seeds see};
\end{axis}
\begin{axis}[name=pb, at={(pa.east)}, xshift=0.85cm, anchor=west,
  width=0.27\linewidth, height=4.0cm, axis on top, tick align=outside, tick pos=left, axis line style={figSoft, line width=0.5pt}, tick label style={font=\scriptsize, color=figInk}, label style={font=\scriptsize, color=figInk}, title style={font=\small, color=figInk, yshift=1pt}, grid=none, scale only axis,
  xmin=0, xmax=8, ymin=-0.4, ymax=5.6,
  xtick={0.5,1.5,2.5,3.5,4.5,5.5,6.5,7.5},
  xticklabels={{P-G1~low},{P-G1~high},{P-A1},{P-A2~worse},{P-P1},{P-P2},{P-L2~(row~11)},{P-L1~(row~13)}},
  xticklabel style={rotate=55, anchor=east, font=\tiny},
  ytick={0,1,2,3,4,5},
  title={(b) criteria, seed by seed}]
\draw[figGrid, line width=6pt, opacity=0.5] (axis cs:0,4) -- (axis cs:8,4);
\node[font=\tiny, figSoft, anchor=south west] at (axis cs:0.05,4.15) {bar: $\ge 4/5$};
\addplot[only marks, mark=o, mark size=2.4pt, figSoft] coordinates
  {(0.5,5) (0.5,5) (0.5,5) (1.5,5) (1.5,5) (1.5,5) (2.5,5) (2.5,3) (2.5,5) (3.5,1) (3.5,1) (3.5,1) (4.5,5) (4.5,5) (4.5,4) (5.5,1) (5.5,0) (5.5,1) (6.5,5) (6.5,5) (6.5,5) (7.5,0) (7.5,0) (7.5,0)};
\addplot[only marks, mark=*, mark size=2.0pt, figA] coordinates
  {(0.5,5) (1.5,5) (2.5,4) (3.5,1) (4.5,4) (5.5,2) (6.5,5) (7.5,0)};
\end{axis}
\begin{axis}[name=pc, at={(pb.east)}, xshift=0.85cm, anchor=west,
  width=0.17\linewidth, height=4.0cm, axis on top, tick align=outside, tick pos=left, axis line style={figSoft, line width=0.5pt}, tick label style={font=\scriptsize, color=figInk}, label style={font=\scriptsize, color=figInk}, title style={font=\small, color=figInk, yshift=1pt}, grid=none, scale only axis,
  xmin=-0.5, xmax=1.05, ymin=-0.6, ymax=4.6,
  xlabel={control capture}, ytick=\empty,
  title={(c) control: $3\to44$}]
\draw[figSoft, -{Stealth[length=3pt]}, line width=0.6pt] (axis cs:0.000,0) -- (axis cs:0.682,0);\node[font=\tiny, figInk, anchor=west] at (axis cs:-0.48,0.3) {\texttt{\tiny gemma-3-27}};
\draw[figSoft, -{Stealth[length=3pt]}, line width=0.6pt] (axis cs:0.333,1) -- (axis cs:0.705,1);\node[font=\tiny, figInk, anchor=west] at (axis cs:-0.48,1.3) {\texttt{\tiny granite-4.}};
\draw[figSoft, -{Stealth[length=3pt]}, line width=0.6pt] (axis cs:0.333,2) -- (axis cs:0.727,2);\node[font=\tiny, figInk, anchor=west] at (axis cs:-0.48,2.3) {\texttt{\tiny phi-4}};
\draw[figSoft, -{Stealth[length=3pt]}, line width=0.6pt] (axis cs:0.667,3) -- (axis cs:0.818,3);\node[font=\tiny, figInk, anchor=west] at (axis cs:-0.48,3.3) {\texttt{\tiny ministral-}};
\draw[figSoft, -{Stealth[length=3pt]}, line width=0.6pt] (axis cs:1.000,4) -- (axis cs:1.000,4);\node[font=\tiny, figInk, anchor=west] at (axis cs:-0.48,4.3) {\texttt{\tiny qwen3-8b}};
\addplot[only marks, mark=o, mark size=1.7pt, figPred] coordinates
  {(0.000,0) (0.333,1) (0.333,2) (0.667,3) (1.000,4)};
\addplot[only marks, mark=*, mark size=1.7pt, figA] coordinates
  {(0.682,0) (0.705,1) (0.727,2) (0.818,3) (1.000,4)};
\end{axis}
\end{tikzpicture}
\caption{\textbf{The resampling unit, applied where it had not been.} (a) At $44$ seeds a pooled
level carries the design effect this paper published ($3.75$, dashed) and a paired contrast does
not; hollow markers are what a three-seed subsample of the same runs reports, which for levels is
wrong by a factor of four. That is the argument for measuring the contrast at forty-four rather
than at three. (b) Each of the eight frozen criteria of the four three-seed arms: filled markers
are the published value, hollow markers the same rule evaluated on one seed at a time, band the
bar of $\ge 4/5$. Seven criteria clear or miss on all three seeds; P-A1 clears on two and misses
on one, and its row is downgraded. (c) The control those arms were measured against, at three
seeds and at forty-four, on the identical configuration. It moves by up to $0.68$ in capture rate
and no verdict follows it.}
\label{fig:seedscale}
\end{figure}

\paragraph{A claim of ours that does not survive the same treatment.} Claims 2 and 4 are rank
correlations over five models whose \emph{dependent} variable is a capture rate, which at three seeds
has a standard error of $0.29$ on a $[0,1]$ scale by \eqref{eq:se}, the construction
\S\ref{sec:power} blames for a spurious $+1.00$ elsewhere in this paper. Recomputed on the $44$-seed
capture rates, the two move in opposite directions.

\begin{warnbox}
\textbf{Claim 4 is not supported at forty-four seeds.} The published statistic is Spearman
$+0.92$ to $+0.98$ between a model's early copy rate and its capture rate, at exact
$p \le 0.067$. At forty-four seeds it is $+0.31$ to $+0.80$, with no window reaching
$p < 0.13$. The diagnosis is in the predictor and not in the correlation: the early copy rate
separates the five models by $0.333$ on three topics and by $0.061$ on forty-four, so at the
larger sample there is nothing left to order with. Found by applying this paper's own
argument to its own row. The ordering claim is regraded and the mechanism recorded with it.
\end{warnbox}

\begin{claimbox}
\textbf{Claim 2's ordering is exact, and three seeds could only show the weaker number.} The figure
in \S\ref{sec:bimodal} reports Spearman $+0.82$ ``with one adjacent inversion, not exact
reproduction'', all a simulation compared against three observed seeds can support ($p = 0.133$).
Against the $44$-seed capture rates the same construction gives $+0.90$ from the published curve and
$+1.00$ at $p = 0.017$ from the curve measured on the topics it predicts, the only ordering in this
paper reaching conventional significance. \emph{The inversion was an artefact of the three-seed
denominator on the observed side.}
\end{claimbox}

\paragraph{The estimator behind claim 2, bounded rather than described.} $\gamma$ is measured cell by
cell and cells the runs never visited are filled from the nearest measured cell below, which
\S\ref{sec:gammaout} called ``the conservative direction''. The natural objection is that arm A's
one-sided failure is an artefact of that rule. It is testable, since monotonicity gives a two-sided
bracket for nothing: filling from below is a pointwise lower bound on the curve and filling from
above a pointwise upper bound, so simulating both brackets what a fully measured curve would have
produced. \textbf{The objection fails, for a reason that also costs this paper a diagnosis.} The fill
rule invents cells for \texttt{qwen3-8b} alone, the one model arm A predicts exactly, so for the four
models that under-predict the bracket has \emph{zero width}, and support-weighted isotonic regression
moves arm A's error from $12.4$ to $12.3$ seed units. The fill rule is not the diagnosis of arm A's
one-sidedness; the diagnosis is the one \S\ref{sec:gammaout} already had, that three topics are not
forty-four. A second correction runs the other way. The $z$ scores in Table~\ref{tab:gammaout} divide
by the binomial error of the \emph{observed} rate and treat the simulated rate as exact, when it is a
function of a $\gamma$ estimated from three seeds; with a seed-clustered bootstrap on that side too,
arm A's observed rate falls inside its own band for three models of five, so its failure is
established for two and not for the five the $z$ scores imply. Arm B's is inside for five of five,
what a landed prediction looks like. Appendix~\ref{app:gammaest} has all four estimators.

\noindent The transferable form is a rule about which quantity to spend on. Ask first
whether the criterion is a level or a contrast. A level over $m$ models sharing $s$ seeds buys
$sm/(1+(m{-}1)\rho)$ observations and here $\rho \approx 0.65$, so models are nearly free
information; a within-seed contrast is not penalised and can be run at few seeds, provided the
criterion is checked seed by seed rather than pooled. Neither rule rescues a per-model
\emph{ordering}, whose independent variable is itself an estimate; that needs tens of seeds
before it can carry a comparison, which is what claim 4 has just paid for.
Appendix~\ref{app:seedscale} gives every table behind this section.

% =====================================================================
\section{Discussion}
\label{sec:discussion}

\subsection{Relation to the parent framework}
\label{sec:parent}

This paper takes three things from the framework it descends from \citep{cch2026} and deliberately
takes no fourth. \textbf{The object}: that framework's subject is a pair of channels, one compressive
and composable, one a scalable verbatim index, and its claim is about the \emph{class} holding both;
our loop is an index channel with no state channel, and \S\ref{sec:sched} adds the smallest state
channel the apparatus admits. \textbf{The question}, from the layer the parent marks untested: it
flags the systems-level version, including the cost-matched comparison, as open, and E-A is that
comparison at the smallest scale where it can be made. \textbf{No constant}: the family already
contains a measured instance of a coefficient failing to cross regimes, a training-time study of ours
having transplanted the parent's load accounting and had it excluded with the direction reversed
\citep{dpd2026}, so nothing here is calibrated against anything there.

\begin{claimbox}
\textbf{What we return is a bound, not a verdict.} The parent's claim is about a class; our apparatus
instantiates its smallest member, a single self-fed slot. E-A's tie bounds \emph{that member} and
says nothing about the class: a richer state channel, or one fed by the stream rather than by the
system's own output, is untouched. What the tie does establish is a boundary on the object itself,
since a \emph{self-fed} slot is not a state channel in the parent's sense, whose channel is written
by the stream. In this paper's language, a state channel that reads only its own writes is another
cycle through the same loop, and \eqref{eq:dyn} applies to it unchanged.
\end{claimbox}

% =====================================================================
\subsection{Limitations}
\label{sec:limits}

Each derivation carries an assumption, and the useful form of a limitation is what breaks when
that assumption does.

\paragraph{(i) Append-only is load-bearing, and the remedy this paper recommends violates it,
but a 2026 remedy exists that does not, which narrows the claim.}
\eqref{eq:ceiling} needs a store that never removes a record: that is what makes $f<1$ strictly, what
puts the captured runs on a computable ceiling, and what makes \emph{Score} $=C\cdot s$ a statement
about mismeasurement rather than loss. Under eviction all three fail together. This is uncomfortable
because bound (ii) of \S\ref{sec:limit} recommends exactly the stores that break it, forgetting being
the only intervention that survives unbounded tenure, so the two halves of the paper hold in adjacent
regimes. The adjacent regime is not unmeasured: \citet{romerl2026} take the same limit over a fixed
coordinate set whose contents are replaced, and get a \emph{steady-state occupancy} of erroneous
coordinates, the object append-only forbids. The division is clean and the hole is theirs to fill and
ours to name: no bound in this paper survives eviction, and no occupancy of theirs is defined without
it. What we still owe is a measurement spanning the two, a deployed store being usually append-only
with a bounded window rather than either extreme.

\emph{The claim in this paragraph's title was too strong when written, and we narrow it here rather
than in a footnote.} \citet{beliefmem2026} attack the same failure (an agent that commits to one
conclusion per observation, acts on it, never revisits the alternatives and thereby reinforces the
error) \textbf{without removing anything}: each candidate conclusion is stored as its own entry,
every entry carries a probability updated by a Noisy-OR rule as observations arrive, and retrieval
surfaces the candidates together. A remedy for the failure we study exists that is \emph{compatible}
with append-only, and ``the recommended remedy violates it'' is true of \emph{forgetting}, the remedy
bound~(ii) of \S\ref{sec:limit} recommends, not of remedies in general.

\textbf{What our own limit says about it is sharper than a citation, and it is not a refutation of
either paper.} Prop.~\ref{prop:ceiling} is a statement about what an append-only store
\emph{contains}: no truth record can leave it, so the ceiling stands under \citet{beliefmem2026}
unchanged. Their store is append-only in exactly the sense the proposition needs. What their
design changes is which functional of the store the agent acts on: not a count of records but
a probability attached to them. The question our apparatus would ask is whether the mutable
part has simply moved, from the record set to the weight vector, and whether an accumulating
majority of false observations drives the Noisy-OR weight the way it drives our $f$. We do not know,
and \textbf{we decline to guess, because the answer depends on an update rule whose behaviour under
a growing false majority is a derivation nobody in this paper has done.} The experiment that would
answer it, one substitution in \texttt{agent.py} and no new data, is named and priced in
\S\ref{sec:owed1}(1).

\paragraph{(ii) The bimodality has no mechanism, and we know which one is missing.} The variance half
of \eqref{eq:varsum} predicts a spread falling as $1/\sqrt{n_0}$ and the data says $1.07\times$ over
an eightfold range. What fails is conditional independence between successive reads, so an
exchangeable-increment account is excluded and a correlation model is required. Until there is one,
``the outcome is a boundary choice'' is a description with a measured power requirement attached, not
an explanation.

\paragraph{(ii$'$) The decision timescale has no derivation, and this paper removed the one it
had.} \S\ref{sec:early} measures that $\mathrm{sign}(f_t-f_0)$ at $t{=}2$ of $150$ already
predicts the outcome, which we had explained by \eqref{eq:dyn}'s $1/n$. Prop.~\ref{prop:step}
removes it: the cumulative bound permits either endpoint until $t{=}49$, so it is vacuous over the
entire interval in which the decision is taken. We state this as an open problem rather than a
caveat, because it is sharp and decidable: what concentrates the outcome into the first few steps,
when neither the drift's sign nor the step bound distinguishes those steps from any others?
Limitation (ii) says the shape has no mechanism; this says the \emph{timing} has none either, and the
two are plausibly one missing account of the correlation between successive reads.

\paragraph{(iii) The budget transient is derived and untested, and our criterion sits on the
wrong side of it.} \eqref{eq:dilution} puts the peak near $R/T\approx0.5$, so the registered
falsifier, which expected the effect to shrink between $R{=}15$ and $R{=}45$, was set on the rising
side. We report it as frozen and then ran the sweep past the peak on $44$ fresh seeds. It did not
resolve, at $+0.46\sigma$, not for want of seeds but because the effect fell from $+0.115$ at three
seeds to $+0.021$ at forty-four while the interval refused to shrink as fast as $1/\sqrt{n}$.
The bill for $3\sigma$ went up, not down, when we paid the first instalment: about $1895$
seeds. The peak's location is marked not resolvable at $\le 1980$ runs and left there; the
\emph{ordering} never needed it.

\paragraph{(iv) The $f_0$ sweep is bought but coarse, and the divide is still not located.} The
sweep is run (Table~\ref{tab:f0}) at four levels, $20$ seeds each, one model rather than five since
\S\ref{sec:power}'s own finding is that this apparatus should buy seeds. It settles that the
empty interval is not an artefact of $f_0$ and removes the reading that $f_0{=}0.9$ is the
adversarial extreme, but not \S\ref{sec:boundaries}'s question: capture runs
$0.500/0.700/0.750/0.682$, which is not monotone, and four points on one model cannot locate a basin
boundary. A finer grid on more models would locate it (\S\ref{sec:owed1}(4)), and it is no longer
the cheapest arm this paper names; the missing correlation account is.

\paragraph{(v) Uniform retrieval was an assumption and is now a measured restriction; $w>1$ and
the retriever family remain unknown.} Retrieval is uniform-without-replacement within a single
tracked topic, the assumption \eqref{eq:dyn} makes, and $w{=}1$ is what makes $n=n_0+t$ and hence
\eqref{eq:dilution} exact. \S\ref{sec:retriever} removes the first half of this ignorance, so every threshold above is scoped to
uniform retrieval by name. Two things remain. One retriever is one point: \textsc{MiniLM} cannot
separate one-digit-different numerals, and the probe that sized the apparatus put the truth last of
three candidates at $0.9076/0.9115/0.9143$, so a character-level or hybrid index is a different
experiment. $w>1$ and multiple tracked topics change terms nobody here has varied.

\paragraph{(vi) The capacity coupling is vacuous here.} Converting the discrimination accuracy
contamination control requires into bits \citep{cover2006elements} gives $\alpha_{\min}=0$ for four
of five models, self-cleaning at every contamination level with no provenance at all. The derivation
stands over an empty region, and the surviving statement is a design rule: measure $\alpha_{\min}$
before spending anything on provenance.

\paragraph{(vii) The synthetic world is bounded rather than defended, and the incidence question
is open.} Truth must be constructible for propagation to be traceable, and that premise is now
measured: on \S\ref{sec:real}'s empty-store probe the synthetic topics return $p_0 = 0.000$, so
the published $\gamma$ was a zero-prior measurement. Placed on a real axis the threshold survives
at $353/360$ against $39/40$ on the synthetic control (denominators that differ ninefold, so
this is agreement on real facts rather than a matched comparison), while the copy probability does
not, $\gamma(1)$ being $1.000$ on synthetic topics against $0.700$ at the lowest-prior real cell.
The idealisation is an upper bound on the copy probability, not an instance of it, and
P-T1-N3 records that as a registered failure. Whether deployed stores instantiate these dynamics
in large measure remains unestablished, and by \S\ref{sec:parent} it is not inherited from the
parent either.

% =====================================================================

\paragraph{Artifact availability, and what stating it costs this paper.} Every preregistration,
grader, runner and result file named here exists in a repository released publicly \emph{after}
acceptance; until then a reader has this document's timestamp and nothing else. By the companion
paper's definition of a \emph{locatable} test \citep{t2library2026}, a criterion a reader cannot
reach is indistinguishable from an untested claim, so the locatable share of every criterion
in this paper is $\mathbf{0}$ for everyone but us while the repository is private, whatever our own
commit history says. That is not a caveat we are adding; it is a reading of the second wall
applied to ourselves at the one moment it is expensive. What the preprint buys is the other half, a
timestamp issued by someone who is not the claimant. The ordering is: criteria frozen in private
commits, anchored by this preprint's date, made locatable on release. We quote no coefficient from
that companion: its $c_1$ is a property of a \emph{cross-paper} ledger rather than of one paper's
criteria, and \S\ref{sec:parent} has already committed this paper to taking no constant from the
family.

\section{What remains owed}
\label{sec:owed1}

\emph{None of the five arms below can move this paper's two results.} Prop.~\ref{prop:ceiling} is a
statement about what an append-only store contains and is proved, not measured; the bimodality of
\S\ref{sec:bimodal} is measured on seeds this paper reports in full. What the arms bear on is
\emph{mechanism} (why the outcome concentrates when it does) and how far the account transports
to stores that are append-only in the record and mutable in the read-out.

\textbf{(1) The Noisy-OR read-out, the cheapest arm this paper can name.} Hold this paper's
append-only store and its planted falsehood fixed, replace the retrieval read-out with a Noisy-OR
weighted one \citep{beliefmem2026}, and ask whether the bimodality survives. It is one substitution
in \texttt{agent.py} and no new data. It settles whether the mutable part has merely moved from the
record set to the weight vector; we decline to guess, the answer depending on an update rule's
behaviour under a growing false majority and that derivation not being done here.

\textbf{(2) A correlation model for successive reads, which is one account and not two.} The variance
half of \eqref{eq:varsum} predicts a spread falling as $1/\sqrt{n_0}$ and the data gives
$1.07\times$ over an eightfold range, so conditional independence fails. The same missing account
plausibly explains the \emph{timing}: Prop.~\ref{prop:step} shows the cumulative bound is vacuous over
the whole interval in which the decision is taken, yet the sign at $t{=}2$ of $150$ already predicts
the outcome. We state both as open problems rather than caveats, because each is sharp and
decidable.

\textbf{(3) The peak's location, priced and left.} \eqref{eq:dilution} puts the peak near
$R/T\approx0.5$ and the sweep past it did not resolve at $+0.46\sigma$; the bill for $3\sigma$ is
about $1895$ seeds, and it rose rather than fell when the first instalment was paid. The peak is
marked not resolvable at $\le 1980$ runs. The ordering this paper reports never needed it.

\textbf{(4) A finer $f_0$ grid on more models, and the retriever family at $w > 1$.} Four levels on
one model settle that the empty interval is not an artefact of $f_0$ but cannot locate a basin
boundary; uniform retrieval within a tracked topic is now a measured restriction rather than an
assumption, and behaviour at $w > 1$ is unmeasured.

\textbf{(5) A stream-fed slot, the one arm that could return a verdict on the parent framework rather
than on us.} \S\ref{sec:parent} takes the parent's untested systems-level layer as this paper's
question and \S\ref{sec:sched} answers it at the smallest scale the apparatus admits: a
\emph{self-fed} slot, which tied. That tie bounds one member of the parent's class and is not
evidence about the class, the parent's state channel being written by the stream and ours by the
agent's own previous answer. The arm is the same comparison with the slot fed from the record stream
at write time, cost-matched at the line level exactly as E-A was, and its outcome is the only
quantity named in this family that could move the parent's own status rather than describe our
instrument. We state it and do not run it: the apparatus needs a write-side hook it does not have,
and the direction of the result would not touch either of this paper's two results, both
statements about an append-only store under an index read-out and neither mentioning a state
channel.

% =====================================================================
\section{Conclusion}
\label{sec:conclusion}

A store an agent writes and later reads is a closed loop, and because writing never deletes, its
dynamics are constrained before any model is chosen: the reachable states have a hard upper edge,
information is diluted rather than lost, and what holds the outcome down is a bounded grounding
budget rather than a fixed fraction of verified data. What we measured is a shape rather than a
curve. Outcomes are bimodal; a copy function calibrated once and fitted nowhere orders the models
within a seed; a falling score reports commitment rather than capability, at a mismeasurement of
$+0.383$; and at matched budget, when grounding is spent dominates how much of it there is.

The other half of the paper is the programme turned on itself, and it is where the design rules come
from. Three axes of our own died to preregistrations we wrote first; a consistency gate armed the
false majority, a result about admission control rather than about us; and a provenance control
produced an apparatus failure graded \textsc{untested} rather than allowed to read as support.
\emph{The most consequential reaches past this paper}: the resampling unit is the seed, the seed
selects the topic, the topic dominates the variance, and a variance share reading $2.0\%$ at three
seeds reads $68.8\%$ at nine. No verdict changed, every frozen criterion being a point estimate
against a threshold, but every summary interval was built on the wrong unit, so an apparatus of this
shape should buy seeds and not models.

Two registered arms turned the same treatment on the apparatus itself. Uniform retrieval is not
innocuous: an embedding ranking moves capture by $0.500$, and a sweep of frozen retrieval slots
reproduces both endpoints while refuting the variance account we offered for them, leaving the
displacement with two proposed and refuted mechanisms. The synthetic world, this paper's largest
exposure, is bounded rather than defended: on $36$ Wikidata facts the threshold retains $353/360$
agreement, with the same-batch synthetic arm at $39/40$ on a ninefold smaller denominator, while the
copy probability under a unanimous false store runs $0.300$ lower on real facts, so the idealisation
overstates the risk by a measured amount and preserves the prediction.

For practice: budget grounding by schedule, measure the store rather than the agent, and do not
read a falling exact-match score on an append-only store as capability loss. For method: the
resampling unit is whatever selects the world, and if that is the seed then models are the cheap
axis and seeds the expensive one. We spent the budget the other way round, and the only reason we
can say so is that the criteria were frozen before the data existed.

% =====================================================================
\bibliographystyle{plainnat}
\bibliography{refs}

% This one STAYS a \clearpage. It is the bibliography-to-appendix boundary, and an
% \appendix opening in the white below the last reference reads as a run-on rather than
% as a new part of the document. The seven barriers this pass replaced were all inside the
% appendix, between sections of it.
\clearpage

\appendix

\section*{Guide to this appendix}

\noindent The appendix is long because the paper grades seventy-two rows against nineteen
pre-registrations and 4{,}602 runs, and a graded row a reader cannot check is not graded. It is not
one block: every section below is a named topic, nothing nested more than one level. Three groups,
and a reader checking the paper rather than reading it wants the third.

\begin{enumerate}[leftmargin=1.6em,itemsep=1pt,topsep=2pt,label=(\roman*)]
\item \textbf{Arms that carry a claim in the body and are reported in full here}
  (\S\ref{app:moved} onward): the $220$-run re-test of the empty interval, claim 2b out of sample,
  the frontier and $f_0$ scope arms the abstract quotes, the budget sweep past the peak, the
  accidental replicate, and the resampling-unit result that is contribution~(v). \emph{These are
  here for length, not for weight.}
\item \textbf{Derivations, and the accounts the data refuted}: the two limits of the dilution law, the variance bound
  and the sweep refuting it, the proofs, and the three axes we killed.
\item \textbf{The record, generated at build time and not transcribed} (\S\ref{app:prereg}
  onward): every criterion with the commit that froze it, the invariants checked on all $7{,}244$
  runs, per-model detail behind every pooled number, the four analyses this round added, and what
  the whole record cost. \emph{Every table in this group is written from \texttt{results/} by
  \texttt{make\_appendix.py}, so an appendix that disagrees with the repository is a build failure
  rather than a reading someone has to catch.}
\end{enumerate}

\noindent \textbf{Where to start if you are checking rather than reading.} \S\ref{app:seedscale}
is the round-8 knife applied to the rows it had not reached, \S\ref{app:gammaest} is the four
estimators for $\gamma$, and \S\ref{app:prereg} is the freeze record every criterion ID resolves
against. Those three are the ones that would change a verdict if they were wrong.

\section{E-mean: a single centre}
\label{app:mean}

\noindent The arm behind \S\ref{sec:related}'s statement that the defect is in the mean rather than in summarising,
graded in Table~\ref{tab:mean}.

\begin{table}[!htbp]
\centering\small
\caption{E-mean, $220$ runs, criteria frozen before the data. ``Share described'' is the fraction
of runs within $\pm 0.10$ of the centre, a half-width fixed in the prereg at about four seed
units. The last row bounds the whole column. The whole tabular is generated from
\texttt{results/MEAN.json}.}
\label{tab:mean}
% GENERATED by t1_grounding/figures/make_t1_tikz.py from results/MEAN.json.
\begin{tabular}{@{}llrrl@{}}
\toprule
& centre & value & share described & verdict \\
\midrule
P-T1-M1 & pooled mean & $0.792$ & $18/220 = 0.082$ & \textbf{supported} (bar $\le 0.40$) \\
P-T1-M2 & median & $0.981$ & $150/220 = 0.682$ & reported; \emph{narrows M1} \\
         & best single centre & $0.894$ & $164/220 = 0.745$ & upper bound on any one number \\
\bottomrule
\end{tabular}

\end{table}

\section{Three axes and their falsifiers}
\label{app:killed}

\begin{scopebox}
Each of the following was a post-hoc reading of data in hand, turned into a pre-registration with
fresh seeds and then refuted by it; \S\ref{sec:killed} says why they are listed and does not
need repeating here. What this appendix adds is each one in full.

\textbf{(a) Damage scales with unaided cleaning ability.} Three interventions appeared to hurt
the cleanest model most, at rank correlation $+1.00$ over five models. Registered on seven fresh
seeds with the generating data excluded, the ordering survives ($0.264$ vs $0.249$) and the
magnitude does not, the frozen bar being $0.30$ against a measured gap of $0.015$. The diagnosis
is in the generating data: the ``cleanest model'' had a $3$-seed baseline of $0.150$ that becomes
$0.729$ on seven seeds, so the independent variable was noise. This is the failure mode of a
superseded model of ours in which fixed points computed from $12$-sample cells moved to zero when
recomputed from $188$.

\textbf{(b) A marker's effect is a single function of label--truth correlation.} Three points
(anti-correlated, scrambled, correlated) fell monotonically in five of five models. Two new arms
at partial correlation, whose positions the hypothesis \emph{predicts}, break it: strict
five-point monotonicity holds in $0/5$, every model rising at the scrambled point above the
partially anti-correlated one. Refuted; the surviving statement is the weaker P-L2.

\textbf{(c) A sign reversal unifying two timescales.} We proposed that one-shot aggregation
failure and the store-level ratchet are the same mechanism at $\tau{=}1$ and $\tau\to\infty$,
distinguished by whether a better model is safer or more exposed. That rests on (a) and dies with
it.
\end{scopebox}

\section{E-retriever: three costs}
\label{app:retriever}

\S\ref{sec:retriever} states the ranker-family check and its result in full: no ranker of the
three separates a truth record from its one-digit corruption, so the mechanism is structural
rather than the embedder's. It is not repeated here. What this appendix adds is the arm's cost,
which the body gives in one sentence.

\begin{warnbox}
\textbf{Three costs of this arm, two of them ours and one paid in dollars.}
\emph{(i)} P-T1-R3's ladder has almost no dynamic range and is not monotone: $k_{\rm eff}$ runs
only $1.00$ to $1.95$ across $k=1$ to $8$, with capture $0.000$, $0.400$, $0.233$, $0.533$ along
it. The amendment predicted this compression before the data and reduced R3's weight in advance;
an interpolation across a non-monotone four-point ladder cannot carry much either way.
\emph{(ii)} \textbf{P-T1-R4 is \textsc{supported} but does not mean what its name says.} The
deduplicated arm displaces capture by $0.367$ against ranking's $0.500$, so deduplication does
reduce the displacement; but the same amendment recorded that with only two distinct on-topic
values, one-record-per-value fills two slots and pads the other six from off-topic, measured here
at $59.8\%$. That arm tests dedup-\emph{plus}-padding, and its $k_{\rm eff}$ of $3.22$ is
distractor variety rather than evidence.
\emph{(iii)} The run halted on its own spend guard at $\$5.06$ against a hard stop of $\$5.00$,
having cost $\$5.13$ against a preregistered estimate of $\$2.06$. Thirty-three runs were
dropped, all in one model's high-contamination cells, so \textbf{the floor clause excludes
\texttt{qwen3-8b} entirely and two of three models are judged.} Every number above is on $n=2$
models $\times$ $10$ seeds.
\end{warnbox}

\noindent \textbf{One post-hoc observation, because it points at this paper's own headline.} If
ranking is a variance sink, its signature is a thinner middle rather than a shifted mean. Using
\S\ref{sec:bimodal}'s own cuts, unrefitted, the share of runs inside the empty interval
$(0.306, 0.512)$ is $0.050$ under the control and $0.000$ under both ranking arms. The direction
agrees, but two runs against zero out of forty is an observation and not evidence. It is recorded
because it is the cheapest thing to register next.

\section{Two limits of the dilution law}
\label{app:limits}

\noindent This appendix carries items (iii) and (iv) of \S\ref{sec:limit} in full. Both are
derived from \eqref{eq:dilution} and neither is measured in this paper.

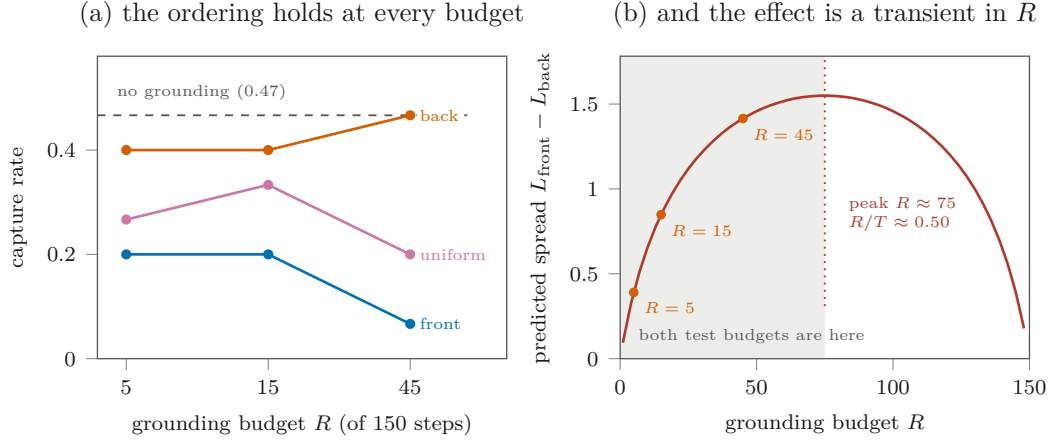
\begin{figure}[!tbp]
\centering
% GENERATED by t1_grounding/figures/make_t1_tikz.py -- do not edit by hand.
% Every coordinate is recomputed from t1_grounding/results/ at generation time.

\begin{tikzpicture}
\begin{axis}[name=pa, width=0.34\linewidth, height=4.0cm, axis on top, tick align=outside, tick pos=left, axis line style={figSoft, line width=0.5pt}, tick label style={font=\scriptsize, color=figInk}, label style={font=\scriptsize, color=figInk}, title style={font=\small, color=figInk, yshift=1pt}, grid=none, scale only axis, xmode=log,
  xmin=4, xmax=95, ymin=0, ymax=0.58, xtick={5,15,45}, xticklabels={5,15,45},
  log ticks with fixed point,
  xlabel={grounding budget $R$ (of 150 steps)}, ylabel={capture rate},
  title={(a) the ordering holds at every budget}]
\addplot[figSoft, dashed, line width=0.7pt, forget plot]
  coordinates {(4,0.4667) (70,0.4667)};
\node[font=\tiny, figSoft, anchor=south west] at (axis cs:4.3,0.477)
  {no grounding (0.47)};
\addplot[figA, mark=*, mark size=1.4pt, line width=1.0pt] coordinates {(5.0000,0.2000) (15.0000,0.2000) (45.0000,0.0667)};
\node[font=\tiny, text=figA, fill=white, inner sep=1pt, anchor=west] at (axis cs:47.25,0.067) {front};
\addplot[figD, mark=*, mark size=1.4pt, line width=1.0pt] coordinates {(5.0000,0.2667) (15.0000,0.3333) (45.0000,0.2000)};
\node[font=\tiny, text=figD, fill=white, inner sep=1pt, anchor=west] at (axis cs:47.25,0.200) {uniform};
\addplot[figB, mark=*, mark size=1.4pt, line width=1.0pt] coordinates {(5.0000,0.4000) (15.0000,0.4000) (45.0000,0.4667)};
\node[font=\tiny, text=figB, fill=white, inner sep=1pt, anchor=west] at (axis cs:47.25,0.467) {back};
\end{axis}
\begin{axis}[name=pb, at={(pa.east)}, xshift=1.5cm, anchor=west,
  width=0.34\linewidth, height=4.0cm, axis on top, tick align=outside, tick pos=left, axis line style={figSoft, line width=0.5pt}, tick label style={font=\scriptsize, color=figInk}, label style={font=\scriptsize, color=figInk}, title style={font=\small, color=figInk, yshift=1pt}, grid=none, scale only axis,
  xmin=0, xmax=150, ymin=0, ymax=1.782,
  xlabel={grounding budget $R$},
  ylabel={predicted spread $L_{\rm front}-L_{\rm back}$},
  title={(b) and the effect is a transient in $R$}]
\addplot[draw=none, fill=figGrid, fill opacity=0.45, forget plot]
  coordinates {(0,0) (75,0) (75,1.782) (0,1.782)}
  \closedcycle;
\addplot[figPred, line width=1.0pt] coordinates {(1.000,0.094) (4.000,0.326) (7.000,0.507) (10.000,0.654) (13.000,0.777) (16.000,0.881) (19.000,0.971) (22.000,1.050) (25.000,1.119) (28.000,1.180) (31.000,1.233) (34.000,1.281) (37.000,1.324) (40.000,1.362) (43.000,1.395) (46.000,1.424) (49.000,1.450) (52.000,1.473) (55.000,1.492) (58.000,1.508) (61.000,1.522) (64.000,1.533) (67.000,1.541) (70.000,1.546) (73.000,1.549) (76.000,1.550) (79.000,1.547) (82.000,1.543) (85.000,1.536) (88.000,1.526) (91.000,1.513) (94.000,1.498) (97.000,1.480) (100.000,1.458) (103.000,1.433) (106.000,1.405) (109.000,1.373) (112.000,1.337) (115.000,1.296) (118.000,1.250) (121.000,1.198) (124.000,1.140) (127.000,1.074) (130.000,0.999) (133.000,0.913) (136.000,0.813) (139.000,0.697) (142.000,0.559) (145.000,0.391) (148.000,0.178)};
\addplot[figPred, dotted, line width=0.7pt, forget plot]
  coordinates {(75,0.310) (75,1.782)};
\node[font=\tiny, figPred, align=left, anchor=west] at (axis cs:80,0.852)
  {peak $R\approx75$\\$R/T\approx0.50$};
\node[font=\tiny, figSoft, anchor=west] at (axis cs:3,0.132)
  {both test budgets are here};
\addplot[only marks, mark=*, mark size=1.6pt, figB] coordinates {(5,0.3907)};
\node[font=\tiny, figB, anchor=north west] at (axis cs:5,0.3907) {$R=5$};
\addplot[only marks, mark=*, mark size=1.6pt, figB] coordinates {(15,0.8482)};
\node[font=\tiny, figB, anchor=north west] at (axis cs:15,0.8482) {$R=15$};
\addplot[only marks, mark=*, mark size=1.6pt, figB] coordinates {(45,1.4150)};
\node[font=\tiny, figB, anchor=north west] at (axis cs:45,1.4150) {$R=45$};
\end{axis}
\end{tikzpicture}
\caption{\textbf{The ordering survives every budget; the effect's size is a transient we did not
test.} (a) Capture rate by schedule at $R \in \{5,15,45\}$ against the ungrounded control; at
$R{=}45$ back-loading has returned to the control exactly. (b) The closed form
\eqref{eq:dilution}: the predicted spread vanishes at $R{=}0$ and at $R{=}T$, where all three
schedules are one arm, so it peaks near $R{=}75$. Both budgets we registered a falsifier on sit
on the rising side, which is why that falsifier failed upward.}
\label{fig:budget}
\end{figure}

\textbf{(iii) The budget dependence is a transient, derived but not measured.}
$\mathrm{spread}(R) = L_{\mathrm{front}} - L_{\mathrm{back}}$ vanishes at both ends: at $R{=}0$
nothing is grounded, and at $R{=}T$ every step is, so the three schedules are one arm (Table~\ref{tab:limits}, Figure~\ref{fig:budget}). A quantity
positive in between and zero at both ends has an interior maximum, and \eqref{eq:dilution} puts
it near $R/T \approx 0.5$: $\mathrm{spread}$ runs $0.39, 0.85, 1.42, 1.55, 1.22, 0$ at
$R = 5, 15, 45, 75, 120, 150$.

\textbf{(iv) The transient is in the fraction; in the tenure the effect is unbounded, and the two
limits do not commute.} Item (iii) is stated at fixed $T$ and read alone invites the conclusion
that schedule stops mattering once the budget is large. The apparatus's own limit says the
opposite. Write $\alpha = R/T$ and let $T \to \infty$ with $\alpha$ fixed. Then
\begin{equation}
L_{\mathrm{back}} \;\longrightarrow\; \ln\frac{1}{1-\alpha}
\quad\text{(a constant)},
\qquad
L_{\mathrm{front}} \;\approx\; \ln\frac{n_0 + \alpha T}{n_0} \;\longrightarrow\; \infty
\quad\text{(like $\ln T$)}.
\label{eq:limits}
\end{equation}
$L_{\mathrm{back}}$ carries a hard ceiling no tenure can lift, $0.105$ at $\alpha{=}0.1$ and
$0.693$ at $\alpha{=}0.5$, reached to three decimals by $T{=}15000$, while $L_{\mathrm{front}}$
has no ceiling. Three consequences follow from \eqref{eq:dilution}, none needing a new run:

\begin{table}[!htbp]
\centering\small
\caption{The three limits of \eqref{eq:dilution} at $n_0{=}10$. The peak's \emph{location} is
scale-free and its \emph{height} is not, which is why (iii) and (iv) are both true. The last row
is the only place the three schedules coincide.}
\label{tab:limits}
\begin{tabular}{@{}p{0.17\linewidth}p{0.28\linewidth}p{0.28\linewidth}p{0.15\linewidth}@{}}
\toprule
limit taken & what happens to the spread & at $T{=}150/1500/15000$ & regime \\
\midrule
$R$ fixed, $T\to\infty$        & back $\to 0$, front $\to$ const.
  & --- & finite budget, long tenure \\
$\alpha$ fixed, $T\to\infty$   & diverges like $\ln T$
  & $1.55 / 3.69 / 5.98$ at $\alpha{=}\tfrac12$ & \textbf{ours} \\
peak over $R$ at fixed $T$     & at $R^\ast/T = 0.50$ always
  & height $\approx \ln\frac{n_0+T/2}{n_0} - \ln 2$ & --- \\
$\alpha = 1$ exactly           & identically $0$, every $T$
  & $0 / 0 / 0$ & free re-grounding \\
\bottomrule
\end{tabular}
\end{table}

\emph{First}, the peak of (iii) sits at $R^\ast/T = 0.50$ at $T = 150$, $1500$ and $15000$ alike:
its location is a fraction and does not move with tenure. \emph{Second}, its height does move,
growing as $\ln T$ ($1.55$, $3.69$, $5.98$). A transient in $\alpha$ and an unbounded quantity in
$T$ are therefore not in tension; they are one surface read along two axes. \emph{Third}, and
this is the sharp one, the spread diverges for \emph{every} $\alpha < 1$ and is identically zero
at $\alpha = 1$, so
\[
\lim_{\alpha \to 1}\ \lim_{T \to \infty} \mathrm{spread}(\alpha, T) \;=\; \infty
\qquad\text{while}\qquad
\lim_{T \to \infty}\ \lim_{\alpha \to 1} \mathrm{spread}(\alpha, T) \;=\; 0 .
\]
The order of the limits decides the answer. Free re-grounding is not a regime this one approaches
smoothly as the budget grows; it is a single point at which the schedule axis collapses, and one
step off it the axis reopens without bound.

\begin{warnbox}
\textbf{A frozen criterion of ours embeds an assumption this derivation exposes.} We registered
``the schedule effect must shrink as the budget grows'' as the falsifier for the boundary claim
and picked $R \in \{5,45\}$ to test it. The derivation, written after the criterion was frozen
and before any of that data was read, puts the peak near $R/T \approx 0.5$, so \emph{both}
$R{=}15$ and $R{=}45$ lie on the rising side and the mechanism predicts the spread to
\emph{grow}. The criterion was set on the wrong side of a transient, and item (iv) says why the
mistake was easy to make: ``larger budget'' is ambiguous between $R$ at fixed $T$, where the
effect is a transient, and $\alpha$ at growing $T$, where it is unbounded. We wrote the falsifier
for the first reading and the apparatus runs in the second. We report it as frozen, do not
rewrite it, and do not treat a failure in the predicted direction as a confirmation: confirming a
transient needs a sweep past $R/T > 0.5$, which we have not run. Claim (iii) is therefore
\textbf{derived and untested}, and \texttt{prereg/E\_BUDGET\_PREREG.md} records when each half
was written.
\end{warnbox}

\section{The variance bound under sweep}
\label{app:variance}

\noindent The decomposition behind \S\ref{sec:limitbimodal}'s refutation of the $1/\sqrt{n_0}$ bound.

\paragraph{The natural derivation, and the bound it predicts.} One write moves $f$ by
$(X-f)/(n{+}1)$ with $X \sim \mathrm{Bernoulli}(\geff(f))$, so treating successive writes as
conditionally independent,
\begin{equation}
\underbrace{\textstyle\sum_{n_0}^{N} \frac{\geff(1-\geff)}{n^2} \;=\; O(1/n_0)}
  _{\text{total variance the process can ever accumulate}},
\qquad
\underbrace{\textstyle\sum_{n_0}^{N} \frac{\geff(f)-f}{n} \;\approx\;
  (\geff{-}f)\ln\frac{N}{n_0}}_{\text{total drift}} .
\label{eq:varsum}
\end{equation}
The variance sum \emph{converges}: unlike the drift, which grows logarithmically without bound,
the noise a run can accumulate over an unbounded tenure is finite and set by how small the store
was at the start. That predicts a bound: the spread of the limiting $f$ should fall as
$1/\sqrt{n_0}$, so the two outcomes should merge as the initial store grows (Figure~\ref{fig:varbound}).

\begin{warnbox}
\textbf{The bound is wrong.} Four sweeps at $n_0 \in \{10,20,40,80\}$, three models and three
seeds each, run for $3n_0$ steps so all four cover the same stretch of $n/n_0$, decompose the
spread of $f_{\mathrm{end}}$ into within- and between-model parts:

\smallskip
% Twelve hand-typed three-decimal values, every one of them recomputable from
% results/cl_n*.json, in the appendix whose subject is a bound we derived and then refuted.
\centerline{\small% GENERATED by t1_grounding/figures/make_t1_tikz.py -- do not edit by hand.
% Recomputed from results/cl_n{10,20,40,80}.json at build time.
\begin{tabular}{@{}lcccc@{}}
\toprule
$n_0$ & 10 & 20 & 40 & 80 \\
\midrule
within-model sd & 0.318 & 0.282 & 0.291 & 0.298 \\
$\times\sqrt{n_0}$ & 1.01 & 1.26 & 1.84 & 2.67 \\
between-model sd & 0.084 & 0.083 & 0.038 & 0.072 \\
\bottomrule
\end{tabular}
}
\smallskip

\noindent An eightfold increase in $n_0$ should shrink the within-model spread by
$\sqrt{8}=2.83$. It shrinks by $1.07$. The bound fails flatly rather than marginally, so what
breaks is the conditional-independence step in \eqref{eq:varsum}: successive reads of the same
store by the same model are not independent draws from $\geff$, and whatever correlates them does
not weaken as the store grows. The between-model part is flat, as a property of $\gamma$ rather
than of $n_0$ should be, which is what makes the within-model failure interpretable rather than a
decomposition artefact.
\end{warnbox}

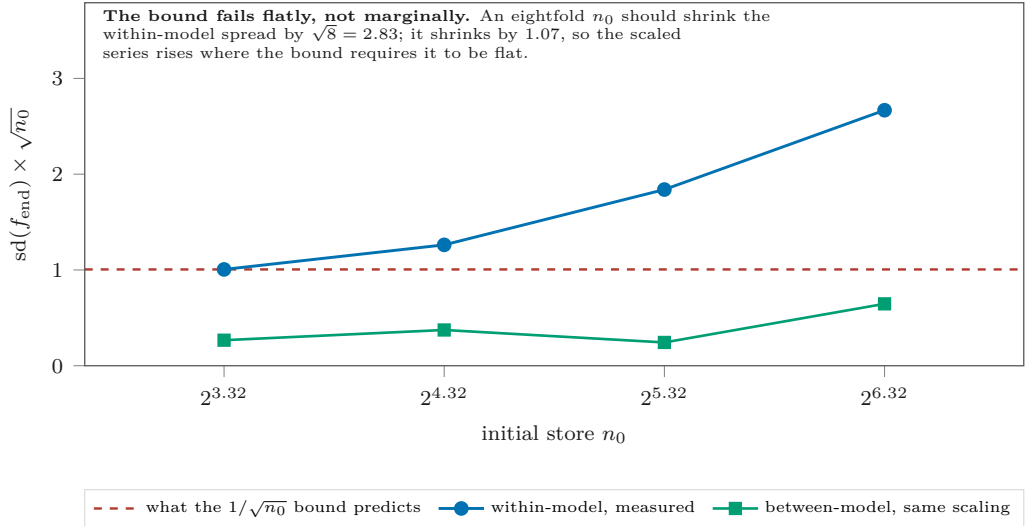
\begin{figure}[!htbp]
\centering
% GENERATED by t1_grounding/figures/make_t1_tikz.py -- do not edit by hand.
% Every coordinate is recomputed from t1_grounding/results/ at generation time.

\begin{tikzpicture}
\begin{axis}[axis on top, tick align=outside, tick pos=left, axis line style={figSoft, line width=0.5pt}, tick label style={font=\scriptsize, color=figInk}, label style={font=\scriptsize, color=figInk}, title style={font=\small, color=figInk, yshift=1pt}, grid=none, scale only axis, width=0.78\linewidth, height=4.8cm, xmode=log, log basis x=2,
  xmin=6.45, xmax=124.0, ymin=0, ymax=3.79,
  xtick={10,20,40,80},
  xlabel={initial store $n_0$},
  ylabel={$\mathrm{sd}(f_{\mathrm{end}}) \times \sqrt{n_0}$},
  legend style={font=\tiny, draw=figGrid, fill=white, at={(0.5,-0.34)}, anchor=north,
    legend columns=3, /tikz/every even column/.append style={column sep=4pt}}]
\addplot[figPred, dashed, line width=0.9pt] coordinates
  {(6.45,1.0055) (124.0,1.0055)};
\addlegendentry{what the $1/\sqrt{n_0}$ bound predicts}
\addplot[figA, mark=*, mark size=2.2pt, line width=1.1pt] coordinates
  {(10,1.0055) (20,1.2620) (40,1.8395) (80,2.6676)};
\addlegendentry{within-model, measured}
\addplot[figC, mark=square*, mark size=2.0pt, line width=1.0pt] coordinates
  {(10,0.2666) (20,0.3733) (40,0.2431) (80,0.6461)};
\addlegendentry{between-model, same scaling}
\node[font=\tiny, align=left, fill=white, fill opacity=0.85, text opacity=1, inner sep=1pt, text=figInk, anchor=north west] at (axis cs:6.76,3.75)
  {\textbf{The bound fails flatly, not marginally.} An eightfold $n_0$ should shrink
  the\\within-model spread by $\sqrt{8} = 2.83$; it shrinks by $1.07$, so the
  scaled\\series rises where the bound requires it to be flat.};
\end{axis}
\end{tikzpicture}
\caption{\textbf{The bound as a flat line, and the measurement that is not one.} If the spread
of $f_{\mathrm{end}}$ fell as $1/\sqrt{n_0}$, multiplying it by $\sqrt{n_0}$ would give the same
value at every store size, which is the dashed line drawn through the first point. The measured
within-model series rises instead: an eightfold increase in $n_0$ shrinks the spread by $1.07$
where the bound requires $2.83$. \emph{The between-model series is drawn on the same scaling and
for the same reason}: flat in the raw units, as a property of $\gamma$ rather than of
$n_0$ should be, so the within-model failure is a failure of the conditional-independence step in
\eqref{eq:varsum} and not an artefact of the decomposition. Recomputed at build time.}
\label{fig:varbound}
\end{figure}

\section{Proofs}
\label{app:proofs}

\begin{proof}[Proof of Prop.~\ref{prop:ceiling}]
Append-only means the record set is non-decreasing under inclusion, so the
$n_0 - n_{\mathrm{false}}(0)$ true records at $t{=}0$ are true records at every $t$, whence
$n(t) - n_{\mathrm{false}}(t) \ge n_0 - n_{\mathrm{false}}(0)$, which rearranges to the middle
term of \eqref{eq:ceiling}. That term is $<1$ whenever the store starts with at least one truth;
the apparatus has exactly one at $f_0 = 0.9$, which is why the ceiling is tight rather than
decorative. Equality holds iff every one of the $n(t)-n_0$ writes after $t{=}0$ was false. For the
budgeted form, a run making $W$ ungrounded and $R$ grounded writes has $n = n_0 + W + R$ and,
grounded writes being true by construction,
$n_{\mathrm{false}} \le n_{\mathrm{false}}(0) + W$. Note what the bound does \emph{not} say: it
constrains the endpoint and says nothing about the path, which is why \S\ref{sec:bimodal} needs a
separate argument for the shape of the outcome distribution.
\end{proof}

\begin{proof}[Proof of Prop.~\ref{prop:dilution}]
A grounded write is true and appends one record, so $f \mapsto f\cdot n/(n{+}1)$. Composing the
$R$ writes and taking logarithms,
$L = \sum_i \log(1+(n_0+t_i)^{-1}) = \sum_i (n_0+t_i)^{-1} + O(\sum_i (n_0+t_i)^{-2})$. The three
closed forms follow from
$\sum_{t=a}^{b-1}(n_0+t)^{-1} = \log\frac{n_0+b}{n_0+a} + O(n_0^{-1})$ applied to
$\{0,\dots,R{-}1\}$, to an arithmetic grid of spacing $T/R$, and to $\{T{-}R,\dots,T{-}1\}$.

\emph{The ordering needs no approximation and no parameter.} The summand $(n_0+t)^{-1}$ is
strictly decreasing, and the three schedules assign $R$ writes to the $R$ smallest indices, an
interior grid, and the $R$ largest. A sum of a strictly decreasing function over the $R$ smallest
indices strictly exceeds the sum over any other $R$-subset, which strictly exceeds the sum over
the $R$ largest; for $0<R<T$ the uniform grid is neither extreme. Hence
$L_{\mathrm{front}} > L_{\mathrm{unif}} > L_{\mathrm{back}}$ for every $n_0$, $T$ and $R$ in
range.
\end{proof}

\begin{proof}[Proof of Prop.~\ref{prop:se}]
The first half is the binomial standard error of a mean of $s$ independent Bernoulli$(p)$ draws.
For the second, index a run by (seed $j$, model $k$) with $\operatorname{Var}(X_{jk})=\sigma^2$,
$\operatorname{Corr}(X_{jk},X_{jk'})=\rho$ for $k\ne k'$ (runs sharing a seed share the topic and
the planted falsehood, which generates $\rho$) and independence across seeds. Then
$\operatorname{Var}\bigl(\tfrac{1}{sm}\sum_{j,k}X_{jk}\bigr) =
\tfrac{\sigma^2}{sm}[1+(m-1)\rho]$, the variance of $sm/[1+(m-1)\rho]$ independent observations.
The bracket is the design effect; it equals $1$ only at $\rho=0$ and grows linearly in $m$, so
adding models to a fixed seed set buys information at a rate that falls to zero.
\end{proof}

\begin{proof}[Proof of Prop.~\ref{prop:step}]
A write changes $n_{\mathrm{false}}$ by $0$ or $1$ and $n$ by exactly $1$, so
$|f(t{+}1)-f(t)| = |n_{\mathrm{false}}(t{+}1) - f(t)(n(t)+1)|/(n(t)+1) \le 1/(n(t)+1)$ with
$n(t) = n_0+t$. Summing and comparing with $\int_t^T (n_0+s)^{-1}\,ds$ gives the logarithm. The
threshold is the least $t$ with $\sum_{s=t}^{T-1}(n_0+s+1)^{-1} < 1$, computed directly.
\end{proof}

\section{The falsifier on \texorpdfstring{$220$}{220} further runs}
\label{sec:bimodal2}
\label{app:moved}

One registered claim is graded \textbf{failed} on exactly the ground that ``the
generating estimate was three seeds of a two-valued variable.'' The split above is three seeds
of a two-valued variable. Same $n$, same kind of outcome, opposite treatment, and
\S\ref{sec:power}, which is also ours, measured that the seed carries $1.33$ runs of
information and that its variance share is $2.0\%$ at three seeds against $68.8\%$ at nine.
Three seeds is precisely the sample in which the seed effect is invisible, so the fifteen runs
above are worth about four independent observations. We therefore froze
\texttt{prereg/E\_BIMODAL2\_PREREG.md}, committed its grader before the data existed, and ran
the identical arm on forty-four seeds that had never been used: $220$ runs, no grounding,
every other setting copied field by field and asserted at load time (Table~\ref{tab:bimodal2}).

\begin{table}[!htbp]
\centering\small
\caption{E-bimodal2, seeds $9$--$52$, $220$ runs, against the fifteen at the same cuts. The
cuts are the paper's own: $0.306$ is the largest escaped run above and $0.512$ the smallest
captured one, so the middle column is the interval this section calls empty.}
\label{tab:bimodal2}
% GENERATED by t1_grounding/figures/make_t1_tikz.py from results/BIMODAL2.json.
\begin{tabular}{@{}lrccc@{}}
\toprule
& $\le 0.306$ & in $(0.306, 0.512)$ & $\ge 0.512$ & share in the middle \\
\midrule
$3$ seeds ($n{=}15$) & $8$ & $0$ & $7$ & $0.000$ \\
$44$ seeds ($n{=}220$) & $39$ & $8$ & $173$ & $\mathbf{0.036}$ \\
\midrule
\multicolumn{4}{@{}l}{uniform baseline for an interval of width $0.206$}
  & $0.206$ \\
\multicolumn{4}{@{}l}{P-T1-B1$'$, frozen before the data}
  & bar $\le 0.10$ \\
\bottomrule
\end{tabular}

\end{table}

\noindent \textbf{P-T1-B1$'$ and P-T1-B2 are both supported, so claim 1 keeps its headline at
$n{=}220$.} The interval holds $3.6\%$ of runs where a single mode spread over the same range
would put $20.6\%$; the two cluster means are $0.161$ and $0.953$, a gap of $0.791$ against a
frozen bar of $0.60$; and the middle band's density is $0.15$ of the clusters' against a bar of
$0.25$. P-T1-B5 confirms no single model carries the verdict, the per-model counts inside
the interval are $2, 2, 2, 2, 0$. Clustering over the forty-four seeds gives a standard error
of $0.0131$ on the middle share against $0.0126$ unclustered, a design effect of $1.08$; that
is far below the $3.75$ \S\ref{sec:power} measured, because this indicator is nearly constant
within a seed while $f_{\mathrm{end}}$ is not.

\begin{warnbox}
\textbf{Those figures moved when the cuts stopped being typed, and the move is in our favour, so
we state it here rather than in a footnote.} The grader's cuts were the constants $0.306$ and
$0.512$, described in its own comment as $\max(\text{escaped})$ and $\min(\text{captured})$ over
the fifteen runs of \S\ref{sec:bimodal}, and they are those two values \emph{rounded down}, so
the run at $0.30625$ that defines the lower edge fell strictly above it and counted as middle. The
two rows of Table~\ref{tab:bimodal2} were then counted against different thresholds: the
forty-four-seed row against the constants ($37/10/173$, share $0.045$) and the three-seed row
against the definition ($8/0/7$). On the definition throughout, the forty-four-seed row is
$39/8/173$ and the share is $\mathbf{0.036}$. \emph{P-T1-B1$'$ is \textsc{supported} under both},
and the corrected number is the more favourable one, which is the direction that obliges us to
print the arithmetic rather than the conclusion.
\end{warnbox}

\noindent \textbf{What did \emph{not} reproduce is the split.} Eight clean against
seven captured became thirty-nine against one hundred and seventy-three. The emptiness of the
interval is a property of the dynamics and it survived a fifteen-fold sample; the
\emph{balance} between the two modes was a property of three topics. Anything in this paper
that reads the near-even split as evidence about how often contamination locks in is
unsupported; what the data supports is that at $f_0{=}0.9$ these five models lock in about
four times in five.

\textbf{One amendment ran against us and one specification error was ours.} P-T1-B1 as
originally frozen put the clusters at $f_{\mathrm{end}} \le 0.10$ and $\ge 0.90$; those are not
this paper's numbers, and applied to the fifteen runs above they split them $3/6/6$ rather than
$8/0/7$. The criterion tested a claim stronger than any we had made. We found this from the
\emph{old} data before reading a single $f_{\mathrm{end}}$ from the new, amended the prereg on
the record, and graded the replacement; the original is marked \textbf{mis-specified}, not
failed, under the same rule that protects \S\ref{sec:sched}'s apparatus failures. Second,
the prereg's variance clause compared this experiment's seed share to \S\ref{sec:power}'s
$68.8\%$, but those are different arms, grounded at $R{=}45$ there, ungrounded here. The
clean same-arm check was already on disk: $2.0\%$ at three seeds, $68.8\%$ at nine,
$\mathbf{67.6\%}$ at fifty-three. \emph{$68.8\%$ was not a small-sample artifact}; it is
stable on its own arm, and the $57.9\%$ measured here is a second condition in which the seed
again dominates the model ($6.0\%$). We record the specification error rather than deleting it.

\noindent Two scope questions, both registered and both answered against us, are set out in
Appendix~\ref{app:moved}: whether the effect is an artefact of the number of models, and
whether it survives the one apparatus choice we could not vary. Neither changes a verdict.

\section{Claim 2b given an actually out-of-sample test}
\label{app:splithalf}

\paragraph{Arm B is in sample, and calling it out of sample was our own inertia.}
Arm B estimates $\gamma$ on seeds $9$--$52$ and predicts seeds $9$--$52$, so its landing is
consistent both with $\gamma(\varphi)$ being a sufficient primitive and with a nine-cell curve
fitted to $44$ seeds reproducing those $44$ seeds, which would be worth nothing. Closing this
costs only a re-partition of data already on disk: \texttt{gamma\_splithalf.py} holds out half
the seeds, fits $\gamma$ on the other half, and compares against the same held-out rates, so the
only variable is where $\gamma$ came from (Table~\ref{tab:splithalf}). Two partitions are run, contiguous and parity, since
reporting one would let the partition do the arguing, and the reading was written before the
numbers: within $1.5\times$ of in-sample means $\gamma$ transports, beyond $3\times$ means arm B
was measuring its own fit.

\begin{table}[!htbp]
\centering\small
\caption{Claim 2b given an actually out-of-sample test. Each row holds out $22$ seeds and
predicts them twice, differing only in the fitting set; errors are in seed units of the
held-out half, where one unit is $0.045$. Not pre-registered and not graded. Arm A is carried
down from Table~\ref{tab:gammaout} for scale, in its own units.}
\label{tab:splithalf}
\begin{tabular}{@{}lccccc@{}}
\toprule
& \multicolumn{2}{c}{mean $|{\rm sim}-{\rm obs}|$} & & \multicolumn{2}{c}{worst $|z|$} \\
\cmidrule(r){2-3}\cmidrule(l){5-6}
split & in-sample & out-of-sample & ratio & in & out \\
\midrule
contiguous & $1.3$ & $1.5$ & $\mathbf{1.15\times}$ & $1.44$ & $2.41$ \\
parity     & $1.5$ & $1.8$ & $\mathbf{1.16\times}$ & $1.55$ & $2.71$ \\
\midrule
\multicolumn{6}{@{}l}{signed mean $z$ out of sample: $+0.32$ and $+0.35$} \\
\multicolumn{6}{@{}l}{\emph{arm A}, for scale: $12.1$ seed units of $44$, worst $z$ $8.50$,
  signed mean $z$ $-5.15$, under-predicts $5/5$} \\
\bottomrule
\end{tabular}
\end{table}

\noindent $\gamma$ transports. Fitting it on unseen topics costs $15$--$16\%$ of accuracy rather
than a factor of anything, and the diagnostic that condemned arm A is gone: arm A's errors were
one-sided at signed mean $z = -5.15$ with all five models under-predicted, while out of sample
the signed mean $z$ is $+0.32$ and $+0.35$. In absolute terms the out-of-sample error is
$0.068$--$0.081$ in capture rate against arm A's $0.28$. Claim 2b therefore stands on a test that
could have killed it. Two things are not flattering. The worst single cell degrades from
$1.5\sigma$ to $2.7\sigma$ and is \texttt{granite-4.1-8b} both times, though its four
out-of-sample $z$ values ($+2.41$, $-0.50$, $-0.84$, $+2.71$) swing in sign, so this is variance
$22$ seeds cannot suppress rather than a missing variable. The whole test
is post hoc: it was run because writing this paragraph exposed the gap, which is a worse reason
to have found it than a prereg.

\paragraph{The failure has a name, and naming it says what would have prevented it.} Arm A is a
transport failure in the sense of \citet{pearl2011transport,bareinboim2014external}: a quantity
estimated in one population is used to predict in another, and licensing that use requires a
selection diagram marking the respects in which the two populations differ. Here the two
populations are two topic samples, the respect in which they differ is the topic, and no diagram
was written, the published curve was carried across the difference silently, on the strength
of $\gamma$ being called a primitive. What arm B shows is that the machinery is right and the
transport was not licensed: re-estimating within the target population repairs the prediction
entirely. The general form of the lesson is that a constant is not a primitive because it is
written as a function; it is a primitive when the function's argument list is complete, and
$\varphi$ was not the whole argument list. A companion paper \citep{t2library2026} counts how many
such licences a growing programme owes and finds the count quadratic in the number of
populations.

\begin{warnbox}
\textbf{The one-sidedness is ours, and it runs opposite to the bias we declared in advance.}
A $\gamma$ cell with no observations behind it is filled from the nearest measured cell below
it, and \texttt{baseline\_urn.py} calls that ``the conservative direction for a
monotone-increasing $\gamma$.'' On three topics, $17$ of $45$ cells have fewer than thirty
reads behind them, against $3$ of $45$ on forty-four topics; \texttt{qwen3-8b}'s curve had
\emph{zero} observations at $j{=}0..4$ and eight at $j{=}5$, so most of it was the fill rule
rather than a measurement. That rule is conservative in the \emph{curve} and not in the
\emph{outcome}: this loop is positive feedback, so a $\gamma$ biased low is a capture rate
biased low, and the bias cannot cancel. Worse for us, the simplification we \emph{did} declare
in advance points the other way, the scripted reader asserts only the seeded false value,
so it ``should if anything capture more easily than the model it imitates.'' We named one bias
and shipped its opposite, and the reason we could not see it is that at three seeds the bar
was thirty times wider than the error.
\end{warnbox}

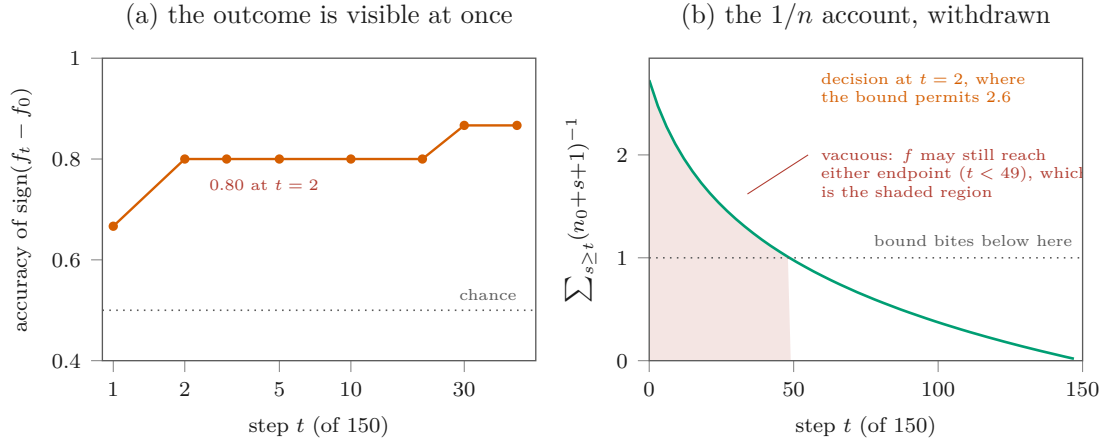
\begin{figure}[!tbp]
\centering
% GENERATED by t1_grounding/figures/make_t1_tikz.py -- do not edit by hand.
% Every coordinate is recomputed from t1_grounding/results/ at generation time.

\begin{tikzpicture}
\begin{axis}[name=pa, width=0.36\linewidth, height=4.0cm, axis on top, tick align=outside, tick pos=left, axis line style={figSoft, line width=0.5pt}, tick label style={font=\scriptsize, color=figInk}, label style={font=\scriptsize, color=figInk}, title style={font=\small, color=figInk, yshift=1pt}, grid=none, scale only axis, xmode=log,
  xmin=0.9, xmax=60, ymin=0.4, ymax=1.0, xtick={1,2,5,10,30}, xticklabels={1,2,5,10,30},
  % xmax 95 and not 70: the three series are direct-labelled at R=47.25 and at 70 the
  % frame's right spine fell inside the widest of them, striking the last letter of
  % "uniform". A vertical rule through one glyph on one line is under the strike test's
  % floor, so the word shipped with its ending crossed out.
  log ticks with fixed point,
  xlabel={step $t$ (of 150)}, ylabel={accuracy of $\mathrm{sign}(f_t-f_0)$},
  title={(a) the outcome is visible at once}]
\addplot[figB, mark=*, mark size=1.3pt, line width=0.9pt]
  coordinates {(1.0000,0.6667) (2.0000,0.8000) (3.0000,0.8000) (5.0000,0.8000) (10.0000,0.8000) (20.0000,0.8000) (30.0000,0.8667) (50.0000,0.8667)};
\addplot[figSoft, dotted, line width=0.6pt, forget plot] coordinates {(0.9,0.5) (60,0.5)};
\node[font=\tiny, figSoft, anchor=south east] at (axis cs:55,0.5) {chance};
\node[font=\tiny, figPred, anchor=north west] at (axis cs:2.3,0.780)
  {0.80 at $t=2$};
\end{axis}
\begin{axis}[name=pb, at={(pa.east)}, xshift=1.5cm, anchor=west,
  width=0.36\linewidth, height=4.0cm, axis on top, tick align=outside, tick pos=left, axis line style={figSoft, line width=0.5pt}, tick label style={font=\scriptsize, color=figInk}, label style={font=\scriptsize, color=figInk}, title style={font=\small, color=figInk, yshift=1pt}, grid=none, scale only axis,
  xmin=0, xmax=150, ymin=0, ymax=2.94,
  xlabel={step $t$ (of 150)},
  ylabel={$\sum_{s\geq t}(n_0{+}s{+}1)^{-1}$},
  title={(b) the $1/n$ account, withdrawn}]
\addplot[draw=none, fill=figPred, fill opacity=0.13, forget plot]
  coordinates {(0.0000,2.7265) (1.0000,2.6356) (2.0000,2.5523) (3.0000,2.4754) (4.0000,2.4039) (5.0000,2.3373) (6.0000,2.2748) (7.0000,2.2160) (8.0000,2.1604) (9.0000,2.1078) (10.0000,2.0578) (11.0000,2.0102) (12.0000,1.9647) (13.0000,1.9212) (14.0000,1.8796) (15.0000,1.8396) (16.0000,1.8011) (17.0000,1.7641) (18.0000,1.7283) (19.0000,1.6939) (20.0000,1.6605) (21.0000,1.6283) (22.0000,1.5970) (23.0000,1.5667) (24.0000,1.5373) (25.0000,1.5087) (26.0000,1.4810) (27.0000,1.4539) (28.0000,1.4276) (29.0000,1.4020) (30.0000,1.3770) (31.0000,1.3526) (32.0000,1.3288) (33.0000,1.3055) (34.0000,1.2828) (35.0000,1.2606) (36.0000,1.2388) (37.0000,1.2175) (38.0000,1.1967) (39.0000,1.1763) (40.0000,1.1563) (41.0000,1.1367) (42.0000,1.1175) (43.0000,1.0986) (44.0000,1.0801) (45.0000,1.0619) (46.0000,1.0440) (47.0000,1.0265) (48.0000,1.0093) (49.0000,0.0000) (0.0000,0.0000)} \closedcycle;
\addplot[figC, line width=1.0pt] coordinates {(0.000,2.727) (3.000,2.475) (6.000,2.275) (9.000,2.108) (12.000,1.965) (15.000,1.840) (18.000,1.728) (21.000,1.628) (24.000,1.537) (27.000,1.454) (30.000,1.377) (33.000,1.306) (36.000,1.239) (39.000,1.176) (42.000,1.117) (45.000,1.062) (48.000,1.009) (51.000,0.959) (54.000,0.912) (57.000,0.866) (60.000,0.823) (63.000,0.781) (66.000,0.741) (69.000,0.703) (72.000,0.665) (75.000,0.630) (78.000,0.595) (81.000,0.562) (84.000,0.530) (87.000,0.498) (90.000,0.468) (93.000,0.439) (96.000,0.410) (99.000,0.382) (102.000,0.355) (105.000,0.329) (108.000,0.303) (111.000,0.278) (114.000,0.254) (117.000,0.230) (120.000,0.207) (123.000,0.184) (126.000,0.162) (129.000,0.140) (132.000,0.119) (135.000,0.098) (138.000,0.078) (141.000,0.058) (144.000,0.038) (147.000,0.019)};
\addplot[figSoft, dotted, line width=0.6pt, forget plot] coordinates {(0,1) (150,1)};
\node[font=\tiny, figSoft, anchor=south east] at (axis cs:150,1.02) {bound bites below here};
\node[font=\tiny, text=figB, align=left, anchor=north west]
  at (axis cs:56,2.90)
  {decision at $t=2$, where\\the bound permits 2.6};
\draw[figPred, line width=0.4pt] (axis cs:55,2.00)
  -- (axis cs:34,1.62);
\node[font=\tiny, text=figPred, align=left, anchor=north west]
  at (axis cs:56,2.16)
  {vacuous: $f$ may still reach\\either endpoint ($t<49$), which\\is the shaded region};
\end{axis}
\end{tikzpicture}
\caption{\textbf{Decided early, and the $1/n$ account of why fails.} (a) The sign of
the first move predicts the final outcome from step $2$ and does not improve thereafter; this
is the measurement and it stands. (b) \emph{The account that fails.} Plotting $1/n$ invites
shading ``the window in which the outcome can still
move''. Prop.~\ref{prop:step} shows that window is not there: the quantity that actually
constrains $f$ is the \emph{cumulative} bound $\sum_{s\ge t}(n_0{+}s{+}1)^{-1}$, plotted here,
and it forbids nothing until it falls below $1$ at $t{=}49$. At the measured decision point
$t{=}2$ it permits $2.55$, more than the whole range of $f$. The shaded region is where the
$1/n$ explanation has no content.}
\label{fig:early}
\end{figure}

\label{sec:early}
\paragraph{Decided early, but not for the reason we gave.} $\mathrm{sign}(f_t - f_0)$
predicts capture in $12$ of $15$ runs at $t{=}2$, and no better at $t{=}30$. The natural
reading is \eqref{eq:dyn}'s restoring force carrying $1/n$, the decision
being taken while $n$ is smallest. \textbf{Prop.~\ref{prop:step} refutes it}: the
$1/n$ bound permits $f$ to reach either endpoint for the first $49$ of $150$ steps, so it cannot
explain a decision at $t{=}2$ (Figure~\ref{fig:early}). The early decision is measured and not derived, and the mechanism
that concentrates it is open.

\paragraph{We attacked it, from logs already on disk, and the attempt failed.} The diagnosis
frozen in \texttt{prereg/E\_TIMESCALE\_PREREG.md} was that we had been looking in the wrong state
variable: \eqref{eq:dyn} governs $f$, whose timescale is $O(n)$, but what decides a write is
$\varphi_t = j_t/k$, the false share of the $k$ retrieved records, and $\varphi$ lives on
$\{0,\dots,k\}$ with $k{=}8$ \emph{constant} in $t$. The hypothesis was that the outcome is set
by which absorbing boundary $\varphi$ reaches first, in $O(1)$ steps. On $220$ runs, zero API
cost:

\begin{itemize}[leftmargin=1.4em,itemsep=1pt,topsep=2pt]
\item \textbf{P-T1-D1 \textsc{supported}}: median first passage to $\varphi \in \{0,1\}$ is $2$
  steps, and the preregistration predicted this would pass cheaply at $f_0{=}0.9$, so it
  establishes little.
\item \textbf{P-T1-D2 \textsc{not supported}, at $219/219 = 1.000$}: \emph{every} run leaves the
  boundary after touching it. $\varphi$ crosses; it does not absorb, and the framing is dead
  rather than dented.
\item \textbf{P-T1-D3 \textsc{not supported}}: Spearman$(\tau_{\rm abs}, T_{\rm commit}) = +0.108$
  against a bar of $0.50$ on the $57$ runs that ever contradict their own outcome. We found a
  faster process and did not show it is the mechanism.
\end{itemize}

\noindent The preregistration said in advance what to do if D3 failed, and we do it:
\emph{\S\ref{sec:limits}(ii$'$) stands word for word}, with this recorded as a route tried and
not gone through. A mechanism that the data on disk cannot separate from its alternative is not
a mechanism we have measured.

\begin{warnbox}
\textbf{What did come out, and it makes the refuted account worse rather than better.}
P-T1-D4 asked whether the realised step is far below Prop.~\ref{prop:step}'s bound at the steps
where the decision is observed. It is: median $|f_{t+1}-f_t| = 0.0055$ against a median bound of
$0.0625$ over $t \le 10$, a factor of $11.4$. The reason is exact and post hoc (a write moves
$f$ by $(1-f)/(n+1)$ if it is false and by $-f/(n+1)$ if it is true, an identity with residual
$10^{-16}$ over $2420$ logged steps) so the bound $1/(n+1)$ is attained \emph{only} at
$f \in \{0,1\}$ and discards a factor of $\min(f, 1-f)$. \textbf{Prop.~\ref{prop:step}'s bound is
therefore not merely vacuous before $t{=}49$; it is loose by an order of magnitude exactly where
the decision is taken.} That is a statement about the argument we withdrew, not about the
dynamics, and it supplies no replacement.
\end{warnbox}

\section{Two scope questions, both registered}
\label{sec:scope}

Everything above is five open-weight models of $8$--$27$B at a single initial composition
$f_0 = 0.9$. Two objections follow, and neither is answerable by argument. \emph{Is this a
small-model artifact?}, to which our own data seemed to say yes: over the five models
$\mathrm{Spearman}(\text{parameters}, \text{capture}) = -0.671$ ($p=0.215$), the $8$B mean is
$0.841$ against $0.705$ for $14$--$27$B, and the largest model captures least. \emph{Is
$f_0 = 0.9$ doing the work?}, nine of ten starting records false is an adversarial initial
condition, and a store that starts mostly false ending mostly false is not obviously a finding.
Both were preregistered (\texttt{prereg/E\_FRONTIER\_PREREG.md},
\texttt{prereg/E\_F0\_PREREG.md}), both frozen before their runners existed, and in the first
case \textbf{we wrote our own prediction into the record and it was wrong}.

\begin{table}[!htbp]
\centering\small
\caption{\textbf{E-frontier: the phenomenon at frontier scale, against our own registered
prediction that it would vanish.} The \texttt{bimodal2} apparatus field for field (asserted
by the grader, not assumed) on seeds $9$--$28$, a prefix of the $44$ the open models used, so
runs pair seed for seed. \emph{The prereg states, before the run, that we expected pooled capture
below $0.5$}; the reasoning was the scale trend above. P-T1-F1's bar was $0.50$ and no answers
were lost on either model.}
\label{tab:frontier}
% GENERATED by t1_grounding/figures/make_t1_tikz.py from results/FRONTIER.json.
\begin{tabular}{@{}lrrrrl@{}}
\toprule
model & runs & calls & lost (\%) & capture & \\
\midrule
\texttt{anthropic/claude-sonnet-4.5} & $20$ & $3000$ & $0.00$ & $\mathbf{1.000}$
  & captured on \emph{every} seed \\
\texttt{openai/gpt-4o} & $20$ & $3000$ & $0.00$ & $0.700$
  &  \\
\midrule
pooled & $40$ & $6000$ & $0.00$ & $\mathbf{0.850}$ & P-T1-F1 bar $0.50$:
  \textsc{supported} \\
\bottomrule
\end{tabular}

\end{table}

\noindent \textbf{Two further frozen criteria were graded on this arm and the paper had not
printed either.} P-T1-F3 asks whether \emph{bimodality} survives at frontier scale, separately
from the rate: the share of runs landing in the interval \S\ref{sec:bimodal} calls empty must be
$\le 0.10$. It comes out $3/40 = 0.075$, \textsc{supported}, and the split is informative, $0/20$ for \texttt{claude-sonnet-4.5} at mean $f_{\mathrm{end}} = 0.923$ and $3/20$ for
\texttt{gpt-4o} at $0.671$. \emph{Two peaks survive whatever the rate does}, so $f_0$ and model
class select which peak is occupied rather than whether peaks exist, and that is the structural
half of the claim rather than its direction.

P-T1-F4 is the harder one and it \textbf{splits}. Re-measuring $\gamma$ on each frontier model's
own runs and simulating asks whether the primitive is sufficient for a \emph{model family it was
never fitted on}, a stronger extrapolation than \S\ref{sec:gammaout}'s cross-topic split-half.
The bar was $2.0$ seed units, one seed unit being $0.050$ at twenty seeds.
\texttt{claude-sonnet-4.5} is exact ($1.000$ observed, $1.000$ simulated, $0.0$ seed units);
\texttt{gpt-4o} is \emph{outside} it, $0.700$ observed against $0.901$ simulated, $4.0$ seed
units, and the sign is \textbf{over}-prediction where \S\ref{sec:gammaout}'s arm A
under-predicted every model (Table~\ref{tab:frontier}). The two failures of $\gamma$ we have measured run in opposite
directions, which is what a sufficient primitive with an under-powered estimator looks like and
is not what a missing variable looks like, but we have two models and cannot separate those
readings here.

\noindent \textbf{The phenomenon survives, and the model that survives it worst is the newest
one.} \texttt{claude-sonnet-4.5} is captured on $20$ of $20$ seeds, which is the same rate as
\texttt{qwen3-8b} at $8$B and higher than every $14$--$27$B model. The $-0.671$ rank
correlation was therefore not a scale law; it was five points with three of them at one size. We report
this as the correction of a reading of ours rather than as a prediction confirmed, because
our registered prediction was the opposite: we expected capture below $0.5$ and wrote down
why. Parameter counts for both frontier models are undisclosed, so they are excluded from any
correlation rather than replaced with a proxy, substituting price or context length would
report a different known quantity as the measurement.

\noindent \textbf{What forty runs at two models buys, computed rather than asserted.} This is the
paper's most-quoted sentence and its thinnest arm, so the accounting belongs beside it. The
resampling unit is still the seed, so the relevant factor is $1+(m{-}1)\rho$ with $m{=}2$, not the
$3.75$ measured at $m{=}5$: on this arm's own twenty shared seeds $\rho = 0.254$ for
$f_{\mathrm{end}}$, giving a design effect of $1.25$ and $31.9$ effective observations of $40$.
For the headline the arithmetic is simpler still, because \texttt{claude-sonnet-4.5} captured on
every seed: $20/20$ puts a one-sided $95\%$ lower bound of $0.861$ on its capture probability,
against a bar registered at $0.50$. The sentence survives its own power accounting. Three
weaknesses are not fixed by that. Parameter counts for both models are undisclosed, so neither
can enter a correlation. Served models at temperature $0$ are not reproducible
(Table~\ref{tab:replicate}'s two realisations are the only place we can price that), so a
replication may differ by more than seed noise. P-T1-F4 also splits, on two models, which is the
smallest sample on which a split can be reported at all. \emph{Owed and registered}: four or more
frontier models at $44$ seeds, and one end-to-end replication inside a deployed memory framework
reporting only whether the interval is empty. Neither is run, and the second is the single
measurement that would most change what this paper is evidence about.

\begin{table}[!htbp]
\centering\small
\caption{\textbf{E-f0: the empty interval is empty at every initial composition, and $f_0=0.9$
is not the worst one.} \texttt{gemma-3-27b} at the \texttt{bimodal2} apparatus, seeds $9$--$28$,
one model rather than five because \S\ref{sec:power}'s own finding is that this apparatus should
buy seeds and not models. ``Void'' is the interval $(0.306, 0.512)$ that
\S\ref{sec:bimodal} calls empty. P-T1-F0-1's bar was a share $\le 0.20$, frozen before the
runner existed. \textbf{The whole tabular is generated}, and it reads the band from the same
artifact Table~\ref{tab:bimodal2} does, so the two cannot disagree about where the empty interval
is, \emph{which they did while both were typed}: this table's $f_0{=}0.9$ row said $3/44$ in the
void, the count against the rounded cuts, where the cuts' own definition gives $2/44$.}
\label{tab:f0}
% GENERATED by t1_grounding/figures/make_t1_tikz.py from the f0 shards, the
% bimodal2 gemma arm, and BIMODAL2.json for the band.
\begin{tabular}{@{}rrrrrrl@{}}
\toprule
$f_0$ & $n$ & capture & mean & median & in void & \\
\midrule
$0.1$ & $20$ & $0.500$ & $0.459$ & $0.441$ & $\mathbf{0/20}$
  & \pred{and the mean lands \emph{inside} the void} \\
$0.3$ & $20$ & $0.700$ & $0.676$ & $0.956$ & $0/20$
  &  \\
$0.5$ & $20$ & $\mathbf{0.750}$ & $0.723$ & $0.969$ & $0/20$
  & highest capture of the four \\
$0.9$ & $44$ & $0.682$ & $0.724$ & $0.994$ & $2/44$
  & the arm \S\ref{sec:bimodal2} reports \\
\bottomrule
\end{tabular}

\end{table}

\noindent Three readings, and the middle one answers the objection outright.

(i) The void is empty at every $f_0$ we bought, $0/20$ three times over against a
frozen bar of $0.20$. Two peaks are therefore a property of the dynamics; $f_0$ selects \emph{which}
peak a run lands on, not whether peaks exist. This is the structural half of claim 1 and it does
not depend on any capture rate.

\textbf{(ii) Capture is not monotone in $f_0$, and $f_0=0.9$ is not the adversarial extreme.}
It reads $0.500 / 0.700 / 0.750 / 0.682$: the \emph{highest} capture in this sweep is at
$f_0=0.5$, where the store starts balanced. Hence ``you started from ninety percent falsehood''
does not describe what drives the outcome, a balanced store is worse than a mostly-false one,
which \eqref{eq:dyn} anticipates, since a store already at the append-only ceiling has little
room left to move while a balanced one sits where the per-step variance is largest. We did not
freeze a criterion on the shape and do not add one now.

\textbf{(iii) At $f_0=0.1$ the pooled mean falls inside a strictly empty interval.} The mean is
$0.459$ and \emph{no run} lies in $(0.306, 0.512)$; the runs split $10$--$10$ between escape and
capture. \S\ref{sec:related}'s objection to summarising this distribution by its mean was
measured on the $f_0=0.9$ arm at $8.2\%$ of runs described (P-T1-M1), and $f_0=0.1$ is the
sharper case: there the summary statistic denotes a region containing no observation at all.
That is the strongest form of the claim and we did not preregister it, so it is reported and not
graded.

\begin{warnbox}
\textbf{A loader bug found in E-f0's first output, and what it did and did not affect.}
\texttt{run\_sched\_gate.py} writes every run twice, one shard, then again into a merged file, so the grader's first glob counted each run twice and reported $40$ runs where $20$ seeds
were bought. \emph{Shares were unaffected}, because duplicating a set preserves its proportions,
so P-T1-F0-1's verdict is the same either way. What was wrong was the $n$, and a wrong $n$
misstates precision, which is worse than a wrong share because a reader cannot see it. The loader
now deduplicates on the full identifying configuration.
\end{warnbox}

\section{The sweep past the peak}
\label{sec:sweeps}

\paragraph{The sweep past the peak graded \textsc{untested}.} Three further budgets,
$R \in \{75, 120, 148\}$ of $150$ steps, all else identical, five models and three seeds each.
The design discriminates rather than confirms: \eqref{eq:dilution} puts the spread down $22\%$ by
$R{=}120$, while \S\ref{sec:bimodal}'s early-decision result keeps back-loading harmless until
its first grounding write enters the $t{\le}2$ window, at step $T{-}R$. The accounts agree at
$R{=}148$ and disagree at $R{=}120$, and the criteria were frozen on that disagreement
(\texttt{prereg/E\_TRANSIENT\_PREREG.md}).

\begin{table}[!htbp]
\centering\small
\caption{The sweep past the peak. \emph{Capture} is the frozen observable and it dies:
a grounding write is a truth record, so the ceiling itself carries the budget,
$f_{\max}(R) = (n_{\mathrm{false}} + Tw)/(n_0 + Tw + R) = 159/(160{+}R)$, and by $R{=}148$ a
run must reach $97\%$ of its own ceiling to be counted as captured. The last column divides
that ceiling out, per run: a step whose answer did not parse writes nothing and is not logged
($123$ of $168{,}660$ across the whole programme), so each run's ceiling uses its own write
count rather than the nominal $T$. Using the nominal value instead shifts $s_{\mathrm{norm}}$
by at most $0.0012$. $s$ is back minus front.}
\label{tab:transient}
\begin{tabular}{@{}rcrccccc@{}}
\toprule
$R$ & $R/T$ & back's first write & $f_{\max}(R)$ & \multicolumn{2}{c}{capture (frozen)}
  & \multicolumn{2}{c}{$\overline{f_{\mathrm{end}}}/f_{\max}$ (post hoc)} \\
\cmidrule(lr){5-6}\cmidrule(l){7-8}
 & & at step & & front / back & $s$ & front / back & $s_{\mathrm{norm}}$ \\
\midrule
5   & 0.03 & 145 & 0.964 & 0.20 / 0.40 & 0.20 & 0.275 / 0.468 & 0.193 \\
15  & 0.10 & 135 & 0.909 & 0.20 / 0.40 & 0.20 & 0.275 / 0.489 & 0.214 \\
45  & 0.30 & 105 & 0.776 & 0.07 / 0.47 & \textbf{0.40} & 0.211 / 0.507 & \textbf{0.295} \\
75  & 0.50 & 75  & 0.677 & 0.07 / 0.33 & 0.27 & 0.190 / 0.396 & 0.205 \\
120 & 0.80 & 30  & 0.568 & 0.00 / 0.00 & \emph{no info} & 0.213 / 0.331 & 0.117 \\
148 & 0.99 & 2   & 0.516 & 0.00 / 0.00 & \emph{no info} & 0.216 / 0.219 & \textbf{0.003} \\
\bottomrule
\end{tabular}
\end{table}

\noindent P-T1-T1 and P-T1-T2 are \textsc{untested} for apparatus failure. At $R{=}120$ and
$R{=}148$ all three arms give capture $0.00$, so by the floor clause registered with E-budget
those cells carry no information. Prop.~\ref{prop:ceiling} says why they are empty rather than
unlucky: at $R{=}148$ the append-only ceiling has fallen to $f_{\max} = 0.516$, only $0.016$
above the capture threshold, so a run must land in a window one sixtieth as wide as the state
space to be scored captured. Had those cells been used, P-T1-T1 would have read
$s(148){=}0.00 < s(75){-}0.15$ and been marked supported on the strength of no run anywhere
reaching the threshold. The observable ran out, not the mechanism, and the defect is ours:
$f_{\max}(R)$ follows from \eqref{eq:ceiling} and was derivable before the sweep was launched.

Read on a scale-free observable the transient appears, and the discriminator still fails to
discriminate. Dividing out $f_{\max}(R)$ gives the last column of Table~\ref{tab:transient}: a
rise to $0.295$ at $R/T{=}0.30$, then $0.205$, $0.117$, and $0.003$ at $R{=}148$, where the three
arms sit within $0.009$ of one another. Changing the observable after the frozen one floors is
how a failed prediction is normally rescued, so we grade nothing on it, and the frozen thresholds
applied to it land inside the undecided band, $-0.088$ against $\le -0.10$ for dilution and
$\ge -0.05$ for early decision. Neither account wins, and by \S\ref{sec:power} neither available
interval is usable. The peak's location is claimed in neither direction anywhere in this paper.

\paragraph{The scale-free retest, on six fresh seeds, splits.} Two scale-free criteria were
frozen and six unused seeds bought (\texttt{prereg/E\_TRANSIENT\_PREREG.md} amendment~1). P-T1-S1
uses $f_{\mathrm{end}}$ over the run's own ceiling; P-T1-S2 uses capture at
$\tfrac12 f_{\max}(R)$, and exists so that the conclusion cannot depend on which scale-free
observable we happened to pick. Both required
$\mathrm{spread}(120) \le \mathrm{spread}(45) - 0.15$.

\begin{table}[!htbp]
\centering\small
\caption{E-transient2, seeds $3$--$8$ only; seeds $0$--$2$ generated the hypothesis and are
excluded. $30$ runs per arm. Standard errors are clustered over the six seeds, these arms
share their seed grid by construction, so the contrast is within-seed and clustering
\emph{reduces} the interval rather than widening it (\S\ref{sec:power}).}
\label{tab:t2split}
\begin{tabular}{@{}lccccc@{}}
\toprule
observable & spread$(45)$ & spread$(120)$ & difference & bar & verdict \\
\midrule
$f_{\mathrm{end}} / f_{\max}(R)$ & $0.153$ & $0.093$ & $-0.060$ & $\le -0.15$
  & \textbf{failed} \\
capture at $\tfrac12 f_{\max}(R)$ & $0.167$ & $0.000$ & $-0.167$ & $\le -0.15$
  & \textbf{supported} \\
\bottomrule
\end{tabular}
\end{table}

\noindent P-T1-S4 is a split, and the prereg requires reporting it rather than choosing. The
collapse's existence depends on which of two equally defensible readings is used. Two disclosures
make it worse: the supported side clears its bar by $0.017$, half of one seed unit, and the
$R{=}120$ cell carrying that support has its arms at $0.700/0.733/0.700$, a span of exactly one
seed unit, which by the standard registered for the capture-rate observable is no information at
all (Table~\ref{tab:t2split}). We do not use this to overturn the verdict, since amendment~1 carried no floor clause and
adding one after the data is moving the goalposts.

\paragraph{We then bought the resolution we had declared unbuyable, and it did not resolve.}
P-T1-S3 stated in advance that separating the peak needs ``about $220$ runs per cell, about $44$
seeds, and we do not buy it.'' That records a budget rather than a limit of nature. We bought
it: $R \in \{45,75,120\}$, three schedules, five models, seeds $9$--$52$, $1980$ runs at $150$
steps, criteria frozen in \texttt{prereg/E\_TRANSIENT3\_PREREG.md} with the grader committed
before the data existed. Seeds $0$--$8$ generated the hypothesis and are excluded, dropping $315$
runs already on disk. Every cell is complete and balanced.

\begin{table}[!htbp]
\centering\small
\caption{E-transient3, seeds $9$--$52$, $220$ runs per arm. $s_{\mathrm{norm}}$ is
$\overline{f_{\mathrm{end}}}(\text{back}) - \overline{f_{\mathrm{end}}}(\text{front})$ with each
run divided by its own ceiling. Standard errors are clustered over the $44$ seeds. The
ordering front $<$ uniform $<$ back holds at all three budgets; it is the \emph{decline} that
was under test. The whole tabular is generated from \texttt{results/TRANSIENT3.json}, an
artifact that did not exist until this table did, because the grader printed these means on every
run and persisted none of them.}
\label{tab:transient3}
% GENERATED by t1_grounding/figures/make_t1_tikz.py from results/TRANSIENT3.json.
\begin{tabular}{@{}rcccc@{}}
\toprule
$R$ & front & uniform & back & $s_{\mathrm{norm}}$ \\
\midrule
$45$ & $0.587$ & $0.666$ & $0.773$ & $0.186$ \\
$75$ & $0.578$ & $0.631$ & $0.743$ & $0.165$ \\
$120$ & $0.564$ & $0.591$ & $0.677$ & $0.113$ \\
\bottomrule
\end{tabular}

\end{table}

\noindent At $44$ seeds the discriminator is undecided rather than resolved. The contrast
$s_{\mathrm{norm}}(120) - s_{\mathrm{norm}}(75)$ is $-0.052$ at $1.14\sigma$, against $-0.088$
at three seeds, which was outside the band; the shape criterion reads $-0.074$ against a bar of
$-0.15$ at $1.62\sigma$ and fails on magnitude and on significance alike. What survives is the
ordering, front $<$ uniform $<$ back at every budget, which is what the schedule axis rests on
and never needed the transient to assert (Table~\ref{tab:transient3}). The bill went up rather than down when we paid the
first instalment: the effect fell from $0.115$ at three seeds to $0.021$ at forty-four while the
standard error shrank $21\%$ more slowly than $1/\sqrt{n}$, so $3\sigma$ now needs about $1895$
seeds, or $9475$ runs per cell. Claim (iii)'s statement about the peak's location is marked not
resolvable at ${\le}1980$ runs and left there.

\begin{claimbox}
\textbf{The boundary this draws.} In generational retraining, real data is a fixed corpus
present in every round and re-usable without limit, so ``when to ground'' is not a question
that arises; the established results there concern accumulation and proportion
\citep{gerstgrasser2024accum,dey2024pathway}. Their setting is the
\emph{free-regrounding} limit of ours, $R \to T$, where \eqref{eq:dilution} makes the three
schedules identical and only the fraction is left to vary. By item (iv) that limit is a
single point rather than a neighbourhood: the schedule spread is zero at $\alpha{=}1$ and
diverges like $\ln T$ at every $\alpha < 1$, so no amount of increasing their re-use budget
turns their setting into ours and no amount of decreasing ours turns it into theirs. Ours is the budgeted case, where
the same fraction delivered at a different time is worth an order of magnitude less. The
containment runs along the budget, not along the store size, accumulation itself is
theirs \citep{gerstgrasser2024accum}.
\end{claimbox}

\section{An accidental replicate}
\label{sec:replicate}

This sweep was run twice. The second run was launched by a command intended only to count
configurations, and it overwrote the first run's raw data; we record the process failure rather
than omit it. What it bought is a measurement this design otherwise lacks, since every other
figure here prices seed variation alone.

\begin{table}[!htbp]
\centering\small
\caption{The same protocol, same seeds, same budget, run twice. Only the second run's raw
data survives; the first survives as its logged pooled figures. Note that the agent runs at
temperature $0$, so the two realisations differ only because the served model is not
reproducible at that setting.}
\label{tab:replicate}
\begin{tabular}{@{}lrccl@{}}
\toprule
pooled capture rate & front & uniform & back & frozen criteria \\
\midrule
first realisation  & 0.20 & 0.27 & 0.33 & P-S1 $5/5$, P-S2 \checkmark, P-S3 $\times$ \\
second realisation (on disk) & 0.20 & 0.33 & 0.40 & P-S1 $5/5$, P-S2 \checkmark, P-S3 $\times$ \\
\midrule
movement & 0 & $+1/15$ & $+1/15$ & \textbf{identical} \\
\bottomrule
\end{tabular}
\end{table}

\begin{claimbox}
The point estimates move by one run in fifteen on two of three arms, while the ordering, its
$5/5$ model-level agreement and all three frozen verdicts are identical. The schedule result is
robust in the form it is claimed, an ordering, and the pooled rates should be read as $\pm$ one
run rather than as two-decimal quantities. Every other number here carries the same unpriced
replication variance, and this table is the only place we can bound it.
\end{claimbox}

\paragraph{A write-side gate erases the differences between models.} Admitting the agent's
answer only when it agrees with the majority of retrieved records is the most conservative
deterministic form of consistency-gated admission, and weaker than any deployed one: \citet{amac2026}
admit on five factors of which agreement is one, and \citet{memguard2026} type memories by role
rather than filtering on agreement at all. \textbf{A harmful regime in the weakest gate is not ``conservative evidence about stronger
versions''}: the inference needs harm to be monotone in gate strength, we have not tested a second
gate, and \citet{memguard2026}'s mechanism is a reason to expect non-monotonicity, since typing
the store changes the reference class our own diagnosis blames. What the arm supports is a
statement about \emph{this} gate. The registered prediction
was a sign flip, helpful at low contamination and harmful at high, and it holds in five of five
models at each end. The numbers show more than that. On seven fresh seeds the ungated baselines
span $0.729$, $0.745$ and $0.932$; gated, all three sit at $0.9936$, $0.9937$ and $0.9935$,
separating only in the fourth decimal.

\begin{claimbox}
The gate's reference class is the store the agent itself wrote. It therefore fixes the
store's composition \emph{upstream} of the reading behaviour, and a model whose cleaning
lived in that behaviour has nothing left to clean with. The observable consequence is not
that the gate is harmful on average but that it removes the differences between models: a
$0.2$ spread in unaided outcome becomes a $0.001$ spread.
\end{claimbox}

\noindent That majority aggregation can lock in error is established for one-shot aggregation
over samples, where the mechanism is error correlation and a better model is therefore safer
\citep{consensustrap2026,minoritysentinel2026}. Two things differ here: the filter sits on the
write side rather than the read side, and it acts over a tenure rather than at one decision,
which is what decouples composition from reading. A third proposed difference, that the
model-quality dependence reverses sign between the two timescales, was killed;
\S\ref{sec:killed} reports its death.

\paragraph{A cost-matched state slot: a tie, reported as one.} The parent framework's
central object is a pair of channels, one compressive and composable, and it flags the
systems-level version of its claim as untested \citep{cch2026}. The smallest such channel
this apparatus admits is one line carrying the agent's own previous answer, with one
retrieved record removed to pay for it, eight lines per decision in both arms, so the
comparison is cost-matched at the line level. We registered two competing readings: the slot
supplies invalidation throughput and helps at both contamination levels, or the slot is a
self-reinforcing write channel and helps only at low contamination. They agree at $f_0{=}0.1$
and disagree at $f_0{=}0.9$, which is the discriminating cell.

At $f_0{=}0.1$ the slot helps or is neutral in four of five models (P-A1). At $f_0{=}0.9$ it
gives no verdict: zero models better, one worse, four tied (P-A2). The per-model spread is
what the tie hides, one model moves $+0.84$ in capture and another $-0.30$, so this is
a genuine failure to separate rather than a small effect. What follows is a boundary rather
than a result: a \emph{self-fed} slot is not the state channel of the parent framework,
whose channel is fed by the stream rather than by the system's own output, and testing that
distinction needs a stream-fed slot this apparatus does not yet have.

\paragraph{A source marker is not a directional lever.} Our registered prediction, that
marking each record's origin lowers contamination at high $f_0$, fails at $2/5$ (P-P2);
its mechanism claim, that the marker lowers $\gamma$, fails at $3/5$ (P-P3). It helps
sharply in some models ($\gamma$ from $0.799$ to $0.220$) and hurts in others ($0.309$ to
$0.580$). What does replicate is a null we did not predict: with the label scrambled, identical tokens, identical wording, no information, the outcome sits at $0.856$
against a two-end midpoint of $0.664$, nearer the lying end than the middle, in five of
five models.

\begin{claimbox}
An uninformative provenance marker is not a neutral cost. It occupies the same budget as
content and demands a per-record judgement it cannot support, which is the behavioural
reading of the capacity coupling: source bits and content bits are drawn from one budget.
\end{claimbox}

% =====================================================================

\section{Unresolvability from the boundary structure}
\label{sec:power}

The three deaths below share one cause, and \S\ref{sec:boundaries} is what makes it
calculable rather than a lesson learned. If the outcome is a choice between two boundaries, escape to $f{=}0$ or saturation at $(n{-}1)/n$ (then $f_{\mathrm{end}}$ is nearly a
two-valued random variable, and a per-model mean over $s$ seeds is an estimate of the
\emph{choice probability} $p$) which is Prop.~\ref{prop:se} (Table~\ref{tab:power}).

\noindent Both halves matter and the second is the one that caught us. \eqref{eq:se} is a
\emph{per-model} statement; the design-effect factor $1 + (m{-}1)\rho$ applies to everything
pooled, which is most of this paper. The decomposition below measures $\rho$ and gets a
design effect of $3.75$ at nine seeds, so $45$ runs carry the information of $12$; measured again
at $44$ seeds on two further arms it is $3.08$ and $3.56$--$3.66$, and pair by pair it runs
$1.73$--$4.95$ (\S\ref{sec:knife}).

\paragraph{The resampling unit, and why every error bar in this paper was the wrong size.}
\eqref{eq:se} is a \emph{per-model} statement and it survives: with one run per seed the
marginal variance of a Bernoulli outcome is $\bar p(1-\bar p)$ whatever the topic-to-topic
spread in $p$. What does not survive is any quantity pooled across models, and that is most
of them. The seed is not a noise knob (it seeds the world and the corruption draw, so it
selects the topic and the false value) and on the $R{=}45$ front arm the variance
decomposes as
\begin{center}\small
\begin{tabular}{@{}lccr@{}}
\toprule
 & between-seed & between-model & residual \\
\midrule
at $3$ seeds & \textbf{2.0\%} & 94.4\% & 3.6\% \\
at $9$ seeds & \textbf{68.8\%} & 14.2\% & 17.0\% \\
\bottomrule
\end{tabular}
\end{center}

\noindent \emph{The dominant variance component is invisible in the sample that was used to
decide it was small.} At three seeds the topic contributes $2\%$ and the model looks like the
whole story; six more seeds make it $69\%$. Four of those nine seeds are topics on which
\emph{all five} models end contaminated (front-arm $f_{\mathrm{end}}$ $0.64$--$0.71$) and five
are topics on which all five end clean ($0.14$--$0.20$), the two groups separated by
$5.0\sigma$.

Two consequences, and they point in opposite directions, so we state them per quantity rather
than issuing one correction. With $m{=}5$ models per seed and an intra-seed correlation of
$0.69$, the design effect is $1 + (m-1)\rho = 3.75$: forty-five runs carry the information of
twelve, so \textbf{each seed is worth $1.33$ runs no matter how many models we run on it.}
For a pooled \emph{level} the seed effect adds, and treating runs as independent understates
the standard error by $1.86\times$ ($0.056$ against $0.105$). For a within-seed \emph{contrast}, which is what every schedule result is, since the arms share the grid by construction, the seed effect cancels, and treating runs as independent \emph{overstates} it by $1.11\times$
($0.037$ against $0.034$). The same mistake made our levels look too precise and our
contrasts look too noisy.

\textbf{No verdict in this paper changes.} Every frozen criterion was set on a point estimate
against a bar, not on a $\sigma$: the schedule ordering is $5/5$ models monotone per seed, the
budget falsifier failed on $0.20$ against $0.40$, the variance bound failed on $1.07\times$
against a required $2.83\times$. What changes is the design advice, and it is the opposite of
what we did: this apparatus should have bought seeds and not models. The five-model
grid was chosen to price cross-model generality and it does that, but it contributes almost
nothing to the precision of anything pooled. All bootstrapped intervals below use a cluster
bootstrap over seeds and report the cluster count, because with three clusters a clustered
interval is itself untrustworthy, measured here, clustered and unclustered standard errors
differ by up to $2.6\times$ at three seeds and agree to within $2\%$ at nine.

\noindent This also settles the failure recorded in \S\ref{sec:limitbimodal}. A diffusion
would accumulate variance as $\sum 1/n^2 = O(1/n_0)$ and its spread would fall as
$1/\sqrt{n_0}$; a boundary choice has variance $p(1-p)$ times the ceiling squared, in which
$n_0$ does not appear at all. Measured, the within-model spread is flat in $n_0$ to
$1.07\times$ over an eightfold range. The variance bound failed and the ceiling result
holds for the same reason, and the two findings are one: this process does not diffuse to a
stopping point, it picks a side.

\begin{table}[!htbp]
\centering\small
\caption{What it takes to resolve a cross-model difference in the choice probability, at
$80\%$ power and $95\%$ confidence, against what we ran.}
\label{tab:power}
\begin{tabular}{@{}lccc@{}}
\toprule
difference in $p$ to be resolved & $0.4$ & $0.3$ & $0.2$ \\
\midrule
seeds required per model & 24 & 44 & 98 \\
seeds we ran & \multicolumn{3}{c}{3, then 7 in the registered re-test} \\
\bottomrule
\end{tabular}
\end{table}

\begin{claimbox}
\textbf{The axis that died was unresolvable by construction, and \eqref{eq:se} says so
without any data.} A cross-model correlation whose independent variable is a $3$-seed
estimate of a boundary-choice probability is a correlation against an estimate with a
standard error of $0.29$ on a $[0,1]$ scale. That is why the rank correlation of $+1.00$
evaporated: the model we called ``cleanest'' had $\hat p$ from three seeds, and seven seeds
moved its mean from $0.150$ to $0.729$. Nothing was measured; the ordering was a draw.
\end{claimbox}

\noindent The transferable form of this is a design rule for anyone reading a per-model
number off a self-writing loop: first ask whether the outcome is interior or a
boundary choice. If it is a boundary choice, per-model means need tens of seeds before they
can carry a comparison, and no amount of care in the downstream arithmetic substitutes, which is the same lesson this project learned once before, when fixed points computed
exactly from $12$-sample cells moved to zero at $188$ samples. The failure mode is not
sloppy arithmetic on noisy inputs; it is \emph{exact} arithmetic on noisy inputs, which is
harder to see.

% GENERATED by t1_grounding/make_appendix.py -- do not edit by hand.
% Every number below is read from the repository at build time.
% Emitted as \section, NOT wrapped in a second \appendix: see main()'s comment.

\section{Every criterion this paper grades, with its freeze date}
\label{app:prereg}

Each row is a preregistration file: the criteria it freezes, the commit that
introduced it, and that commit's UTC timestamp. The dates are read from \texttt{git log
-{}-reverse} at build time rather than transcribed. Amendments are counted separately and
never overwrite the original, so a criterion that was mis-specified stays in the record
alongside its replacement; \S\ref{sec:bimodal2} discusses the one case where that
mattered.

\begin{table}[h]
\centering\footnotesize
\caption{The preregistration record, generated from git by
\texttt{make\_appendix.py}.}
\label{tab:appprereg}
\begin{tabular}{@{}llp{0.40\linewidth}cr@{}}
\toprule
experiment & frozen at & criteria & amend. & words \\
\midrule
\texttt{E\_A} & \texttt{a46631d9}, 2026-08-06 07:15 & P-A1, P-A2, P-A3 & 0 & 161 \\
\texttt{E\_BIMODAL2} & \texttt{2072fee0}, 2026-08-07 01:51 & P-T1-B1, P-T1-B2, P-T1-B3, P-T1-B4, P-T1-B5, P-T1-B1', P-T1-B1'' & 3 & 513 \\
\texttt{E\_BUDGET} & \texttt{a46631d9}, 2026-08-06 07:15 & P-B1, P-B2, P-B3 & 0 & 314 \\
\texttt{E\_F0} & \texttt{45d6f80b}, 2026-08-07 16:28 & P-T1-F0 & 1 & 330 \\
\texttt{E\_FRONTIER} & \texttt{a4095483}, 2026-08-07 16:19 & P-T1-F1, P-T1-F2, P-T1-F3, P-T1-F4, P-T1-F5 & 0 & 315 \\
\texttt{E\_GATE} & \texttt{a46631d9}, 2026-08-06 07:15 & P-G1, P-G2 & 0 & 79 \\
\texttt{E\_LABEL} & \texttt{a46631d9}, 2026-08-06 07:15 & P-L1, P-L2, P-L3 & 0 & 117 \\
\texttt{E\_MEAN} & \texttt{5ddb9f4e}, 2026-08-07 06:35 & P-T1-M1, P-T1-M2, P-T1-M3, P-T1-M4 & 0 & 236 \\
\texttt{E\_MULTIVALUE} & \texttt{60dddef7}, 2026-08-18 07:42 & P-T1-V1, P-T1-V2, P-T1-V3, P-T1-V4, P-T1-V5 & 2 & 812 \\
\texttt{E\_PROV} & \texttt{a46631d9}, 2026-08-06 07:15 & P-P1, P-P2, P-P3 & 0 & 224 \\
\texttt{E\_RATCHET} & \texttt{a46631d9}, 2026-08-06 07:15 & P-R1, P-R2, P-R3 & 0 & 140 \\
\texttt{E\_REAL} & \texttt{0ad31a69}, 2026-08-09 03:01 & P-T1-N0, P-T1-N1, P-T1-N2, P-T1-N3, P-T1-N4, P-T1-N5, P-T1-N6 & 1 & 3140 \\
\texttt{E\_RETRIEVER} & \texttt{87f0c5d7}, 2026-08-08 14:28 & P-T1-R0, P-T1-R1, P-T1-R2, P-T1-R3, P-T1-R4, P-T1-R5 & 1 & 2771 \\
\texttt{E\_SCHED} & \texttt{a46631d9}, 2026-08-06 07:15 & P-S1, P-S2, P-S3 & 0 & 116 \\
\texttt{E\_TEMP} & \texttt{d76ccf6a}, 2026-08-09 01:35 & P-T1-X0, P-T1-X1, P-T1-X2, P-T1-X3, P-T1-X4 & 0 & 1415 \\
\texttt{E\_TIMESCALE} & \texttt{b8c38446}, 2026-08-08 11:08 & P-T1-D1, P-T1-D2, P-T1-D3, P-T1-D4 & 0 & 322 \\
\texttt{E\_TRANSIENT3} & \texttt{fd89cc33}, 2026-08-06 07:56 & P-T1-U1, P-T1-U2, P-T1-U3, P-T1-U4, P-T1-U5 & 0 & 312 \\
\texttt{E\_TRANSIENT} & \texttt{cbf4a12c}, 2026-08-06 05:19 & P-T1-T1, P-T1-T2, P-T1-T3, P-T1-T4, P-T1-S1, P-T1-S2, P-T1-S3, P-T1-S4 & 1 & 550 \\
\texttt{E\_VERIFY} & \texttt{b5aab439}, 2026-08-08 11:20 & --- & 0 & 283 \\
\bottomrule
\end{tabular}
\end{table}

\section{The invariants, checked on every run in the repository}
\label{app:invariants}

\texttt{verify\_invariants.py} asserts each of the following over every run
on disk, and refuses the build if one fails. They are not statistics; they are properties the
apparatus must have if it implemented the protocol at all, and several of them exist because
one build did not. The count in the right-hand column is the number of individual
assertions evaluated, 7,929,204 in total, all of which
hold.

\begin{table}[h]
\centering\footnotesize
\caption{Machine-checked invariants, run at build time.}
\label{tab:appinv}
\begin{tabular}{@{}lp{0.60\linewidth}r@{}}
\toprule
& invariant & assertions \\
\midrule
\ding{51} & answer labels partition & 1,114,218 \\
\ding{51} & api\_calls >= logged steps & 7,544 \\
\ding{51} & bimodal2 config == converge's n\_false0=9 arm & 220 \\
\ding{51} & bimodal2 uses only fresh seeds (>=9) & 220 \\
\ding{51} & budgets matched within an R & 6 \\
\ding{51} & converge n\_false0=9 arm is internally uniform & 15 \\
\ding{51} & converge's arm and bimodal2 share no seed & 1 \\
\ding{51} & dropped steps are absent from the log, not zero-filled & 7,544 \\
\ding{51} & f0 == n\_false0/n0 & 7,544 \\
\ding{51} & f\_any*n is an integer & 1,114,218 \\
\ding{51} & f\_end == last logged f\_any & 7,544 \\
\ding{51} & f\_seed <= f\_any & 1,114,218 \\
\ding{51} & false count <= planted + own writes (eq:ceiling) & 1,114,218 \\
\ding{51} & gamma denominators sum to logged steps & 7,544 \\
\ding{51} & gamma numerator <= denominator & 46,949 \\
\ding{51} & gamma\_crit == 1/k + r/w & 7,544 \\
\ding{51} & gamma\_hat == the j=1 cell & 4,790 \\
\ding{51} & gamma\_n == the j=1 denominator & 4,790 \\
\ding{51} & ground\_steps distinct and in range & 7,544 \\
\ding{51} & logged steps <= configured steps & 7,544 \\
\ding{51} & n\_false\_retrieved <= k & 1,114,218 \\
\ding{51} & observed\_growth == (f\_end > f0) & 7,544 \\
\ding{51} & one configuration, one run within a collection & 7,544 \\
\ding{51} & predicted\_growth == (gamma\_hat > gamma\_crit) & 4,790 \\
\ding{51} & shard id uniquely determines cfg & 1 \\
\ding{51} & store never shrinks & 1,106,674 \\
\ding{51} & store size accountable per step & 1,114,218 \\
\midrule
& \textbf{total} & \textbf{7,929,204} \\
\bottomrule
\end{tabular}
\end{table}

\section{E-bimodal2 per model}
\label{app:bimodal2}

P-T1-B5 puts every verdict on the five models pooled and forbids replacing a
pooled verdict with a per-model one. The per-model figures are printed here because forbidding
their use as evidence is not a reason to withhold them: a reader who suspects one model carries
the result should be able to check. All 220 runs, seeds $9$--$52$.

\begin{table}[h]
\centering\footnotesize
\caption{E-bimodal2 by model. \emph{Middle} counts runs in the interval
$(0.306, 0.512)$ that \S\ref{sec:bimodal} calls empty; the pooled bar frozen in the prereg is
$10\%$ against a uniform baseline of $20.6\%$.}
\label{tab:appbimodal2}
\begin{tabular}{@{}lrrrrrr@{}}
\toprule
model & $n$ & $\le 0.306$ & middle & $\ge 0.512$ & middle share
& mean $f_{\mathrm{end}}$ \\
\midrule
\texttt{qwen3-8b} & 44 & 0 & 0 & 44 & 0.000 & 0.939 \\
\texttt{ministral-8b} & 44 & 6 & 2 & 36 & 0.045 & 0.808 \\
\texttt{phi-4} & 44 & 10 & 2 & 32 & 0.045 & 0.757 \\
\texttt{granite-4.1-8b} & 44 & 10 & 3 & 31 & 0.068 & 0.735 \\
\texttt{gemma-3-27b} & 44 & 11 & 3 & 30 & 0.068 & 0.724 \\
\midrule
\textbf{pooled} & 220 & 37 & 10 & 173 & \textbf{0.045} & 0.792 \\
\bottomrule
\end{tabular}
\end{table}

\section{The two $\gamma$ curves, with the denominators behind each cell}
\label{app:gamma}

\S\ref{sec:gammaout} turns on the difference between $\gamma$ measured on three
topics and $\gamma$ measured on forty-four, and the mechanism is visible only with the
\emph{support} printed, the number of reads behind each cell. A cell with no observations is
filled from the nearest measured cell below it, which \texttt{baseline\_urn.py} calls the
conservative direction for a monotone-increasing $\gamma$; it is conservative in the curve and
not in the outcome, because the loop is positive feedback. Rows marked A are the published
curve (seeds $0$--$2$); rows marked B re-measure on seeds $9$--$52$. \textbf{The bold cells are
those with fewer than thirty reads behind them.}

\begin{table}[h]
\centering\scriptsize
\caption{$\gamma(j)$ per model and the support of each cell, $j$ false records
retrieved of $k{=}8$. Generated by \texttt{make\_appendix.py} from the same functions
\texttt{gamma\_outsample.py} uses, so the two cannot disagree.}
\label{tab:appgamma}
\begin{tabular}{@{}llrrrrrrrrr@{}}
\toprule
model & arm & $j{=}0$ & $j{=}1$ & $j{=}2$ & $j{=}3$ & $j{=}4$ & $j{=}5$ & $j{=}6$ & $j{=}7$ & $j{=}8$ \\
\midrule
\texttt{gemma-3-27b-it} & A & 0.00 & 0.00 & 0.03 & 0.02 & 0.03 & 0.00 & \textbf{0.00} & \textbf{0.76} & \textbf{1.00} \\
& \emph{n} & \emph{92} & \emph{102} & \emph{76} & \emph{44} & \emph{32} & \emph{31} & \emph{25} & \emph{29} & \emph{19} \\
 & B & 0.00 & 0.03 & 0.05 & 0.03 & 0.03 & 0.04 & 0.10 & 0.93 & 1.00 \\
& \emph{n} & \emph{351} & \emph{469} & \emph{392} & \emph{244} & \emph{183} & \emph{142} & \emph{126} & \emph{827} & \emph{3866} \\
\addlinespace[2pt]
\texttt{granite-4.1-8b} & A & 0.00 & 0.00 & 0.11 & 0.14 & 0.18 & \textbf{0.30} & \textbf{0.47} & 0.84 & 1.00 \\
& \emph{n} & \emph{42} & \emph{56} & \emph{65} & \emph{58} & \emph{33} & \emph{20} & \emph{17} & \emph{49} & \emph{110} \\
 & B & 0.00 & 0.01 & 0.13 & 0.23 & 0.27 & 0.40 & 0.59 & 0.86 & 1.00 \\
& \emph{n} & \emph{109} & \emph{270} & \emph{411} & \emph{422} & \emph{322} & \emph{363} & \emph{390} & \emph{1099} & \emph{3214} \\
\addlinespace[2pt]
\texttt{ministral-8b-2512} & A & \textbf{0.00} & 0.00 & 0.00 & 0.10 & 0.03 & 0.35 & 0.57 & 0.82 & 1.00 \\
& \emph{n} & \emph{23} & \emph{43} & \emph{37} & \emph{31} & \emph{32} & \emph{49} & \emph{30} & \emph{68} & \emph{137} \\
 & B & 0.00 & 0.03 & 0.04 & 0.10 & 0.21 & 0.33 & 0.71 & 0.86 & 1.00 \\
& \emph{n} & \emph{138} & \emph{196} & \emph{178} & \emph{184} & \emph{189} & \emph{243} & \emph{458} & \emph{1427} & \emph{3587} \\
\addlinespace[2pt]
\texttt{phi-4} & A & 0.00 & 0.00 & 0.01 & 0.11 & \textbf{0.24} & \textbf{0.22} & \textbf{0.33} & \textbf{0.84} & 1.00 \\
& \emph{n} & \emph{52} & \emph{75} & \emph{72} & \emph{47} & \emph{25} & \emph{18} & \emph{6} & \emph{25} & \emph{130} \\
 & B & 0.00 & 0.05 & 0.06 & 0.11 & 0.21 & 0.31 & 0.52 & 0.91 & 1.00 \\
& \emph{n} & \emph{224} & \emph{371} & \emph{376} & \emph{276} & \emph{221} & \emph{203} & \emph{202} & \emph{983} & \emph{3744} \\
\addlinespace[2pt]
\texttt{qwen3-8b} & A & \textbf{0.00} & \textbf{0.00} & \textbf{0.00} & \textbf{0.00} & \textbf{0.00} & \textbf{1.00} & \textbf{0.84} & 0.94 & 1.00 \\
& \emph{n} & \emph{0} & \emph{0} & \emph{0} & \emph{0} & \emph{0} & \emph{8} & \emph{25} & \emph{122} & \emph{295} \\
 & B & \textbf{0.00} & \textbf{0.00} & \textbf{0.31} & 0.42 & 0.54 & 0.68 & 0.88 & 0.90 & 1.00 \\
& \emph{n} & \emph{0} & \emph{2} & \emph{16} & \emph{38} & \emph{84} & \emph{199} & \emph{590} & \emph{1758} & \emph{3913} \\
\addlinespace[2pt]
\bottomrule
\end{tabular}
\end{table}

\noindent Cells with fewer than thirty reads: 17 of 45 on
three topics, 3 of 45 on forty-four. That ratio is the whole of
\S\ref{sec:gammaout}'s mechanism.

\section{The resampling unit on the remaining rows}
\label{app:seedscale}

\S\ref{sec:power} establishes that the seed is the resampling unit and that a
three-seed sample cannot see it. That finding killed one axis and downgraded one row, and was
not carried to the four intervention arms of \S\ref{sec:sched} or to the two per-model orderings
that carry claims 2 and 4. The second omission is the worse of the two. Everything below is post hoc by construction, since it re-reads frozen criteria
on samples that did not exist when they were frozen; no \textsc{failed} row becomes
\textsc{measured} here.

\paragraph{The design effect of $3.75$ is a property of levels, not of contrasts.}
The four intervention arms compare a treated run against a control run \emph{on the same seed},
so their criteria are paired contrasts, and a contrast cancels whatever the seed does to both
arms. That defence cannot be checked at three seeds -- the content of \S\ref{sec:power} is
precisely that three seeds read the seed variance as near zero when it is not -- so it is
checked where a $44$-seed sample exists, on the schedule axis, and only then transported.

\begin{table}[h]
\centering\small
\caption{The same runs, the same $44$ seeds, two quantities. Levels reproduce
\S\ref{sec:power}'s design effect of $3.75$ measured on nine seeds; paired contrasts do not
carry it. The right-hand column is what a three-seed subsample of the identical data reports,
and for levels it is wrong by a factor of four -- which is why the contrast column may not be
read off three seeds either, and is read off forty-four.}
\label{tab:seedscale}
\begin{tabular}{@{}lcccc@{}}
\toprule
quantity, at $44$ seeds & $\rho$ & design effect & clustered/naive se & $\rho$ at $3$ seeds \\
\midrule
level R=45 front & 0.665 & 3.66 & 1.85 & 0.084 \\
level R=75 front & 0.662 & 3.65 & 1.83 & 0.141 \\
level R=120 front & 0.639 & 3.56 & 1.80 & 0.170 \\
\midrule
contrast R=45 front-back & 0.118 & 1.47 & 1.15 & 0.027 \\
contrast R=75 front-back & 0.066 & 1.26 & 0.99 & 0.081 \\
contrast R=120 front-back & 0.035 & 1.14 & 0.93 & 0.284 \\
\bottomrule
\end{tabular}
\end{table}

\noindent The correction the intervention criteria need is a factor of
1.14--1.47 and not of
3.56--3.66, and for a contrast the naive
standard error is if anything conservative (clustered/naive
0.93--1.15). What three seeds
\emph{do} cost is stated next, and it is a different thing.

\paragraph{Each frozen criterion, re-evaluated one seed at a time.} Every
criterion in these arms is a count over five models of a comparison between two three-seed
means. Evaluated one seed at a time it becomes three counts, and a criterion that clears its bar
on one seed and misses on another was decided by the seed rather than by the intervention.

\begin{table}[h]
\centering\small
\caption{The eight frozen criteria of the four 3-seed arms. Column three is the
published value and is reproduced exactly by this script, which is the check that the
re-evaluation is of the same criterion. Column four evaluates the identical rule on each seed
alone. Column five replaces the three-seed control by \texttt{bimodal2}'s $44$-seed control,
available for the high-$f_0$ cells at no cost.}
\label{tab:perseed}
\begin{tabular}{@{}llccccl@{}}
\toprule
arm & $f_0$ & criterion & published & per seed & vs $44$-seed control & reading \\
\midrule
\texttt{E-gate} & low & P-G1 low & 5/5 & 5/5/5 & --- & \textbf{not a count of five} \\
\texttt{E-gate} & high & P-G1 high & 5/5 & 5/5/5 & 5/5 & \textbf{not a count of five} \\
\texttt{E-A} & low & P-A1 & 4/5 & 5/3/5 & --- & \textbf{seed-fragile} \\
\texttt{E-A} & high & P-A2 worse & 1/5 & 1/1/1 & 0/5 & stable \\
\texttt{E-prov} & low & P-P1 & 4/5 & 5/5/4 & --- & stable \\
\texttt{E-prov} & high & P-P2 & 2/5 & 1/0/1 & 3/5 & stable \\
\texttt{E-label} & high & P-L2 (row 11) & 5/5 & 5/5/5 & --- & stable \\
\texttt{E-label} & high & P-L1 (row 13) & 0/5 & 0/0/0 & --- & stable \\
\bottomrule
\end{tabular}
\end{table}

\noindent One criterion of eight is seed-fragile: E-A low (P-A1). It is graded
\textsc{underpowered} in the scorecard rather than \textsc{measured}. Two entries
carry a different defect: E-gate's treated arm has a cross-model spread of
0.0003 against a control
spread of 0.3000, so its
$5/5$ is one collapse observed five times and not five independent confirmations -- which is
\S\ref{sec:sched}'s own finding about that gate, stated as a limit on its own evidence.
\textbf{No verdict changes when the control is replaced}, though the control moves a long way:

\begin{table}[h]
\centering\small
\caption{The ungated, unmarked, no-slot control at $f_0{=}0.9$, as the four
intervention arms measured it and as \texttt{bimodal2} measures the identical configuration.
Nothing was re-run; the second column is $44$ seeds that were already on disk.}
\label{tab:control}
\begin{tabular}{@{}lcccc@{}}
\toprule
model & capture, $3$ seeds & capture, $44$ seeds & $\bar f$, $3$ & $\bar f$, $44$ \\
\midrule
\texttt{gemma-3-27b-it} & 0.00 & 0.682 & 0.150 & 0.724 \\
\texttt{granite-4.1-8b} & 0.33 & 0.705 & 0.444 & 0.735 \\
\texttt{phi-4} & 0.33 & 0.727 & 0.408 & 0.757 \\
\texttt{ministral-8b-2512} & 0.67 & 0.818 & 0.537 & 0.808 \\
\texttt{qwen3-8b} & 1.00 & 1.000 & 0.971 & 0.939 \\
\bottomrule
\end{tabular}
\end{table}

\noindent The largest move is 0.682 in capture rate.
Note also what no amount of care in the analysis can buy back: with three seeds as the unit, the
smallest two-sided $p$ a sign test can attain is
0.25, whatever the data say. At seven seeds it is
0.0156.

\paragraph{One $\rho$ for five models is an assumption.}
Prop.~\ref{prop:se} writes the design effect as $1+(m{-}1)\rho$ with a single $\rho$. Measured
pair by pair at $44$ seeds, $\rho$ is not one number:

\begin{table}[h]
\centering\small
\caption{Intra-seed correlation over the ten model pairs, and the design effect
each pair implies. The pooled figure this paper quotes is a mean over these.}
\label{tab:rhohom}
\begin{tabular}{@{}lcccc@{}}
\toprule
arm & pairs & $\rho$ range & mean $\rho$ & design effect range \\
\midrule
bimodal2 f0=0.9, 44 seeds & 10 & $+0.29$ to $+0.85$ & $+0.53$ & 2.14--4.39 \\
transient3 front R=45, 44 seeds & 10 & $+0.24$ to $+0.95$ & $+0.65$ & 1.94--4.79 \\
transient3 front R=120, 44 seeds & 10 & $+0.18$ to $+0.99$ & $+0.62$ & 1.73--4.95 \\
\bottomrule
\end{tabular}
\end{table}

\noindent The design effect implied by the extreme pairs spans
3.22, which is the size of the effect itself. Every
correction's \emph{direction} is unchanged and the mean $\rho$ is
+0.53, so $3.75$ survives as a pooled figure; it does not
survive as a per-pair one, and \S\ref{sec:power} now quotes the range.

\paragraph{The two per-model orderings, at forty-four seeds.} Claims 2 and 4 are
rank correlations over five models whose \emph{dependent} variable is a capture rate. At three
seeds that rate has a standard error of $0.29$ on a $[0,1]$ scale by \eqref{eq:se}, which is the
construction \S\ref{sec:power}'s claimbox blames for a spurious $+1.00$ elsewhere in this paper.
Both are recomputed here on the $44$-seed capture rates. They move in opposite directions.

\begin{table}[h]
\centering\small
\caption{Claim 4's ordering, at three seeds and at $44$. $p$ is exact over all
$120$ permutations, not asymptotic. The last two columns are the cross-model spread of the
\emph{predictor}: at three topics the early copy rate separates the five models by $0.33$, at
forty-four by $0.06$.}
\label{tab:row4}
\begin{tabular}{@{}lccccc@{}}
\toprule
window & $\rho$ at $3$ seeds & $p$ & $\rho$ at $44$ & $p$ & predictor spread, $3\to44$ \\
\midrule
$W=3$ & $+0.921$ & 0.067 & $+0.500$ & 0.450 & 0.333 $\to$ 0.061 \\
$W=5$ & $+0.921$ & 0.067 & $+0.308$ & 0.667 & 0.400 $\to$ 0.068 \\
$W=8$ & $+0.975$ & 0.033 & $+0.564$ & 0.400 & 0.458 $\to$ 0.065 \\
$W=10$ & $+0.975$ & 0.033 & $+0.800$ & 0.133 & 0.500 $\to$ 0.080 \\
\bottomrule
\end{tabular}
\end{table}

\noindent Claim 4 does not survive: at $44$ seeds the correlation runs
$+0.31$ to $+0.80$ with no window reaching
$p<0.13$, against $+0.92$ to $+0.98$ at $p \le 0.067$ as published. The
reason is visible in the last column rather than in the correlation: the predictor's cross-model
spread falls from
0.333 to 0.061,
so there is nothing left to order with. Claim 2 moves the other way. Its published statistic is
$\rho=+0.821$ with one adjacent inversion, on a simulation from
three-topic $\gamma$ against three observed seeds ($p=0.133$); against
the $44$-seed capture rates the same construction gives
$\rho=+0.900$ from the published curve and
$\rho=+1.000$
($p=0.017$) from the curve measured on the
topics it predicts. \textbf{The inversion was an artefact of three observed seeds}, and claim
2's ordering is exact -- which no reader could have known from the number we published, because
we published the weaker one.

\section{The answer space and the empty interval}
\label{app:answerspace}

The apparatus plants one false value and corrupts one digit, so the standing
objection is that the outcome variable is effectively two-valued and that a two-valued outcome
is bimodal for reasons that need no language model. The objection has an arithmetic half and a
mechanism half, and they have different answers. Everything here is scripted or read off runs
already on disk; no API calls, nothing graded.

\paragraph{First, the premise.} It is not uniformly true. Over the $220$ runs
of E-bimodal2 the median run writes no novel value at all, but the mean is 9.78
and one run writes 128 of $150$. Writing $q$ for the
share of non-truth writes that took a value the store had not seen, the apparatus realises
$q=0.083$ pooled and $q=0.331$ in its
highest-scattering model. That model has 0
of $44$ runs in the dead band.

\begin{table}[h]
\centering\small
\caption{The realised answer space of E-bimodal2, per model, and the outcome
distribution conditional on it. $q$ is the share of non-truth writes taking a fresh value. The
split is observational and confounded with model identity, which is why it is also given per
model.}
\label{tab:answerspace}
\begin{tabular}{@{}lccccc@{}}
\toprule
model & novel writes, median & max & $q$ & capture & in dead band \\
\midrule
\texttt{gemma-3-27b-it} & 2 & 10 & 0.02 & 0.682 & 1/44 \\
\texttt{granite-4.1-8b} & 0 & 0 & 0.00 & 0.705 & 2/44 \\
\texttt{ministral-8b-2512} & 0 & 0 & 0.00 & 0.818 & 2/44 \\
\texttt{phi-4} & 0 & 0 & 0.00 & 0.727 & 2/44 \\
\texttt{qwen3-8b} & 43 & 128 & 0.33 & 1.000 & 0/44 \\
\midrule
\emph{novel writes 0} & \multicolumn{3}{c}{$n=161$} & 0.795 & 6/161 \\
\emph{novel writes 1-4} & \multicolumn{3}{c}{$n=27$} & 0.667 & 0/27 \\
\emph{novel writes 5-19} & \multicolumn{3}{c}{$n=6$} & 0.167 & 1/6 \\
\emph{novel writes >=20} & \multicolumn{3}{c}{$n=26$} & 1.000 & 0/26 \\
\bottomrule
\end{tabular}
\end{table}

\paragraph{The same split on real facts, and it points the other way.} E-real
plants \emph{another real capital} rather than a one-digit corruption
(\texttt{realfacts.py}), so its falsehoods are semantically distinct from the truth; its models
produce novel errors at a pooled $q$ of 0.91; and, unlike E-bimodal2, its runs log the number of distinct values in each
retrieved window, so the store's multi-valuedness is measured rather than inferred from a count
of novel \emph{writes}. The split below is post hoc and confounded with model identity, which is
why it is also given within model.

\begin{table}[h]
\centering\small
\caption{E-real's runs split by whether the store went multi-valued. The band is
the registered void at $f_0{=}0.1$ and the dead band at $f_0{=}0.9$. \emph{The multi-valued
subset has the fuller band}, which is the direction the scripted sweep predicts and the opposite
of what E-bimodal2's highest-$q$ model shows.}
\label{tab:realsplit}
\begin{tabular}{@{}lcccccc@{}}
\toprule
subset & $n$ & in band & share & capture & $\bar f$ & distinct retrieved \\
\midrule
\multicolumn{7}{@{}l}{\emph{f0=0.1 (the registered level, P-T1-N6)}, band $(0.306, 0.512)$} \\
\quad two-valued (novel $=0$) & 338 & 4 & 0.012 & 0.027 & 0.031 & 1.21 \\
\quad multi-valued (novel $\ge 1$) & 22 & 1 & 0.045 & 0.455 & 0.389 & 2.13 \\
\quad multi-valued (novel $\ge 5$) & 19 & 1 & 0.053 & 0.526 & 0.445 & 2.27 \\
\quad \emph{within} \texttt{granite-4.1-8b} & \multicolumn{6}{l}{\emph{two-valued} 1/163, \emph{multi-valued} 1/17} \\
\quad \emph{within} \texttt{ministral-8b-2512} & \multicolumn{6}{l}{\emph{two-valued} 3/175, \emph{multi-valued} 0/5} \\
\addlinespace[3pt]
\multicolumn{7}{@{}l}{\emph{f0=0.9}, band $(0.35, 0.5)$} \\
\quad two-valued (novel $=0$) & 311 & 6 & 0.019 & 0.103 & 0.160 & 1.66 \\
\quad multi-valued (novel $\ge 1$) & 49 & 5 & 0.102 & 0.592 & 0.544 & 2.39 \\
\quad multi-valued (novel $\ge 5$) & 41 & 5 & 0.122 & 0.707 & 0.636 & 2.51 \\
\quad \emph{within} \texttt{granite-4.1-8b} & \multicolumn{6}{l}{\emph{two-valued} 1/141, \emph{multi-valued} 4/39} \\
\quad \emph{within} \texttt{ministral-8b-2512} & \multicolumn{6}{l}{\emph{two-valued} 5/170, \emph{multi-valued} 1/10} \\
\addlinespace[3pt]
\bottomrule
\end{tabular}
\end{table}

\noindent The two arms disagree by a factor of
6.4 and they differ in exactly one thing the scripted
largest-block reader is sensitive to: a synthetic novel value is another one-digit variant of the
same numeral, a real one is a different city. \emph{The operative variable is plausibly how
distinguishable a store's falsehoods are rather than how many there are}, a hypothesis neither
arm was built to test, which is why
\texttt{prereg/E\_MULTIVALUE\_PREREG.md} was amended to vary distinguishability at fixed count
before it is run, and why row 1k is graded \textsc{split} rather than resolved either way.

\paragraph{The arithmetic half: a two-valued space does not empty the interval.}
$f_{\mathrm{end}} = (n_{\mathrm{false},0} + \text{false writes})/(n_0+T)$, so a reader whose copy
probability does not depend on the store lands where that constant puts it, in the interior. The
interval is empty only if the copy probability \emph{rises} with the number of false records
retrieved, which is a statement about the reader and is what $\gamma(j)$ measures. Scripted, on
the same two-valued space:

\begin{table}[h]
\centering\small
\caption{A state-independent reader, copy probability $p$, on the identical urn.
Its dead band is not empty, so bimodality is not a consequence of the answer space having two
values.}
\label{tab:constp}
\begin{tabular}{@{}lccccc@{}}
\toprule
$p$ & capture & in dead band & $\bar f_{\mathrm{end}}$ & min & max \\
\midrule
0.1 & 0.000 & 0/400 & 0.150 & 0.094 & 0.225 \\
0.3 & 0.000 & 153/400 & 0.337 & 0.231 & 0.431 \\
0.5 & 0.723 & 111/400 & 0.528 & 0.438 & 0.625 \\
0.7 & 1.000 & 0/400 & 0.711 & 0.588 & 0.806 \\
0.9 & 1.000 & 0/400 & 0.901 & 0.831 & 0.956 \\
\bottomrule
\end{tabular}
\end{table}

\paragraph{The mechanism half, where the objection has force.} Distinct
falsehoods do not reinforce one another: a reader shown five different wrong values has no
majority to copy. \texttt{baseline\_urn.py}'s docstring asserted the direction of that effect and
never tested it. Below, the scripted reader is given the scattering knob $q$ and two decision
rules -- copy with probability $\gamma(j)$ where $j$ counts non-truth records, or where $j$ is
the largest single-value block among them. Only the second can feel scattering at all, and it is
the realistic one.

\begin{table}[h]
\centering\small
\caption{Scattering sweep on the measured $\gamma$ of the model whose realised
answer space is richest. Under a largest-block rule the empty interval is a statement about a
regime, and the regime has a boundary between $q{=}0.25$ and $q{=}0.5$.}
\label{tab:scatter}
\begin{tabular}{@{}lcccc@{}}
\toprule
$q$ & rule & capture & in dead band & $\bar f_{\mathrm{end}}$ \\
\midrule
0.0 & count non-truths & 1.000 & 0/400 & 0.951 \\
0.25 & count non-truths & 1.000 & 0/400 & 0.957 \\
0.5 & count non-truths & 1.000 & 0/400 & 0.948 \\
0.75 & count non-truths & 1.000 & 0/400 & 0.950 \\
1.0 & count non-truths & 1.000 & 0/400 & 0.953 \\
\midrule
0.0 & largest block & 1.000 & 0/400 & 0.952 \\
0.25 & largest block & 0.963 & 14/400 & 0.697 \\
0.5 & largest block & 0.260 & 241/400 & 0.448 \\
0.75 & largest block & 0.000 & 67/400 & 0.288 \\
1.0 & largest block & 0.000 & 0/400 & 0.196 \\
\bottomrule
\end{tabular}
\end{table}

\noindent Read together with the first table this is a bound rather than a
dismissal. The apparatus reaches $q=0.33$ in one model and the interval is empty there; the
scripted sweep puts the boundary between $q=0.25$ and $q=0.5$ under the rule that can feel
scattering. Claim 1 is therefore a statement about stores whose falsehoods concentrate, the measured
regime includes $q$ up to 0.33, and a language-model arm with three or more candidate
answers is owed and registered: it is the one experiment that would move claim 1's scope rather
than its confidence.

\section{Four estimators for $\gamma$}
\label{app:gammaest}

The fill rule is the estimator, and it has to be bounded rather than
described. Monotonicity gives a bracket for nothing: if $\gamma$ is
non-decreasing in $j$, filling an unobserved cell from the nearest measured cell \emph{below}
is a pointwise lower bound on the curve and filling from \emph{above} is a pointwise upper
bound, so simulating both brackets the capture rate a fully measured curve would have produced.
Isotonic regression on the observed cells, weighted by their support, is a third estimator that
uses every read rather than the nearest one; dropping cells under thirty reads is a fourth.
Cell-level uncertainty is a cluster bootstrap over seeds, because the reads in a cell are not
independent draws but a handful of runs contributing many reads each.

\begin{table}[h]
\centering\scriptsize
\caption{The same comparison as Table~\ref{tab:gammaout} under four estimators,
with a seed-clustered bootstrap band on the simulated rate. Arm A is the published curve
(three topics); arm B re-measures on the forty-four it predicts. The bracket columns coincide
wherever no cell needs filling, which is most of arm A.}
\label{tab:gammaest}
\begin{tabular}{@{}llcccccccc@{}}
\toprule
arm & model & obs. & cells & thin & below & above & isotonic & $\ge 30$ only & bootstrap $95\%$ \\
\midrule
A & \texttt{gemma-3-27b-it} & 0.682 & 9/9 & 3 & 0.081 & 0.081 & 0.081 & 0.000 & $[0.000, 0.151]$ \\
 & \texttt{granite-4.1-8b} & 0.705 & 9/9 & 2 & 0.499 & 0.499 & 0.499 & 0.301 & $[0.000, 0.842]$ \\
 & \texttt{ministral-8b-2512} & 0.818 & 9/9 & 2 & 0.545 & 0.545 & 0.549 & 0.365 & $[0.227, 0.763]$ \\
 & \texttt{phi-4} & 0.727 & 9/9 & 4 & 0.399 & 0.399 & 0.401 & 0.000 & $[0.000, 0.966]$ \\
 & \texttt{qwen3-8b} & 1.000 & 4/9 & 2 & 1.000 & 1.000 & 1.000 & 0.665 & $[0.993, 1.000]$ \\
& \emph{mean $|$sim$-$obs$|$, seed units} & & & & \emph{12.4} & \emph{12.4} & \emph{12.3} & \emph{22.9} & \\
\addlinespace[2pt]
B & \texttt{gemma-3-27b-it} & 0.682 & 9/9 & 0 & 0.606 & 0.606 & 0.606 & 0.606 & $[0.412, 0.807]$ \\
 & \texttt{granite-4.1-8b} & 0.705 & 9/9 & 0 & 0.789 & 0.789 & 0.789 & 0.789 & $[0.645, 0.921]$ \\
 & \texttt{ministral-8b-2512} & 0.818 & 9/9 & 0 & 0.836 & 0.836 & 0.836 & 0.836 & $[0.705, 0.955]$ \\
 & \texttt{phi-4} & 0.727 & 9/9 & 0 & 0.807 & 0.807 & 0.807 & 0.807 & $[0.659, 0.906]$ \\
 & \texttt{qwen3-8b} & 1.000 & 8/9 & 2 & 1.000 & 1.000 & 1.000 & 1.000 & $[1.000, 1.000]$ \\
& \emph{mean $|$sim$-$obs$|$, seed units} & & & & \emph{2.3} & \emph{2.3} & \emph{2.3} & \emph{2.3} & \\
\addlinespace[2pt]
\bottomrule
\end{tabular}
\end{table}

\noindent \textbf{No estimator repairs arm A, and the reason is that the fill
rule barely fires.} It invents cells for \texttt{qwen3-8b}
alone, which is the one model arm A predicts exactly; the four models that under-predict have
every cell measured, so their bracket has zero width and isotonic regression moves the answer by
0.1 of a seed unit. Dropping
thin cells and filling from below makes arm A worse
(22.9 seed units, under-predicting
5/5); dropping them and filling from \emph{above}, which
replaces measured data with an upper bound, still leaves
8.0 seed units. The one-sidedness lives in the measured cells: arm A
is below arm B by +0.029 on average over
well-supported cells (18/28 of them positive) and
+0.044 over thin ones
(8/12). \emph{The fill rule is not the diagnosis of arm A's one-sidedness}: three topics are not
forty-four, and that is the whole of it.

\noindent The bootstrap column carries a second correction, in the other
direction. \S\ref{sec:gammaout}'s $z$ scores divide by the binomial error of the
\emph{observed} rate and treat the simulated rate as exact. It is not exact: it is a function of
a $\gamma$ estimated from three seeds, and the widest band here is
0.966 wide. With that uncertainty put back, arm A's
observed rate falls \emph{inside} its own band for
3/5 models, so
arm A's failure is established for two models of five rather than for the five the $z$ scores
imply. Arm B's is inside for
5/5,
which is what a landed prediction looks like. Both corrections were found by building the bounded
estimator.

\section{Three rankers on the same stores}
\label{app:retfamily}

\S\ref{sec:retriever} attributes a displacement of $0.500$ to ranking
determinism and measures it with one embedder. The two quantities that attribution rests on
(whether a score can separate a truth record from its one-digit corruption, and how many distinct
values survive into the top-$k$) are properties of the ranker and of the stored text, so they can
be computed offline for any ranker, on the exact form \texttt{retriever.record\_text} produces,
over stores of the shape the loop builds. Query: \texttt{What is the Q3 revenue of Palgrove Dynamics?}; truth
\texttt{28.35M}; seeded falsehood \texttt{28.36M}.

\begin{table}[h]
\centering\small
\caption{Three rankers, four store compositions. $k_{\rm eff}$ is distinct
on-topic values in the top-$k$, over the number available. ``$k$ oldest'' marks a ranker whose
top-$k$ is exactly the $k$ oldest on-topic records, which is what happens when every on-topic
score ties and insertion order decides.}
\label{tab:retfamily}
\begin{tabular}{@{}llccccc@{}}
\toprule
store & ranker & $k_{\rm eff}$/avail & truth in top-$k$ & sep(T,F) & $k$ oldest & purity \\
\midrule
t=0    9 false / 1 true & \texttt{bm25} & 1/2 & no & 0.0000 & yes & 1.00 \\
 & \texttt{char3} & 1/2 & no & 0.0000 & yes & 1.00 \\
 & \texttt{minilm} & 2/2 & yes & 0.0052 & no & 1.00 \\
\addlinespace[2pt]
t=70   captured, 79/1 & \texttt{bm25} & 1/2 & no & 0.0000 & yes & 1.00 \\
 & \texttt{char3} & 1/2 & no & 0.0000 & yes & 1.00 \\
 & \texttt{minilm} & 2/2 & yes & 0.0052 & no & 1.00 \\
\addlinespace[2pt]
t=70   escaped, 9/71 & \texttt{bm25} & 1/2 & no & 0.0000 & yes & 1.00 \\
 & \texttt{char3} & 1/2 & no & 0.0000 & yes & 1.00 \\
 & \texttt{minilm} & 1/2 & yes & 0.0052 & no & 1.00 \\
\addlinespace[2pt]
t=70   captured + 3 novel & \texttt{bm25} & 1/5 & no & 0.0000 & yes & 1.00 \\
 & \texttt{char3} & 1/5 & no & 0.0000 & yes & 1.00 \\
 & \texttt{minilm} & 4/5 & yes & 0.0052 & no & 1.00 \\
\addlinespace[2pt]
\bottomrule
\end{tabular}
\end{table}

\noindent Every ranker's retrieved value set is identical on
bm25 12/12; char3 12/12; minilm 12/12
store compositions as the store grows. The largest truth-versus-falsehood separation over all
three is 0.0052, and for
\texttt{bm25}, \texttt{char3} it is
exactly zero. \emph{That is structural rather than a weakness of any of them}: the query is the
question and the corruption is a digit inside the answer, so no query--document score carries
information about the corrupted field. \textsc{MiniLM}'s $k_{\rm eff}$ runs
1--4 against
1--1 for BM25 and
1--1 for character trigrams, so the deployed embedder is the
\emph{least} degenerate of the three and \S\ref{sec:retriever} measured the most favourable
case. A lexical index does not restore a draw; it replaces a similarity rule by a
first-occurrence one.

\section{E-multivalue: the arm bought for row 1k}
\label{app:multivalue}

\S\ref{sec:answerspace} left one question open with evidence on both sides, and
this arm was registered, amended twice and run to close it. Three cells at $f_0{=}0.9$, twenty
seeds, five models, the configuration otherwise equal to E-bimodal2 field for field (asserted by
the runner before any call is made): \texttt{near} plants three falsehoods that are one-digit
corruptions, \texttt{far} plants three at different magnitudes plus a leading-digit change, and
\texttt{sem} plants three different real capitals. \texttt{near} against \texttt{far} is the
load-bearing contrast and holds the world fixed; \texttt{sem} changes the world too and is a
bridge to E-real, which is that cell at one falsehood.

\paragraph{The manipulation check, first, because a criterion evaluated on a
manipulation that did not move is not evidence.} It moved the planted values and it did not move
the store.

\begin{table}[h]
\centering\small
\caption{What the manipulation did. The first two columns are the planted values'
own spread. The rest is the largest single-value non-truth block among the $k{=}8$ retrieved
records, and the number of distinct values retrieved, averaged over the first $W$ steps. The two
cells are indistinguishable at every window: three planted values give a block of three whether
they are near or far apart, which is what makes the contrast a test of whether the \emph{model}
reads near values as agreeing.}
\label{tab:mvmanip}
\begin{tabular}{@{}lcccccccccccc@{}}
\toprule
& \multicolumn{2}{c}{planted max/min} & \multicolumn{5}{c}{largest block, first $W$} & \multicolumn{5}{c}{distinct, first $W$} \\
cell & median & max & $1$ & $3$ & $5$ & $10$ & $150$ & $1$ & $3$ & $5$ & $10$ & $150$ \\
\midrule
\texttt{near} & 1.46$\times$ & 6.85$\times$ & 3.00 & 3.14 & 3.26 & 3.50 & 5.08 & 3.75 & 3.85 & 3.81 & 3.73 & 2.65 \\
\texttt{far} & 100.01$\times$ & 100.16$\times$ & 3.00 & 3.14 & 3.21 & 3.50 & 5.54 & 3.75 & 3.80 & 3.82 & 3.73 & 2.57 \\
\bottomrule
\end{tabular}
\end{table}

\noindent Read the last column of each block first. The store starts with three
distinct falsehoods and \emph{re-concentrates}: the distinct count retrieved falls from
3.75 at the first step to
2.65 over the tenure, and the largest block rises from
3.00 to 5.08.
Whichever falsehood is copied first grows a block, and the initial scatter is transient. That is
a mechanism this arm measured and did not predict, and it is the reason the two criteria below
come out as they do.

\begin{table}[h]
\centering\small
\caption{E-multivalue's three cells against the published $v{=}1$ arm. The band
is $[0.35,0.50]$, the interval \S\ref{sec:bimodal} calls empty.}
\label{tab:multivalue}
\begin{tabular}{@{}lcccc@{}}
\toprule
cell & $n$ & share in the band & capture & $\bar f_{\mathrm{end}}$ \\
\midrule
\texttt{near} & 100 & 0.030 & 0.870 & 0.849 \\
\texttt{far} & 100 & 0.000 & 0.970 & 0.924 \\
\texttt{sem} & 100 & 0.010 & 0.070 & 0.129 \\
\midrule
\texttt{bimodal2} ($v{=}1$) & 220 & 0.032 &
            --- & --- \\
\bottomrule
\end{tabular}
\end{table}

\noindent Both graded criteria came out against the hypothesis that bought the
arm. P-T1-V1$''$ asked for \texttt{far} to exceed \texttt{near} by $0.10$ and got
-0.030: \textsc{failed}. Paired over the twenty seeds the
two cells share, \texttt{far} is fuller on 0 and
\texttt{near} on 3, at a sign-test $p$ of
0.25 where twenty seeds could have reached
2e-06. P-T1-V2$''$ asked whether the count
matters and predicted it does not: \texttt{near} sits -0.002 from the
published one-falsehood arm, \textsc{supported}. The bridge cell also lands at
0.01 beside \S\ref{sec:real}'s
0.019 rather than beside \texttt{far}'s
0.0, so the world carries an effect that distinguishability does not.

\paragraph{P-T1-V3: which variable is $\gamma$ a function of?} No bar, by
design: a twenty-seed arm should not decide a definition. Both $\varphi = j/k$ and the largest
single-value non-truth block $j^*$ are one-dimensional lookup tables over $0 \ldots k$, fitted per
model and held out per seed over 44,989 steps.
$\varphi$ scores -7,791.3 and $j^*$ scores
-9,282.6, so \textbf{$\varphi$ wins by
1,491.3} and Def.~\ref{def:gamma} names the right
variable. That is a result in the primitive's favour and it is also what makes the two nulls
inevitable: \emph{a reader responding to the count of non-truth records cannot respond to how
those records disagree with one another}. The scripted largest-block reader of
\S\ref{sec:answerspace} is a refuted model of this apparatus's reader, and row 1j withdraws the
bound drawn from it.

\section{Runs, steps and cost}
\label{app:cost}

Every experiment in this paper, counted from \texttt{results/}. A ``step'' is
one model call that produced a parseable answer and therefore one write to the store; steps
whose answer did not parse are absent from the log rather than zero-filled, which is one of the
invariants in \S\ref{app:invariants}. The total is 673,008 logged steps.

\begin{table}[h]
\centering\footnotesize
\caption{The experimental record, counted at build time.}
\label{tab:appcost}
\begin{tabular}{@{}lrrr@{}}
\toprule
experiment & runs & models & logged steps \\
\midrule
\texttt{bimodal2} & 220 & 5 & 33,000 \\
\texttt{budget} & 90 & 5 & 13,461 \\
\texttt{cl\_n10} & 9 & 3 & 270 \\
\texttt{cl\_n20} & 9 & 3 & 540 \\
\texttt{cl\_n40} & 9 & 3 & 1,080 \\
\texttt{cl\_n80} & 9 & 3 & 2,160 \\
\texttt{converge} & 30 & 5 & 4,500 \\
\texttt{e1e2} & 144 & 1 & 5,759 \\
\texttt{ea} & 30 & 5 & 4,500 \\
\texttt{f0} & 60 & 1 & 8,991 \\
\texttt{frontier} & 40 & 2 & 6,000 \\
\texttt{gate} & 30 & 5 & 4,500 \\
\texttt{label} & 30 & 5 & 4,499 \\
\texttt{multivalue} & 300 & 5 & 44,989 \\
\texttt{prov2} & 60 & 5 & 9,000 \\
\texttt{prov\_observed\_apparatus\_failure} & 30 & 5 & 4,500 \\
\texttt{ratchet} & 42 & 3 & 6,300 \\
\texttt{retriever} & 420 & 3 & 63,000 \\
\texttt{sched} & 45 & 5 & 6,750 \\
\texttt{temp} & 200 & 2 & 30,000 \\
\texttt{transient} & 135 & 5 & 20,229 \\
\texttt{transient2} & 180 & 5 & 26,983 \\
\texttt{transient3} & 1,980 & 5 & 296,997 \\
\texttt{verify} & 500 & 5 & 75,000 \\
\midrule
\textbf{total} & \textbf{4,602} & & \textbf{673,008} \\
\bottomrule
\end{tabular}
\end{table}

% A barrier here and nowhere else. placeins is loaded without [section] because a barrier
% at every section costs the family 41 short pages; the one place drift actually hurts is
% this one, where a float that has not been placed by the time the scorecard starts sails
% past every continuation page of it and lands on the last page of the paper. T2A's
% Figure 14 and Table 15 did exactly that on 2026-08-24.
\FloatBarrier
\section{The scorecard, every graded row}
\label{app:scorecard}

This paper's claims are graded here rather than in its sections, so that the sections can be about what was found and this appendix can be complete. Criteria were frozen in \texttt{prereg/} before the corresponding data existed, and one generator writes all three tables from one source, so the index below cannot disagree with the rows it counts.

% GENERATED by papers/split_claims.py -- do not edit by hand.
% Counted from the status cells of the same rows the two scorecards print.
{\small
\captionof{table}{\textbf{The scorecard in eight lines, counted from the scorecard.} Every graded row of Table~\ref{tab:claims} and Table~\ref{tab:claimsapp} falls in exactly one class below. Where a family's status names two classes at once, the row is counted in the \emph{more serious} one, which is the only direction of error a summary of one's own scorecard may have. It is an index, not a substitute, and the rows are all still there: 27 claim families in Table~\ref{tab:claims}, 63 sub-rows in Table~\ref{tab:claimsapp}. \emph{What the unflattering rows are for.} 46 of the 87 rows carry a positive grade and 37 an unflattering one, and the second number is not a defect count: each of those rows narrows a claim made above it, names a number of ours that moved, or marks a question this apparatus cannot decide. \textbf{None of them is an experiment left undone}: work nobody has run is not graded here at all, and is bounded where it is stated.}
\label{tab:gradesummary}
\begin{longtable}{@{}lrp{0.58\linewidth}@{}}
\toprule
grade & rows & what the grade means in this paper \\
\midrule
\endfirsthead
\multicolumn{3}{@{}l}{\emph{Table~\ref{tab:gradesummary}, continued.}} \\
\toprule
grade & rows & what the grade means in this paper \\
\midrule
\endhead
withdrawn / superseded & $7$ & asserted, then taken back or re-assigned; the sequence is in the post-mortems \\
not supported / failed & $11$ & a frozen criterion was missed, on our own data \\
self-correction & $7$ & our own arithmetic or our own criterion was wrong and is repaired here \\
not decidable / not evaluable & $8$ & the apparatus or the band cannot separate the outcomes: not a null \\
acknowledged limit & $4$ & a boundary of the result, stated where the result is \\
formal & $1$ & proved from the definitions; no measurement can move it \\
measured & $39$ & an artifact on disk, and a reader may dispute the artifact \\
argued / reported & $6$ & no grade is claimed: an argument, or a number printed without a criterion \\
\emph{other} & $4$ & \emph{status words outside the eight classes above} \\
\midrule
total graded rows & $87$ & every numbered row of both tables \\
\bottomrule
\end{longtable}}

% GENERATED by papers/split_claims.py -- do not edit by hand.
{\footnotesize\setlength{\extrarowheight}{0pt}\renewcommand{\arraystretch}{1.0}
\begin{longtable}{@{}p{\dimexpr0.35\linewidth-1.333\tabcolsep\relax}p{\dimexpr0.17\linewidth-1.333\tabcolsep\relax}p{\dimexpr0.40\linewidth-1.333\tabcolsep\relax}@{}}
\caption{Every registered claim, one row each. Criteria were frozen in \texttt{prereg/} before the corresponding data existed; failures are graded, not removed. Grades: \textbf{measured} (criterion met), \textbf{failed}/\textbf{not supported} (criterion missed), \textbf{refuted} (predicted direction wrong), \textbf{withdrawn} (we retract it), \textbf{untested} (apparatus floored), \textbf{unresolved}/\textbf{undecided} (sample insufficient), \textbf{split} (two admissible observables disagree), \textbf{derived} (not measured here), \textbf{underpowered} (the criterion's own bar is met pooled and missed on at least one of its seeds), \textbf{diagnosed} (a mechanism identified for a graded row, not itself a registered claim).}\label{tab:claims}\\
\toprule
\textbf{Claim} & \textbf{Status} & \textbf{Evidence} \\
\midrule
\endfirsthead
\multicolumn{3}{@{}l}{\emph{Table~\ref{tab:claims}, continued.}} \\
\toprule
\textbf{Claim} & \textbf{Status} & \textbf{Evidence} \\
\midrule
\endhead
\bottomrule
\endlastfoot
\midrule
1. Outcome distribution bimodal, not a decaying mean
  & \textbf{measured}; \emph{below: 2 not\allowbreak\ supported, 1 withdrawn} & $8/7/0$ split; gap $6.4\times$ mean gap \\
2. $\gamma(\varphi)$ alone orders capture rate, zero fitted parameters
  & \textbf{measured}, ordinal; \emph{below: 1 not\allowbreak\ supported, 1 withdrawn} & $\rho=+0.82$ at three seeds; $+1.00$, $p=0.017$ at
  forty-four \\
3. The outcome is decided in the first few steps
  & \textbf{measured}; \emph{below: 1 withdrawn} & $12/15$ at $t{=}2$ of $150$ \\
4. Early copy rate \emph{orders} capture rate
  & \textbf{not\allowbreak\ supported} at forty-four seeds; \textsc{measured} as published
  & $\rho$ $+0.92$/$+0.98$ at three seeds, $+0.31$ to $+0.80$ at forty-four, $p \ge 0.13$ \\
5. The mean-field sign criterion locates the boundary
  & \textbf{measured falsification} of our own & $4/5$ drift down, capture anyway \\
6. Score $= C\cdot s$: contamination costs determinism, not information
  & \textbf{measured} & $+0.383$ pass@8 $-$ pass@1 \\
\midrule
\multicolumn{3}{@{}l}{\emph{Rows 7--7o are one axis in sixteen registered steps, not sixteen
  findings: they record how one claim was}} \\
\multicolumn{3}{@{}l}{\emph{narrowed, floored, split and finally declared unresolvable
  \emph{after} we bought the sample meant to resolve it.}} \\
\midrule
7. At matched budget, grounding \emph{timing} dominates \emph{fraction}
  & \textbf{measured}; \emph{below: 3 not\allowbreak\ supported} & P-S1 $5/5$; $0.20/0.33/0.40$ vs $0.47$ \\
16. The seed selects the topic and dominates the variance
   & \textbf{measured}; invalidates our own error bars & $2.0\%$ at three seeds, $68.8\%$ at
   nine; groups $5.0\sigma$ apart \\
8. A majority admission gate erases model differences
  & \textbf{measured}; its own $5/5$ is one collapse, not five confirmations
  & $0.729$--$0.932 \to 0.994$; treated spread $0.0006$ against control $0.9250$ \\
9. A cost-matched self-fed state slot helps at low $f_0$
  & \textbf{underpowered}: $3$ seeds, criterion met on $2$ of them
  & P-A1 $4/5$ pooled; $5/3/5$ per seed against a bar of $4$ \\
10. A source marker lowers contamination (D13$'$) & \textbf{failed}; \emph{below: 1 not\allowbreak\ supported} & P-P2 $2/5$; P-P3 $3/5$ \\
11. An uninformative marker sits nearer the lying end than the midpoint
  & \textbf{measured} & $0.856$ vs midpoint $0.664$, $5/5$ \\
12. Damage from an intervention scales with unaided cleaning ability
  & \textbf{withdrawn}; \emph{below: 1 withdrawn} & P-R2: gap $0.015$ vs bar $0.30$ \\
13. A marker's effect is a single function of label--truth correlation
  & \textbf{withdrawn}; \emph{below: 1 withdrawn} & P-L1: $0/5$ monotone \\
14. The $n/n_0$ collapse is a discovery about contamination
  & \textbf{withdrawn to an apparatus check}; \emph{below: 1 withdrawn} & a scripted truth-only reader satisfies it exactly \\
15. Provenance bits are needed in this regime & \textbf{vacuous here}
  & $\alpha_{\min}=0$ for $4/5$ models \\
17. The timescale attack: $\varphi$, not $f$, carries the decision, via an absorbing boundary
  & \textbf{not\allowbreak\ supported}, P-T1-D2/D3 & $219/219$ runs leave that boundary;
  \S\ref{sec:limits}(ii$'$) kept word for word \\
18. The retrieval falsifier fired: an embedding ranking displaces capture
  & \textbf{measured}, P-T1-R0/R1; \emph{below: 1 not\allowbreak\ supported} & $-0.500$ at $f_0{=}0.1$ against a zero band of $0.10$ \\
19. The four intervention criteria, re-read against a $44$-seed control
  & \textbf{measured}: no verdict changes & $8$ criteria; control capture moves by up to
  $+0.68$; P-P2 $2/5 \to 3/5$, still failed \\
20. \ldots and at three seeds no clustered sign test can be significant at all
  & \textbf{derived} & minimum attainable two-sided $p$ is $0.25$ at $3$ seeds, $0.016$ at
  $7$ \\
21. \textbf{The two-arm disagreement's within-arm half is decided, and only just}: multi-valued runs hold the dead band at $0.1220$ against $0.0193$
  & \textbf{measured}, at the edge of its own resolution & difference $0.1027$; facts resampled with replacement give $[0.006,\,0.241]$ with $2.0\%$ of draws at or below zero (Tab.~\ref{tab:decidable}) \\
22. \textbf{``The recommended remedy violates append-only'' was too strong}: a 2026 remedy attacks the same failure and removes nothing
  & \textbf{self-\allowbreak correction} & candidate conclusions as separate entries with Noisy-OR probabilities \citep{beliefmem2026}; the claim holds of forgetting, not of remedies \\
23. The abstract said the consistency gate drives every model to $0.994$; the measurement is
  $\gnv{gateHigh}$ & \textbf{self-\allowbreak correction} & every cell lies in
  $[\gnv{gateHighLo}, \gnv{gateHighHi}]$; $\gnv{ceiling}$ is the \emph{ceiling}, a different
  quantity two sentences earlier in the same abstract \\
24. Eleven of this paper's tables are emitted from their artifacts, and generating them found
  three defects & \textbf{code-verifiable} & residual of hand-typed three-decimal values $184 \to
  117$; the three are the two rows below and a $\mathrm{var}(n_f)$ cell that read $0.347$ where its
  grader computes $0.477$ \\
\midrule
\multicolumn{3}{@{}p{0.92\linewidth}}{\emph{One row per claim family. Each \textsc{status} cell names the grades of that family's sub-rows, so a failure is counted here even though its row is elsewhere: all 63 sub-rows are in Table~\ref{tab:claimsapp} in full, of which 5 failed, 5 self-correction, 4 withdrawn, 3 split, 3 acknowledged limit, 2 undecided, 1 untested, 1 unresolved, 1 refuted, 1 not supported. Nothing is dropped, and nothing is summarised into a kinder word than it had.}} \\\\
\end{longtable}}

% GENERATED by papers/split_claims.py -- do not edit by hand.
{\footnotesize\setlength{\extrarowheight}{0pt}
\begin{longtable}{@{}p{\dimexpr0.35\linewidth-1.333\tabcolsep\relax}p{\dimexpr0.17\linewidth-1.333\tabcolsep\relax}p{\dimexpr0.40\linewidth-1.333\tabcolsep\relax}@{}}
\caption{Supporting refinements, moved from Table~\ref{tab:claims} by the stated rule. The direction is deliberate: a selection rule for a scorecard should be biased \emph{against} its authors, so unflattering rows stay in the main table and supporting detail moves here. \texttt{claims\_all.tex} is the source both tables are generated from.}
\label{tab:claimsapp}\\\\
\toprule
\textbf{Claim} & \textbf{Status} & \textbf{Evidence} \\\\
\midrule
\endfirsthead
\multicolumn{3}{@{}l}{\emph{Table~\ref{tab:claimsapp}, continued.}} \\\\
\toprule
\textbf{Claim} & \textbf{Status} & \textbf{Evidence} \\\\
\midrule
\endhead
\bottomrule
\endlastfoot
\multicolumn{3}{@{}p{0.92\linewidth}}{\emph{The 63 rows below were moved here from Table~\ref{tab:claims} by the rule in \texttt{papers/split\_claims.py}: a lettered sub-row whose status is \emph{supportive} moves; every top-level claim and every row carrying a withdrawal, failure or self-correction stays. Nothing is dropped.}} \\\\
\addlinespace[3pt]
1a. The interval survives a $15\times$ sample, sparse rather than empty
  & \textbf{measured} at $n{=}220$ & P-T1-B1$'$: $3.6\%$ vs bar $10\%$, uniform $20.6\%$;
  P-T1-B2 gap $0.799$ \\
1b. The \emph{balance} between the modes
  & \textbf{withdrawn} & $8/7$ becomes $39/173$ \\
1c. Its mean describes $8.2\%$ of the runs
  & \textbf{measured}, frozen first & P-T1-M1: $18/220$ within $\pm0.10$ of $\bar f = 0.792$ \\
1d. Its median describes $68.2\%$, so the objection is to the mean
  & \textbf{reported}; narrows 1c & P-T1-M2: $150/220$; best centre $74.5\%$ \\
1e. Bimodality at frontier scale, against our registered prediction
  & \textbf{measured}, bar $0.50$ & P-T1-F1: pooled $0.850$; \texttt{sonnet-\allowbreak 4.\allowbreak 5} $20/20$ \\
1f. Empty at $f_0 = 0.1/0.3/0.5$ too
  & \textbf{measured}, bar $\le 0.20$ & P-T1-F0-1: $0/20$ three times \\
1g. $f_0{=}0.9$ is not the adversarial extreme
  & \textbf{measured}, post hoc & capture $0.500/0.700/0.750/0.682$, peak at $f_0{=}0.5$ \\
1h. At $f_0{=}0.1$ the mean falls inside a strictly empty interval
  & \textbf{reported}, not registered & mean $0.459$, $0/20$ in $(0.306,0.512)$ \\
1i. The empty interval is \emph{not} an artefact of a two-valued answer space
  & \textbf{measured}, scripted control & a state-independent reader on the same urn puts
  $113/300$ runs in the dead band at $p{=}0.3$ \\
1j. \ldots and the scripted bound we drew from a largest-block reader
  & \textbf{withdrawn}: that reader is refuted as a model of ours
  & held out over $44{,}989$ steps, $\gamma(\varphi)$ beats $\gamma(j^*)$ by $1491$ in
  log-likelihood \\
1k. \ldots and our two arms still disagree about the scope, with two candidates now excluded
  & \textbf{split}: reported, not resolved; both splits post hoc
  & synthetic highest-$q$ model $0/44$ in the band; E-real multi-valued $5/41$ against $6/311$,
  within one model $4/39$ against $1/141$ \\
1l. \ldots the \emph{number} of planted falsehoods does not move the interval
  & \textbf{measured}, P-T1-V2$''$, predicted against our own first prereg
  & $v{=}3$ near sits $-0.002$ from $v{=}1$, bar $\pm0.10$ \\
1m. \ldots nor does their \emph{distinguishability}
  & \textbf{failed}, P-T1-V1$''$, the arm's load-bearing criterion
  & far $-$ near $= -0.030$ against a bar of $+0.10$; $0/20$ seeds fuller, sign-test $p=0.25$ \\
1n. \ldots because an initial scatter is transient: the store re-concentrates
  & \textbf{measured}, not predicted & largest block $3.00 \to 5.08$; distinct retrieved
  $3.75 \to 2.65$ over the tenure \\
1o. \ldots and $\gamma$'s independent variable is $\varphi$, not the largest single-value block
  & \textbf{measured}, P-T1-V3, no bar by design & held-out log-likelihood favours $\varphi$
  by $1491$ over $44{,}989$ steps \\
2a. The published $\gamma$ out of sample
  & \textbf{failed}, one-sidedly & $12.1$ seed units, worst $z$ $8.50$; under-predicts $4/5$,
  the fifth at the ceiling in both \\
2a$^{\dag}$$'$. \ldots and how far that failure is established
  & \textbf{self-\allowbreak correction}: the $z$ credits the simulation with exactness
  & seed-clustered bootstrap: outside the band for $2/5$, not $5/5$; widest band $0.966$ \\
2b. $\gamma$ re-measured on the topics it predicts
  & \textbf{measured} & $2.2$ seed units, worst $z$ $1.32$, $1/5$ \\
2b$^{\dag}$$'$. \ldots and out of sample under a split-half refit
  & \textbf{measured}, not registered & $1.15\times$/$1.16\times$; signed $z$ $+0.32/+0.35$
  vs $-5.15$ \\
2c. The one-sidedness is our fill rule, not the model
  & \textbf{withdrawn} (App.~\ref{app:gammaest}) & the rule fires for \texttt{qwen3-\allowbreak 8b}
  alone, the one model arm A predicts exactly \\
2d. \ldots on a model family it was never fitted on
  & \textbf{split} & P-T1-F4, bar $2$ seed units: \texttt{gpt-\allowbreak 4o} $4.0$, sign opposite to 2a \\
2e. \ldots and no estimator repairs arm A
  & \textbf{measured}, four estimators & bracket width $0$ on $4/5$; isotonic $12.3$ vs
  $12.4$ seed units \\
3a. \ldots because of \eqref{eq:dyn}'s $1/n$
  & \textbf{withdrawn} (the reading, not the measurement)
  & Prop.~\ref{prop:step}: vacuous over the first $49$ steps \\
4a. \ldots because the predictor stops separating the models
  & \textbf{diagnosed} & cross-model spread of the early copy rate $0.333 \to 0.061$ \\
7a. Front-loading's magnitude bar & \textbf{failed} & P-S3: $0.27 < 0.33$ \\
7b. Pooled estimates stable to a re-run; verdicts identical
  & \textbf{measured}, by accident & two realisations, \S\ref{sec:replicate} \\
7c. Ordering survives three budgets; at $R{=}45$ back-loading is worth nothing
  & \textbf{measured} & $0.07/0.20/0.47$ against control $0.47$ \\
7d. The schedule effect shrinks as the budget grows & \textbf{failed}
   & P-B1, P-B2: spread $0.20 \to 0.40$ \\
7e. \ldots and fails in the direction a prior derivation predicts
   & \textbf{attribution, not confirmation} & the confirming sweep is not run \\
7f. Free re-grounding is a single point, not a neighbourhood
   & \textbf{derived}, not measured & spread $0$ at $\alpha{=}1$, $\ln T$ at every
   $\alpha<1$; the limits do not commute \\
7g. The sweep past the peak, on the frozen observable
   & \textbf{untested} (apparatus failure) & $f_{\max}(148){=}0.516$, $0.016$ above the
   threshold \\
7h. \ldots and the discriminator at three seeds
   & \textbf{undecided}, inside the band & $-0.088$; interval unusable either way \\
7i. The scale-free retest on six fresh seeds & \textbf{split}
   & $-0.167$ against $-0.060$, bar $-0.15$ \\
7j. \ldots and the supported side sits on the floor clause's boundary
   & \textbf{disclosed}, not used to overturn & arms $0.700/0.733/0.700$, span $1$ seed unit \\
7k. The resolution P-T1-S3 called unbuyable, bought
   & \textbf{unresolved} at $\le 1980$ runs, marked permanently
   & $+0.46\sigma$ against a $3\sigma$ bar \\
7l. \ldots and the sample it needs \emph{grew} when measured
   & \textbf{measured}, against us & effect $0.115 \to 0.021$; se $21\%$ above the
   $1/\sqrt{n}$ projection; $3\sigma$ needs ${\sim}1895$ seeds \\
7m. \ldots and the discriminator at $44$ seeds & \textbf{undecided} at $1.14\sigma$
   & $-0.052$; at three seeds $-0.088$, outside the band \\
7n. \ldots and the shape criterion at $44$ seeds
   & \textbf{failed} on magnitude and significance & $-0.074$ vs bar $-0.15$, $1.62\sigma$ \\
7o. \ldots and the \emph{ordering}, which the axis rests on, at $220$ runs per arm
   & \textbf{reported}, not graded & front $<$ uniform $<$ back at all three budgets \\
16a. \ldots so each seed is worth $1.33$ runs however many models run on it
   & \textbf{derived} from the measured ICC & design effect $1+(m{-}1)\rho = 3.75$ \\
16b. \ldots and no verdict changes, every criterion being a point estimate against a bar
   & \textbf{stated}, checked row by row & \S\ref{sec:power} \\
16c. \ldots and $3.75$ is a property of pooled \emph{levels}; a paired contrast carries
   $1.14$--$1.47$
   & \textbf{measured} at $44$ seeds, both & levels $3.56$--$3.66$; clustered/naive se
   $1.80$--$1.85$ against $0.93$--$1.15$ \\
16d. \ldots and one $\rho$ for five models is an assumption that does not hold
   & \textbf{measured}, against Prop.~\ref{prop:se}'s form & $\rho$ $+0.18$ to $+0.99$ over
   ten pairs; implied design effect $1.73$--$4.95$ \\
9a. \ldots and its high-$f_0$ direction & \textbf{tie}, reported as such
  & P-A2: $0$ better, $1$ worse, $4$ tie \\
17a. \ldots and $\varphi$ \emph{is} fast, which the prereg discounted in advance
  & \textbf{measured}, P-T1-D1 & median first passage $2$ steps \\
17b. \ldots and Prop.~\ref{prop:step}'s bound is also loose where the decision is taken
  & \textbf{self-\allowbreak correction} & $11.4\times$; identity residual $10^{-16}$ over $2420$ steps \\
18a. \ldots and ranking \emph{polarises} rather than raising collapse, as we predicted
  & \textbf{refuted}, P-T1-R2, registered first & $0.533 \to 0.000$ at $f_0{=}0.1$;
  $0.933 \to 1.000$ at $f_0{=}0.9$ \\
18b. \ldots and the cause is ranking \emph{determinism}, not lost topic routing
  & \textbf{measured}, P-T1-R5 & keeping the oracle reproduces $1.00$ of the displacement \\
18c. \ldots and the mechanism is variance, not breadth
  & \textbf{not\allowbreak\ supported}, P-T1-R3: the registered mechanism fails, the one named as least
  sure survives & identical set on $59/59$ steps vs $0/59$; $\mathrm{var}(n_f)$ $0.240$ vs
  $5.753$ \\
18d. \ldots and $k_{\rm eff}$ misses the exact match it was defined to catch
  & \textbf{self-\allowbreak correction}: the wrong summary of the right idea
  & ranking at $k{=}8$ equals uniform at $k{=}1$ ($0.000/1.000$); $k_{\rm eff}$ $1.40$ vs
  $1.00$ \\
18e. \ldots and the arm overran its spend guard, so one of three models is excluded
  & \textbf{ac\-knowl\-edged limit}: two models judged & \$5.06 stop, \$5.13 spent vs \$2.06
  registered; $33$ runs dropped \\
18f. \ldots and the determinism is ranking's, not \textsc{MiniLM}'s
  & \textbf{measured}, three rankers, offline & max truth/false separation $0.0052$;
  BM25 and char-trigram $k_{\rm eff}$ both $1$ against \textsc{MiniLM}'s $4$ \\
21a. \ldots and the $\sqrt{\mathrm{deff}}$ rule of thumb disagrees with the cluster bootstrap
  & \textbf{measured} & $[-0.029,\,0.235]$ spans zero; both are printed and the bootstrap decides because it is the statistic's own sampling distribution \\
21b. \ldots and the between-arm half is a construct difference, which no item count decides
  & \textbf{ac\-knowl\-edged limit} & the arms differ on whether a novel value is a variant of a numeral or a different city \citep[row~29]{t5rerun2026} \\
21c. \ldots and a decisive replication of the first half costs a fourfold corpus
  & \textbf{derived} & $164$ multi-valued runs per group at the measured design effect, $1{,}405$ runs, $141$ facts against the $36$ we have \\
22a. \ldots and Prop.~\ref{prop:ceiling} stands under it, because the store is still append-only
  & \textbf{formal} & the ceiling is about what the store contains; what changes is which functional the agent acts on \\
22b. \ldots and whether the mutable part merely moved to the weight vector is not answered here
  & \textbf{ac\-knowl\-edged limit} & it needs the Noisy-OR update's behaviour under a growing false majority, which nobody in this paper has derived \\
23a. \ldots and it was found by building this paper's first generated-macro namespace rather than
  by rereading & \textbf{code-verifiable} & every three-decimal literal in the abstract is now a
  $\gnv{}$ macro read from \texttt{results/\allowbreak } at build time; the audit fails if one is replaced \\
23b. \ldots and four of the six graders behind those numbers persisted nothing
  & \textbf{re\-ported}, against us & they recomputed on every run and printed, so the abstract held
  the only durable copy; \texttt{FRONTIER.\allowbreak json} and \texttt{GATE.\allowbreak json} exist as of this version \\
23c. \ldots and the namespace's first name collided with this family's own notation
  & \textbf{self-\allowbreak correction} & \verb|\gnd| is $\mathrm{gnd}$ in all eight preambles,
  \verb|\providecommand| lost the collision silently, and the abstract typeset macro names for one
  build; the emitters use \verb|\newcommand| now, so a collision fails the build \\
24a. The empty interval's cuts were stored rounded, so the run that defines the lower edge fell
  outside it & \textbf{self-\allowbreak correction} & $0.306$ and $0.512$ against the definition's
  $0.30625$ and $0.5125$; the two rows of Table~\ref{tab:bimodal2} were counted against different
  thresholds and Table~\ref{tab:f0}'s void column against a third \\
24b. \ldots and the corrected count is the more favourable one, which is why it is in the body
  & \textbf{measured} & middle share $0.045 \to 0.036$ against a bar of $0.10$; P-T1-B1$'$
  \textsc{supported} under both, and the density ratio falls $0.18 \to 0.15$ against a bar of
  $0.25$ \\
24c. \ldots and four of the graders behind these tables persisted nothing before this version
  & \textbf{re\-ported}, against us & \texttt{FRONTIER}, \texttt{GATE}, \texttt{BIMODAL2},
  \texttt{MEAN} and \texttt{TRANSIENT3} exist now; the papers held the only durable copies \\
\end{longtable}}

\end{document}